\documentclass[10pt]{article} 
\usepackage[preprint]{tmlr}

\usepackage{amsmath,amsfonts,bm}

\def\eqref#1{equation~\ref{#1}}

\def\1{\bm{1}}

\DeclareMathAlphabet{\mathsfit}{\encodingdefault}{\sfdefault}{m}{sl}
\SetMathAlphabet{\mathsfit}{bold}{\encodingdefault}{\sfdefault}{bx}{n}

\usepackage{hyperref}
\usepackage{url}
\usepackage{microtype}
\usepackage{graphicx}
\usepackage{subcaption}
\usepackage{array}
\usepackage{tabularx}
\usepackage{booktabs}
\newcolumntype{Y}{>{\raggedright\arraybackslash}X}
\usepackage{wrapfig}
\usepackage[percent]{overpic}

\newcommand{\cliprankdir}{%
  figures/clipreal%
}

\usepackage{amsmath}
\usepackage{amssymb}
\usepackage{mathtools}
\usepackage{amsthm}

\let\AND\relax
\usepackage{algorithm}
\usepackage{algorithmic}

\usepackage[autostyle=true]{csquotes}
\usepackage{placeins}

\usepackage[capitalize,noabbrev]{cleveref}

\theoremstyle{plain}
\newtheorem{theorem}{Theorem}[section]

\theoremstyle{definition}

\theoremstyle{remark}

\title{Learning $\omega$-Powered Conditional Distributions with Importance-Weighted Flow Matching}

\author{\name Maciej Falkiewicz \email maciej.falkiewicz@hesge.ch \\
      \addr Computer Science Department, University of Geneva\\
      HES-SO/HEG Gen\'eve \\
      \AND
      \name Frantzeska Lavda  \email frantzeska.lavda@hesge.ch \\
      \addr HES-SO/HEG Gen\'eve \\
      \AND
      \name Jason Ramapuram \email jramapuram@google.com\\
      \addr Google Deepmind (work done at Apple) \\
      \AND
      \name Alexandros Kalousis \email alexandros.kalousis@hesge.ch \\
      \addr HES-SO/HEG Gen\'eve \\
}

\def\month{10}  
\def\year{2026} 
\def\openreview{\url{https://openreview.net/forum?id=fBli1qL3dd}} 

\begin{document}

\maketitle

\begin{abstract}
Classifier-Free Guidance (CFG) with guidance scale $\gamma$ is commonly motivated by the $\omega$-powered conditional distribution obtained by setting $\omega=\gamma$, but CFG does not, in general, sample from this distribution exactly. We study whether the powered target can instead be learned explicitly. We introduce Importance-Weighted Conditional Flow Matching (IW-CFM), which reweights conditional flow-matching examples using density ratios obtained from frozen conditional and unconditional flow models. With exact normalized weights, the resulting objective has the same gradient as flow matching on the powered target; in practice, we use self-normalized estimation and a regularized empirical approximation for high-dimensional data. On synthetic distributions, IW-CFM recovers analytical powered targets for which numerically matched CFG, $\gamma=\omega$, exhibits a visible mismatch. On ImageNet, an amortized lightweight corrector achieves better alignment than fixed- and matched-guidance CFG with the corresponding weighted and filtered reference distributions. We additionally present an exploratory CLIP analysis of powered-density rankings on real ImageNet data. Aggregate shifts across selected prompt-specific contrasts coexist with heterogeneous classwise responses, including nonmonotonic trajectories. These results provide a practical route for studying and approximating model-defined powered conditional distributions.\footnote{Code and precomputed ImageNet log-probabilities are publicly available at \url{https://bitbucket.org/dmmlgeneva/iw-cfm}.}
\end{abstract}

\begin{figure}[t]
    \centering
    \begin{subfigure}[t]{0.21\textwidth}
        \centering
        \includegraphics[width=\linewidth]{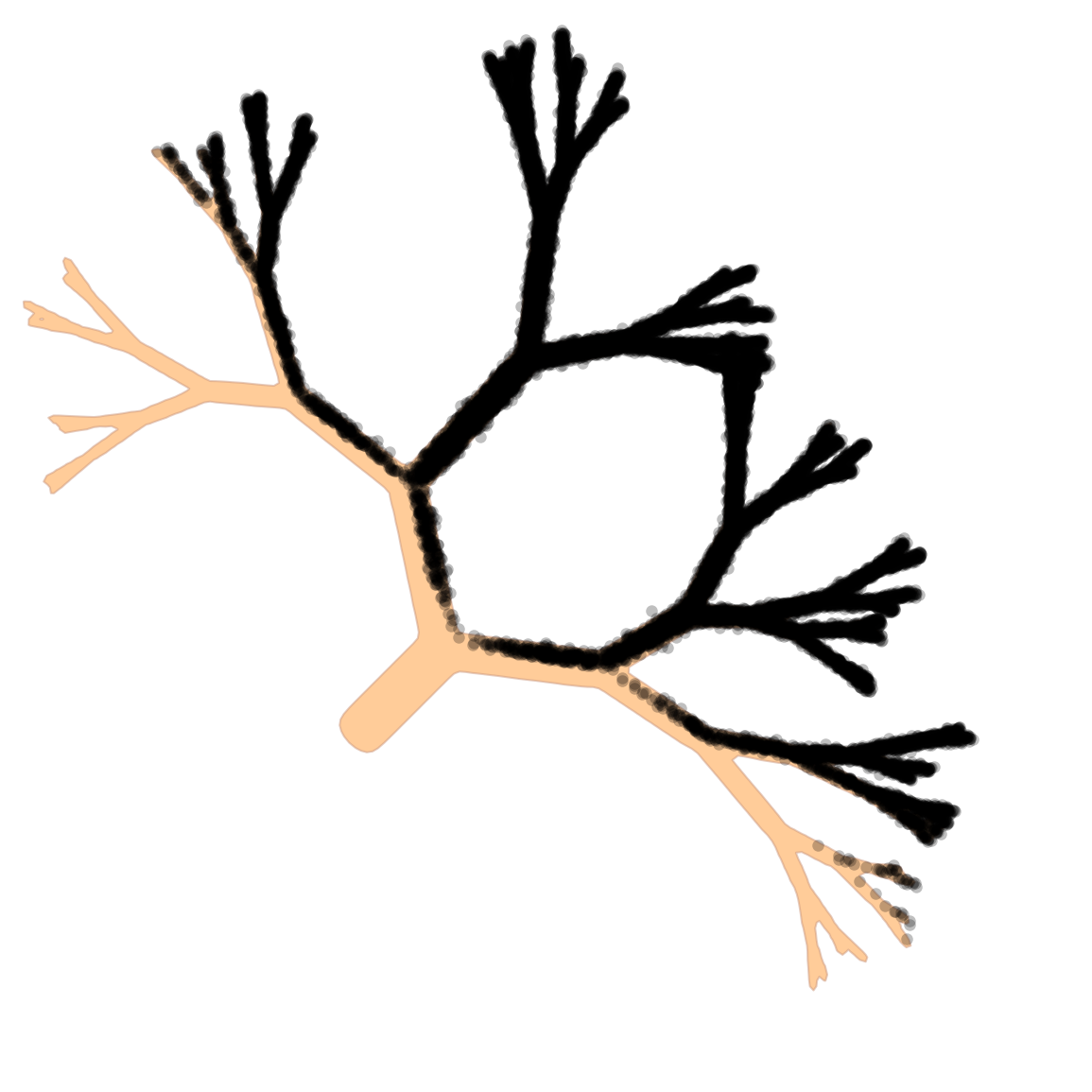}
        \caption{The black dots are CFG samples at $\gamma=4$ and depart from $p_{\omega=4}(x|y)$. The orange contour encloses 99.7\% of the target probability mass.}
        \label{fig:fractal_cfg}
    \end{subfigure}
    \hfill 
    \begin{subfigure}[t]{0.21\textwidth}
        \centering
        \includegraphics[width=\linewidth]{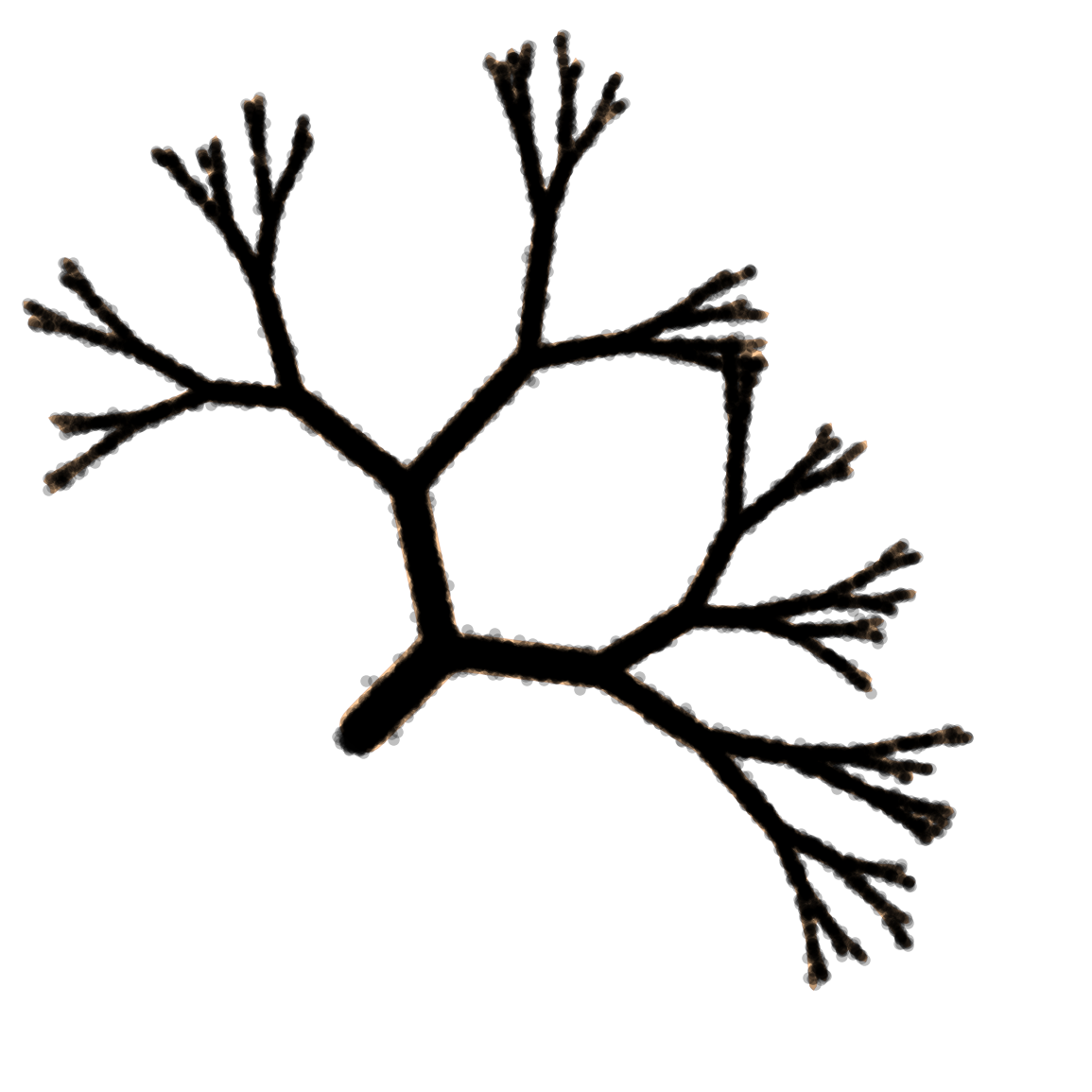}
        \caption{IW-CFM samples at $\omega=4$, drawn without additional CFG, follow $p_{\omega=4}(x|y)$ and cover the region enclosed by its 99.7\% probability-mass contour.}
        \label{fig:fractal_iwcfm}
    \end{subfigure}
    \hfill 
    \begin{subfigure}[t]{0.54\textwidth}
        \centering
        \includegraphics[width=\linewidth]{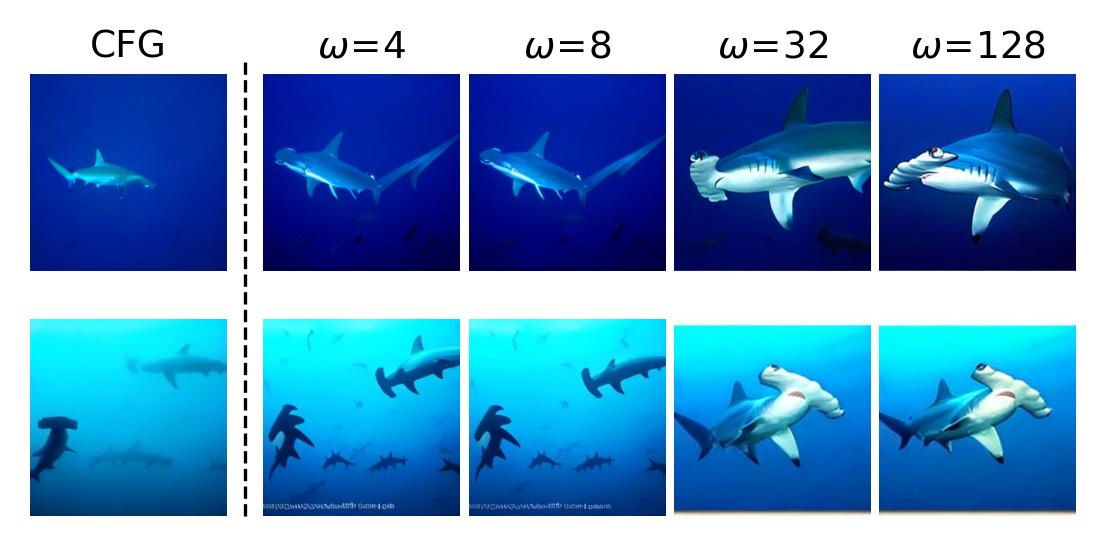}
        \caption{Qualitative comparison on ImageNet. Left column: baseline CFG with guidance scale $\gamma=4$. Right columns: IW-CFM at varying target $\omega$, also using inference-time guidance $\gamma=4$ and initialized from the same prior latents. Thus, the columns vary IW-CFM's powering parameter while holding its inference-time guidance scale fixed.}
        \label{fig:cfg_iwcfm_shark_grid}
    \end{subfigure}

    \caption{Synthetic target recovery and qualitative ImageNet generations. \textbf{(a,b)} On a synthetic dataset, IW-CFM samples match the 99.7\% probability-mass contour of the analytical $p_\omega$, whereas CFG with numerically matched guidance $\gamma=\omega$ shows a visible mismatch. \textbf{(c)} On ImageNet, IW-CFM is trained using regularized empirical weights; both the CFG baseline and IW-CFM use inference-time guidance $\gamma=4$, while IW-CFM additionally varies the target powering parameter $\omega$.}
    \label{fig:combined_comparison}
\end{figure}

\section{Introduction}

Diffusion models \citep{sohl-dickstein15, ddpm, song2021scorebased}, flow matching (FM) \citep{lipman2023flow, liu2023flow}, and stochastic interpolants \citep{albergo2025unifying} are widely used for continuous data such as images and audio. These methods learn a process between a simple prior, such as a Gaussian, and the data distribution. We work in the Gaussian flow-matching setting. Conditional models commonly use \emph{Classifier-Free Guidance (CFG)} \citep{ho2022classifier,zheng2023guided,dao2023flow}, which steers generation toward a condition $y$ by combining conditional and unconditional vector-field estimates.

Both CFG and its predecessor, classifier-based guidance \citep{dhariwal2021diffusion}, are motivated by a \emph{powered distribution},
\begin{equation}
p_\omega(x|y) \propto p(x|y)^\omega p(x)^{1-\omega},
\label{eq:powered_target_intro}
\end{equation}
where $\omega>1$ is the \emph{powering parameter}. Separately, we denote the CFG guidance scale by $\gamma$: CFG combines the conditional and unconditional vector fields as $v_{\mathrm{CFG}}^\gamma=\gamma v(\cdot|y)+(1-\gamma)v(\cdot)$. Its standard motivating interpretation identifies the two numerical values, $\gamma=\omega$, and treats $v_{\mathrm{CFG}}^\gamma$ as if it generated $p_\omega$. This identification is not generally exact \citep{bradley2024classifier, chidambaram2024does}: the CFG-guided backward vector field lacks a corresponding path-consistent forward process, and its endpoint distribution need not equal $p_\omega$, even when $\gamma=\omega$. The resulting mismatch can push probability mass away from the learned data manifold, manifesting as artifacts and unnatural saturation at high guidance scales. A related discrepancy has been identified in discrete diffusion models \citep{kit2025debiasing}.

We address the powered target directly by specializing conditional Energy-Weighted Flow Matching \citep{zhangenergy}. Density estimates from the conditional and unconditional flow models define the endpoint density ratio, and exact normalized weights yield the same population gradient as flow matching on the powered target. We call the resulting method \emph{Importance-Weighted Conditional Flow Matching (IW-CFM)}; it conditions on $\omega$ and learns the flow for $p_\omega(x|y)$ without direct samples from that distribution. Any CFG subsequently applied to the learned IW-CFM vector field is an optional inference-time heuristic controlled separately by $\gamma$.

Beyond target recovery, we examine how powered-density rankings of real images change with $\omega$. Using CLIP \citep{radford2021learning}, which prior work has applied to image quality assessment \citep{clipiga2022, wang2023exploring, zhang2023blind}, we characterize selected prompt-specific contrasts and their classwise trajectories.

Our main contributions are:
\begin{itemize}
    \item Building on Energy-Weighted Flow Matching \citep{zhangenergy}, we instantiate endpoint weighting for model-defined $\omega$-powered conditional targets using density ratios from frozen conditional and unconditional flow models.
    \item We develop a hierarchical self-normalized estimator for heterogeneous $(y,\omega)$ pairs and distinguish this finite-sample procedure from the exact population objective.
    \item We develop a lightweight corrector amortized over $\omega$ and a high-dimensional ImageNet construction based on regularized precomputed weights, proposal sampling, and interpolation across powering parameters.
    \item We validate recovery of analytical powered targets on synthetic distributions and measure alignment with regularized empirical references on ImageNet, where IW-CFM yields lower distances than the evaluated CFG baselines.
    \item We provide an exploratory characterization of powered-density rankings on real images through selected CLIP contrasts, including heterogeneous and nonmonotonic classwise responses.
\end{itemize}

\section{Modeling the $\omega$-Powered Distribution}
We adopt the Flow Matching (FM) framework \citep{lipman2023flow}, which defines a probability path $p_t(x)$ for $t \in [0, 1]$ interpolating between a prior $p_0(x) = \mathcal{N}(x; 0, I)$ and data $p_1(x)$. The flow is governed by a time-dependent vector field $v(\cdot, t)$ via the ODE $dx_t = v(x_t, t) dt$.

\subsection{Background on Conditional Flow Matching}
In Conditional Flow Matching, for each data point $x_1$, we define a conditional probability path $p_t(x|x_1)$ that interpolates between a prior sample $x_0 \sim p_0$ and the data point $x_1$. The path is constructed by defining the interpolated sample $x_t$ and its corresponding target velocity vector $u_t$ as:
\begin{align}
    x_t &= \alpha_t x_1 + \sigma_t x_0 \label{eq:interpolation} \\
    u_t &= \dot{\alpha}_t x_1 + \dot{\sigma}_t x_0 \label{eq:velocity}
\end{align}
where $\alpha_t$ and $\sigma_t$ are differentiable functions of time $t \in [0,1]$ defining the schedule. The goal of a standard conditional flow matching model is to learn a vector field that predicts this target velocity $u_t$ marginalized over $x_1$ \citep{lipman2023flow}.

\subsection{CFG Guidance and the Powered Target}
We reserve $\gamma$ for the CFG guidance scale and $\omega$ for the powering parameter. Given conditional and unconditional vector fields, CFG uses
\begin{equation}
v_{\mathrm{CFG}}^\gamma(x,t,y)
=\gamma v(x,t|y)+(1-\gamma)v(x,t).
\label{eq:cfg_vector_field}
\end{equation}
The usual powered-distribution motivation associates CFG at $\gamma$ with $p_\omega$ at the matched numerical value $\omega=\gamma$. We write this motivating correspondence as
\begin{equation}
\mathrm{CFG}(\gamma=\omega)\ \overset{\text{motivation}}{\longrightarrow}\ p_\omega,
\qquad
\operatorname{Law}\!\left(\mathrm{CFG}(\gamma=\omega)\right)\neq p_\omega
\quad\text{in general}.
\label{eq:cfg_powered_mismatch}
\end{equation}
IW-CFM addresses the right-hand side directly by learning a vector field conditioned on $\omega$. When we later apply CFG to this learned field at inference time, its scale remains a distinct parameter $\gamma$.

\subsection{Density Evaluation and Pre-computation}
To enable importance weighting, we require the log-densities of the unconditional $p(x)$ and conditional $p(x|y)$ distributions. Given a vector field $v$, we use the instantaneous change of variables formula \citep{chen2018neural}:
\begin{equation}
\log p(x) = \log p_0(x_0) - \int_0^1 \nabla \cdot v(x_t, t) \,dt
\label{eq:log_prob_flow}
\end{equation}
where $x_0$ is obtained by integrating the ODE backward from $t=1$ to $t=0$. In high dimensions, we approximate the divergence $\nabla \cdot v$ using the Hutchinson trace estimator \citep{Hutchinson}. We validate the numerical stability and ranking consistency of these lower-step estimators in \cref{app:exp_details_logp}. The conditional log-density, $\log p(x|y)$, is computed similarly using the conditional vector field $v(x|y)$.

\textbf{Practical Implementation.} Because the base vector fields $v(x)$ and $v(x|y)$ remain frozen, we precompute these log-densities for the training set. Corrector training therefore does not backpropagate through the density models.

\subsection{Importance-Weighted Conditional Flow Matching (IW-CFM)}
Our goal is to learn a vector field $v^\phi(x,t,y,\omega)$ that targets the normalized powered distribution $p_\omega(x|y) = \tilde{p}_\omega(x|y)/Z(y,\omega)$. We define the unnormalized weight for a sample $(x, y)$ as
\begin{equation}
\tilde{w}_\omega(x, y) = \frac{\tilde{p}_\omega(x|y)}{p(x|y)} = p(x|y)^{\omega-1} p(x)^{1-\omega}.
\label{eq:importance_weight}
\end{equation}
The exact normalized density ratio is
\begin{equation}
w_\omega^\star(x,y)
= \frac{\tilde{w}_\omega(x,y)}{Z(y,\omega)}
= \frac{p_\omega(x|y)}{p(x|y)},
\qquad
Z(y,\omega)=\mathbb{E}_{p(x|y)}[\tilde{w}_\omega(x,y)].
\label{eq:exact_importance_weight}
\end{equation}
Using this exact ratio gives the population \emph{IW-CFM objective}
\begin{equation}
\mathcal{L}_{\text{IW-CFM}}^\star(\phi, \omega) = \mathbb{E}_{p(y), p(x_1|y)} \bigg[ w_\omega^\star(x_1, y) \mathbb{E}_{p(t), p_t(x | x_1)} \Big[ \| v^\phi(x,t,y,\omega) - u_t(x|x_1) \|^2 \Big] \bigg].
\label{eq:iw_cfm_obj}
\end{equation}
The normalizer $Z(y,\omega)$ is generally unavailable. For a batch of $K$ samples $\{x_i\}_{i=1}^K$ drawn from $p(x|y)$, we therefore use finite-batch self-normalized importance sampling (SNIS) weights
\begin{equation}
\widehat{w}_{\omega,K}^{(i)}
= \frac{\tilde{w}_\omega(x_i,y)}{\sum_{k=1}^K \tilde{w}_\omega(x_k,y)}.
\label{eq:snis_importance_weight}
\end{equation}
As shown in \cref{app:gradient_proof}, the exact objective $\mathcal{L}_{\text{IW-CFM}}^\star$ has the same gradient as the ideal flow-matching objective on $p_\omega$. Replacing $w_\omega^\star$ with $\widehat{w}_{\omega,K}$ gives the practical estimator in \cref{alg:iw_cfm_training}; its gradient is biased for finite $K$ and becomes unbiased asymptotically as $K\to\infty$. While standard Flow Matching minimizes an unweighted MSE (see e.g., \citep{lipman2023flow}), our approach uses these importance weights to re-target the powered distribution.

\begin{algorithm}[t]
\caption{Importance-Weighted Conditional Flow Matching (IW-CFM)}
\label{alg:iw_cfm_training}
\begin{algorithmic}[1]
    \REQUIRE Dataset $\mathcal{D}$, frozen log-density models $\log p(x)$ and $\log p(x|y)$, distribution $p(\omega)$.
    \REQUIRE Hyperparameters: $M$ ($(y, \omega)$ pairs), $K$ (samples per pair).
    \REPEAT
        \STATE \textbf{1. Parallel Hierarchical Sampling}
        \STATE Sample conditions and powering parameters $\{(y_j, \omega_j)\}_{j=1}^M \sim p(y) \times p(\omega)$
        \STATE Sample $M$ subbatches $\mathcal{B}_j = \{x_1^{(i,j)}\}_{i=1}^K \sim \mathcal{D}$ conditioned on $y_j$

        \STATE \textbf{2. Vectorized Importance Weighting}
        \STATE \COMMENT{Compute log-weights $\ell^{(i,j)}$ for all samples in parallel}
        \STATE $\ell^{(i,j)} \gets (\omega_j-1) \log p(x_1^{(i,j)}|y_j) + (1-\omega_j) \log p(x_1^{(i,j)})$
        \STATE $\widehat{w}_{\omega_j,K}^{(i,j)} \gets \text{Softmax}_i(\ell^{(i,j)}) = \frac{\exp(\ell^{(i,j)})}{\sum_{m=1}^K \exp(\ell^{(m,j)})}$ \hfill \text{// Nested Normalization}

        \STATE \textbf{3. Batch Objective Evaluation}
        \STATE Sample $t_{i,j} \sim \mathcal{U}[0,1]$ and noise $x_0^{(i,j)} \sim \mathcal{N}(0, I)$ for all $(i,j)$
        \STATE Compute flow targets $(x_t, u_t)^{(i,j)}$ via \cref{eq:interpolation,eq:velocity}

        \STATE Define the batch loss as the weighted average:
        \STATE $\mathcal{L}_{\text{batch}} \gets \frac{1}{M} \sum_{j=1}^M \sum_{i=1}^K \widehat{w}_{\omega_j,K}^{(i,j)} \| v^\phi(x_t, t, y_j, \omega_j)^{(i,j)} - u_t^{(i,j)} \|^2$

        \STATE Update $\phi \gets \phi - \eta \nabla_\phi \mathcal{L}_{\text{batch}}$ \hfill \text{// Gradient Descent}
    \UNTIL{converged}
\end{algorithmic}
\end{algorithm}

\subsection{Amortized $\omega$-Powering via Lightweight Correctors}
\label{sec:amortized_powering}
To avoid training a separate model for each powering parameter, we treat $\omega$ as a conditioning variable for a single corrector network. We implement this corrector, $v^{\phi}$, using lightweight, parameter-efficient designs.
For synthetic experiments, $v^{\phi}$ is an MLP that predicts an affine correction to the baseline conditional vector field. For real-world image data (ImageNet), we build upon the Scalable Interpolant Transformer (SiT) backbone \citep{ma2024sit}, introducing a parameter-efficient adapter into its Transformer blocks, similar to \citet{jensen2025efficient}. This way, we learn a single model to target the entire family of powered distributions, adding only about 5\% of the parameters compared to the baseline SiT.

A key challenge in this amortized setting is normalizing the importance weights. Normalizing globally across a batch with different conditions and powering parameters changes the effective distribution over $(y,\omega)$: each pair is weighted in proportion to its partition function $Z(y,\omega)$ rather than according to the intended outer distribution $p(y)p(\omega)$ (see \cref{app:bias_variance}). We instead use hierarchical batches and nested normalization. At each training step, we sample $M$ condition--parameter pairs $(y_j, \omega_j)$ and a subbatch of $K$ data points for each pair. We normalize weights only within samples that share $(y_j, \omega_j)$, preserving the intended outer distribution. \Cref{alg:iw_cfm_training} gives the procedure.

\subsection{Scaling to High-Dimensional Data}
\label{sec:scaling_method}

\Cref{alg:iw_cfm_training} provides an asymptotically unbiased estimator for streaming data, which we apply directly in our synthetic experiments. However, this approach is computationally prohibitive for high-dimensional data. We therefore adopt two practical techniques that enable an efficient application of IW-CFM to ImageNet.

\textbf{Amortized Training with Precomputed Weights.}
Sampling from the weighted empirical distribution at an arbitrary continuous $\omega$ would require computing the weight of every ImageNet training instance for every sampled $\omega$. We instead precompute normalized weights over the full dataset at integer powering parameters. Given a continuous $\omega$, we sample training instances with replacement using the weights at the nearest integer $\bar m(\omega)$. For those instances only, we linearly interpolate the self-normalized weights at $\omega$ and multiply the flow-matching loss by $r_\omega^{\mathrm{reg}}(i|y)/r_{\bar m(\omega)}^{\mathrm{reg}}(i|y)$. This proposal is more efficient than uniform sampling followed by SNIS: at large $\omega$, most stored weights are zero, so most uniformly sampled instances would provide no weighted training signal. \Cref{app:distribution_tempering} specifies the zero-weight fallback and numerical threshold.

\textbf{Operational ImageNet Weights.}
The population formulation assumes access to exact conditional and unconditional densities. On ImageNet, we instead use densities induced by frozen pretrained SiT vector fields, which are learned approximations to the data distributions, and evaluate them numerically through augmented ODE integration. Powering amplifies errors in these density estimates. In particular, if the estimated log-density ratio
\[
\widehat d(x,y)
=
\log \widehat p(x|y)-\log \widehat p(x)
\]
differs from its ideal value by an additive error $\epsilon(x,y)$, then the corresponding unnormalized weight is distorted by the multiplicative factor
$\exp((\omega-1)\epsilon(x,y))$. At large $\omega$, even moderate errors can therefore cause a few training instances to dominate the normalized empirical distribution.

Given any fixed set of endpoint weights, importance-weighted flow matching trains toward the distribution encoded by those weights. It does not, however, correct errors in the weights themselves. If exact and reliable densities were available, we could use the unregularized density ratios directly. With the available ImageNet estimates, their extreme concentration produces a brittle empirical target: most normalized weights become numerically zero, while a small number of potentially misestimated instances receive nearly all of the mass. We therefore regularize the weights before training.

Following the standard temperature-scaling transformation \citep{guo2017calibration}, we first contract all log-density-ratio differences by a temperature $T>1$. We then cap the largest per-class log-weights at each precomputed integer powering parameter $m$. For the $N_y$ training instances in class $y$, define
\begin{equation}
\begin{aligned}
a_m(i,y)
&= \frac{m-1}{T}\left[\log p(x_i|y)-\log p(x_i)\right],\\
r_m^{\mathrm{reg}}(i|y)
&= \frac{\exp\!\left(\operatorname{Clip}_{\tau,y}(a_m(i,y))\right)}
{\sum_{j=1}^{N_y}\exp\!\left(\operatorname{Clip}_{\tau,y}(a_m(j,y))\right)}.
\end{aligned}
\label{eq:regularized_integer_target}
\end{equation}
Temperature scaling reduces concentration globally. Increasing $T$ sufficiently to suppress the most extreme weights would also flatten potentially meaningful variation throughout the remainder of the class. We therefore additionally apply $\operatorname{Clip}_{\tau,y}$, which caps only the largest log-weights without discarding their corresponding instances. Clipping limits the influence of extreme density-ratio estimates while retaining variation across the bulk of the data. It constitutes a deliberate bias--stability trade-off rather than an exact approximation that preserves the theoretical $p_m$. \Cref{fig:non_zero_weights} illustrates the finite-precision consequence of the unregularized concentration and the effect of $T$ and $\tau$ on the number of numerically active instances. \Cref{app:distribution_tempering} defines the clipping operator and reports the selected parameters.

For a non-integer $\omega$, let $m_-<\omega<m_+$ be its adjacent precomputed integers and $\lambda=(\omega-m_-)/(m_+-m_-)$. We define
\begin{equation}
r_\omega^{\mathrm{reg}}(i|y)
= (1-\lambda)r_{m_-}^{\mathrm{reg}}(i|y)
+ \lambda r_{m_+}^{\mathrm{reg}}(i|y).
\label{eq:interpolated_target}
\end{equation}
When $\omega=m$ is an integer grid point, we use $r_m^{\mathrm{reg}}$ directly. This interpolation defines a continuous operational family between the precomputed references; it is not claimed to reproduce the exact unregularized powered distribution at non-integer $\omega$.

The resulting $r_\omega^{\mathrm{reg}}$ therefore defines the operational weighted ImageNet reference, rather than an exact estimator of the theoretical $p_\omega$. During corrector training, these weights are evaluated only for instances drawn from the nearest-integer proposal. Moreover, zero-weight re-injection and the loss-correction fallback make the actual training objective a regularized mixture rather than exact sampling from $r_\omega^{\mathrm{reg}}$. Consequently, our exact target-recovery claim applies to the population formulation and controlled synthetic experiments; on ImageNet, we evaluate alignment with the explicitly defined regularized empirical references.

\section{Related Work}

The closest methodological predecessor is Energy-Weighted Flow Matching \citep{zhangenergy}. At the population level, the IW-CFM objective is a specialization of their conditional energy-weighted flow-matching formulation. For a fixed condition $y$, setting the base distribution to $p(x|y)$, the energy to
$E_y(x)=\log p(x)-\log p(x|y)$, and $\beta=\omega-1$ gives
\[
p(x|y)\exp[-\beta E_y(x)]
\propto p(x|y)^\omega p(x)^{1-\omega}.
\]
Thus, the endpoint importance-weighting principle and the resulting population equivalence to flow matching on the reweighted target follow from \citet{zhangenergy}.

Our contribution concerns the conditional, amortized, and high-dimensional realization of this principle. While \citet{zhangenergy} primarily develop energy-weighted generative models for offline reinforcement learning and additionally show, in a low-dimensional synthetic bandit ablation, that a single energy-weighted diffusion score model can be conditioned on $\beta$, IW-CFM uses a lightweight flow-matching corrector to approximate model-defined powered conditional distributions across ImageNet classes and a continuous range of $\omega$. Our parameter-efficient corrector conditioned on $\omega$ is inspired by guidance-distillation architectures \citep{meng2023distillation,jensen2025efficient}, but is trained to approximate a family of powered vector fields rather than to reproduce standard CFG. This setting requires normalization across heterogeneous $(y,\omega)$ pairs: our hierarchical estimator normalizes within each pair, preserving the intended outer distribution $p(y)p(\omega)$ rather than reweighting pairs according to their partition functions. We further develop a regularized ImageNet construction based on precomputed conditional and unconditional densities, per-class weight regularization, nearest-integer proposal sampling, and interpolation between precomputed powering parameters. Our experiments evaluate recovery across a family of analytical synthetic targets and alignment with regularized ImageNet references.

Analyses by \citet{bradley2024classifier} and \citet{chidambaram2024does} provide evidence that CFG does not generally sample from the powered distribution suggested by the numerical identification $\gamma=\omega$ and can shift probability mass away from the data manifold. These works diagnose the discrepancy; IW-CFM instead learns a vector field conditioned directly on the powering parameter $\omega$ under the idealized formulation and its regularized empirical counterpart on ImageNet.

Feynman--Kac Correctors (FKC) \citep{skretafeynman} use particle-based Sequential Monte Carlo to target modified distributions at inference time. Rectified Gradient Guidance (REG) \citep{gao2025reg} formulates powered-distribution sampling as reward optimization over the sampling trajectory, while \citet{kit2025debiasing} study a related debiasing problem for discrete diffusion. IW-CFM instead shifts the main computation to density precomputation and corrector training, after which sampling uses a single learned vector field.

CFG++ \citep{chungcfg++} modifies guidance to reduce off-manifold behavior, while \citet{galashov2025learn} learn a guidance network that better matches the original conditional distribution $p(x|y)$. Their target differs from ours: IW-CFM studies powered conditional distributions and a regularized empirical implementation on ImageNet.

Finally, \citet{karczewski2025diffusion} investigate extreme high-density regions of diffusion models and find that samples with substantially higher likelihood than those produced by standard samplers can be blank, blurry, or cartoon-like. Their results reinforce the distinction between model likelihood and perceptual quality. Our target differs from maximizing either the marginal or conditional likelihood in isolation: $p_\omega(x|y)$ increasingly emphasizes features discriminative of the condition. Their observations nevertheless motivate regularizing the concentrated importance weights in our ImageNet implementation.

\section{Experiments}
Our primary experiments address two questions. First, when the analytical $\omega$-powered density is available, does IW-CFM recover that target more faithfully than CFG with the numerically matched guidance scale $\gamma=\omega$ and the other evaluated alternatives? Second, on ImageNet, how closely do samples from the learned corrector align with weighted and filtered references derived from $r_\omega^{\mathrm{reg}}$? We study the first question on two-dimensional synthetic distributions and the second with a lightweight corrector for a pretrained SiT model. The exploratory analysis of powered-density rankings on real ImageNet data follows in \cref{sec:characterizing_omega_pow}.

\subsection{Synthetic Target Recovery}
We first test IW-CFM against the target density on two-dimensional synthetic datasets: a three-mode Gaussian mixture and a distribution inspired by fractal-like geometries (for details see \cref{app:exp_details_synthetic}). For these tasks, we compute importance weights during training using the minibatch SNIS estimator in \cref{alg:iw_cfm_training}.

As shown in \cref{fig:fractal_cfg,fig:fractal_iwcfm}, CFG with matched guidance $\gamma=\omega$ departs from the analytical target contours, whereas IW-CFM closely follows them. This controlled setting tests the motivating identification in \cref{eq:cfg_powered_mismatch} before the density-estimation and regularization choices needed for ImageNet are introduced. \Cref{app:exp_details_synthetic,app:results_extended_synthetic_iwcfm} provide hyperparameters and additional density visualizations.

We additionally evaluate REG and FKC on the exact conditional and unconditional vector fields of the three-Gaussian problem. Under the tested configurations, neither method closely matches the analytical powered contours. We therefore do not conduct an image-scale comparison, and our ImageNet results make no claim of superiority over REG or FKC. \Cref{app:results_extended_synthetic_reg_fkc} provides implementation details and results.

\subsection{ImageNet Generation}
We scale our method to ImageNet at $256 \times 256$ resolution. The IW-CFM corrector operates in the latent space of the Stable Diffusion VAE used by the pretrained SiT model \citep{rombach2022high}. Log-probability computation, loss reweighting, and sampling all use this latent space; we decode to pixels only for visualization and image-based metrics (CLIP, FID, and DINOv2). The corrector is a parameter-efficient adapter in the frozen Transformer blocks of a pretrained SiT-XL/2 \citep{Ma2024SiTRepo}, following \citet{jensen2025efficient}. We train it with $\omega \sim \mathcal{U}[4, 128]$; the real-data CLIP analysis in \cref{fig:clip-rank-shifted} evaluates a wider range of $\omega$ values.

On ImageNet, IW-CFM samples training data from the precomputed nearest-integer weights and corrects each loss with its interpolated continuous-$\omega$ weight, as described in \cref{sec:scaling_method}. At inference, we optionally apply CFG with scale $\gamma$ to $v^\phi(x,t,y,\omega)$. To support this option, we replace the class token with a null token during 10\% of training iterations. This additional guidance is a sampling heuristic and need not preserve target fidelity. We use unguided inference ($\gamma=1$) to assess the learned distributional shift, $\gamma=1.5$ for FID, FD-DINOv2, and PRDC, and $\gamma=4$ for qualitative visualizations. The ancillary, prompt-specific comparison of IW-CFM and CFG generations in \cref{app:clip_generated_comparison} also uses $\gamma=4$ for both methods. This generated-sample comparison is distinct from the real-data CLIP analysis in \cref{sec:clip_suite}. Throughout the ImageNet results, CFG labels report $\gamma$, whereas IW-CFM labels report the target powering parameter $\omega$; each table or figure caption states the applicable IW-CFM inference scale. We use the ImageNet training set from the baseline SiT model with horizontal-flip augmentation and precompute log-probabilities for both original and flipped images. \Cref{app:distribution_tempering} gives the parameters used by the scaling methods in \cref{sec:scaling_method}.

\subsection{Alignment with Regularized ImageNet References}
\label{sec:imagenet_alignment}
Having established the ImageNet implementation, we first evaluate sample alignment with references derived from the regularized weights $r_\omega^{\mathrm{reg}}$. The separate exploratory analysis in \cref{sec:characterizing_omega_pow} subsequently examines real-data rankings through selected CLIP contrasts.

We report Fr\'echet Inception Distance (FID) \citep{fid}, FD-DINOv2 \citep{stein2023exposing}, and, using Inception features, Precision and Recall \citep{kynkaanniemi2019improved} together with Density and Coverage \citep{naeem2020reliable}. The \enquote{weighted} reference uses $r_\omega^{\mathrm{reg}}$ from \cref{eq:interpolated_target} to compute feature statistics and omits the proposal-only zero-weight re-injection used during training. For the \enquote{filtered} reference, define $A_{\omega,y}=\{i:r_\omega^{\mathrm{reg}}(i|y)>0\}$ from the stored float32 weights; the threshold is exactly zero in float32. For PRDC, we sample 10 instances uniformly without replacement from each $A_{\omega,y}$, producing a class-balanced reference of 10,000 instances. \Cref{app:results_extended_imagenet_target_metrics} reports the same metrics against the standard ImageNet targets as ancillary diagnostics.

Under the guided practical configuration, \cref{fig:fid_wrt_oracle} shows that IW-CFM has substantially lower Fr\'echet distances to these regularized references than either fixed-guidance CFG, $\gamma=1.5$, or numerically matched CFG, $\gamma=\omega$ (in \cref{fig:fid_wrt_oracle_full}), across the evaluated target values $\omega$. The corresponding qualitative comparison in \cref{fig:cfg_vs_iwcfm_imagenet} shows that, for the displayed ImageNet example, matched-guidance CFG loses coherent image structure for $\gamma\geq32$, whereas IW-CFM remains visually coherent at the corresponding target values with inference-time guidance fixed at $\gamma=4$. In addition to integer values of $\omega$, we evaluate IW-CFM at intermediate non-integer values to test the interpolation of the amortized self-normalized weights described in \cref{sec:scaling_method}. Its FID and FD-DINOv2 remain consistent with the neighboring integer targets, indicating that the learned conditioning interpolates effectively between them. In \cref{tab:prdc_inception_filtered}, IW-CFM also achieves the best or near-best PRDC scores with the filtered reference across the evaluated target values. Because these distributional metrics do not rule out memorization, \cref{app:nearest_neighbor_analysis} separately tests for training-data retrieval and reports low rates under its specified test.

The unguided distribution-shift evaluation is reported in \cref{app:metrics_wo_guidance}.

\begin{figure}[t]
    \centering
    \includegraphics[width=0.5\linewidth]{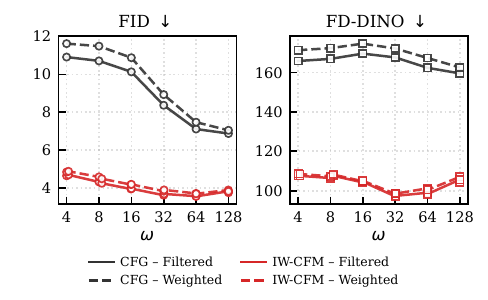}
    \caption{Guided practical evaluation against regularized ImageNet references indexed by the powering parameter $\omega$. IW-CFM targets each displayed $\omega$ and uses inference-time guidance $\gamma=1.5$. The CFG baselines use fixed guidance $\gamma=1.5$. The plots show FID (left) and FD-DINOv2 (right). \enquote{Weighted} uses feature moments weighted by $r_\omega^{\mathrm{reg}}$, while \enquote{Filtered} uses uniform moments over the numerically active weights.}
    \label{fig:fid_wrt_oracle}
\end{figure}

\begin{table}[t]
\begin{center}
\caption{Guided practical evaluation using PRDC \citep{naeem2020reliable} against filtered regularized ImageNet references indexed by $\omega$. Each IW-CFM row targets the listed $\omega$ and uses inference-time guidance $\gamma=1.5$; CFG uses either fixed guidance $\gamma=1.5$ or matched guidance $\gamma=\omega$. The best score in each reference group is bolded. The class-balanced reference contains 10 instances with numerically nonzero stored weights from each of the 1,000 classes.}
\begin{tabular}{llllll}
\toprule
Model and parameter & Reference $\omega$ & Precision & Recall & Density & Coverage \\
\midrule
CFG ($\gamma=1.5$) & 4.0 & 0.55 & 0.50 & 0.95 & 0.90 \\
CFG ($\gamma=4.0$) & 4.0 & \textbf{0.70} & 0.30 & 1.10 & 0.89 \\
IW-CFM ($\omega=4.0$) & 4.0 & 0.69 & \textbf{0.60} & \textbf{1.15} & \textbf{0.98} \\
\midrule
CFG ($\gamma=1.5$) & 8.0 & 0.56 & 0.51 & 0.93 & 0.90 \\
CFG ($\gamma=8.0$) & 8.0 & 0.66 & 0.22 & 1.02 & 0.86 \\
IW-CFM ($\omega=8.0$) & 8.0 & \textbf{0.70} & \textbf{0.61} & \textbf{1.15} & \textbf{0.98} \\
\midrule
CFG ($\gamma=1.5$) & 16.0 & 0.57 & 0.53 & 0.92 & 0.90 \\
CFG ($\gamma=16.0$) & 16.0 & 0.54 & 0.16 & 0.87 & 0.80 \\
IW-CFM ($\omega=16.0$) & 16.0 & \textbf{0.71} & \textbf{0.62} & \textbf{1.15} & \textbf{0.97} \\
\midrule
CFG ($\gamma=1.5$) & 32.0 & 0.58 & 0.57 & 0.88 & 0.87 \\
CFG ($\gamma=32.0$) & 32.0 & 0.29 & 0.14 & 0.41 & 0.48 \\
IW-CFM ($\omega=32.0$) & 32.0 & \textbf{0.70} & \textbf{0.63} & \textbf{1.07} & \textbf{0.97} \\
\midrule
CFG ($\gamma=1.5$) & 64.0 & 0.60 & 0.57 & 0.91 & 0.86 \\
CFG ($\gamma=64.0$) & 64.0 & 0.04 & 0.10 & 0.03 & 0.05 \\
IW-CFM ($\omega=64.0$) & 64.0 & \textbf{0.68} & \textbf{0.64} & \textbf{1.05} & \textbf{0.97} \\
\midrule
CFG ($\gamma=1.5$) & 128.0 & 0.61 & 0.57 & 0.94 & 0.86 \\
CFG ($\gamma=128.0$) & 128.0 & 0.10 & 0.00 & 0.08 & 0.01 \\
IW-CFM ($\omega=128.0$) & 128.0 & \textbf{0.68} & \textbf{0.63} & \textbf{1.06} & \textbf{0.96} \\
\bottomrule
\label{tab:prdc_inception_filtered}
\end{tabular}
\end{center}
\end{table}

\subsection{Ablation of the ImageNet Experiment}

We independently ablate two components of the ImageNet implementation: the numerical accuracy of log-density estimation and the regularization applied to the resulting importance weights. In each experiment, all other settings are held fixed.

\textbf{Number of log-density integration steps.}
We conduct this ablation on a fixed subset of 100 ImageNet categories, described in \cref{app:ablation_num_steps}, while keeping the target-regularization parameters at their default values. For each evaluated step count, we recompute the log-densities and importance weights and train a new corrector from scratch. All models are evaluated against the common weighted and filtered references constructed from the 250-step log-density estimates, using 50 generated samples per category. As shown in \cref{fig:ablation_num_steps}, the results remain relatively close at $\omega\in\{4,8\}$, begin to separate at $\omega=16$, and diverge substantially for $\omega\geq32$. This behavior is consistent with the error-amplification effect discussed in \cref{sec:scaling_method}: an additive error in the estimated log-density ratio is multiplied by $\omega-1$ before exponentiation. Consequently, the numerical precision used for density estimation should be selected according to the largest powering parameter of interest; a step count adequate at small $\omega$ need not remain adequate under stronger powering.

\begin{figure}[t]
    \centering
    \includegraphics[width=0.75\linewidth]{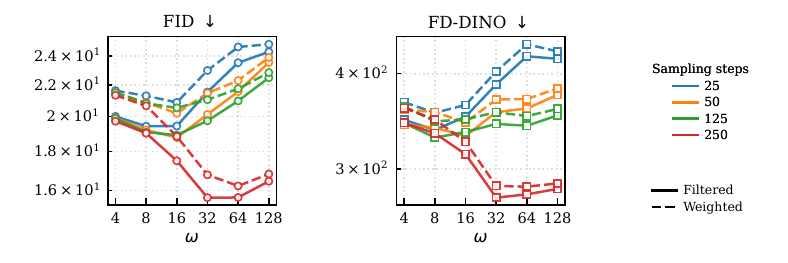}
    \caption{\textbf{Sensitivity to the number of log-density integration steps.}
    For each step count, we recompute the ImageNet log-densities and weights and train a corrector from scratch on the same subset of 100 categories. All models are evaluated against references constructed from the 250-step estimates, using 50 generated samples per category. Solid and dashed lines show distances to the filtered and weighted references, respectively. Both FID and FD-DINOv2 show increasing sensitivity to density-estimation accuracy as $\omega$ grows. Lower is better.}
    \label{fig:ablation_num_steps}
\end{figure}

\textbf{Target regularization parameters.}
Using the default log-density estimates computed with 250 integration steps, we vary the temperature $T\in\{1,1.2,1.5,2.5,5\}$ and the clipped upper fraction $\tau\in\{0.5\%,1\%,2.5\%,5\%\}$, yielding 20 configurations. For each configuration, we train a corrector from scratch and evaluate it using 50 generated samples per ImageNet category against weighted and filtered references constructed with the matching $(T,\tau)$. The two reference constructions exhibit similar qualitative trends in \cref{fig:ablation_regularization_params}.

Increasing $\tau$ consistently reduces the target-relative distances. This is expected: more aggressive clipping attenuates a larger portion of the density-ratio signal and moves the operational target toward the unweighted class-conditional empirical distribution, making the required correction smaller. Because the reference itself changes with $\tau$, these lower distances do not imply that more aggressive clipping is preferable. To retain as much of the powered signal as possible, $\tau$ should instead be the smallest value that still prevents pathological weight concentration.

The effect of $T$ depends on $\omega$: stronger tempering generally worsens alignment at small $\omega$ but becomes comparable or beneficial at large $\omega$. This interaction may reflect both the signal--concentration trade-off and the fixed inference-time guidance used in this evaluation. Strong tempering makes low-$\omega$ targets close to the base distribution, leaving only a weak correction signal, whereas at large $\omega$ it reduces extreme weight concentration while retaining a substantial target shift. The default $(T,\tau)=(1.2,1\%)$ applies comparatively mild regularization and provides stable performance across the evaluated range.

\begin{figure}[t]
    \centering
    \includegraphics[width=\linewidth]{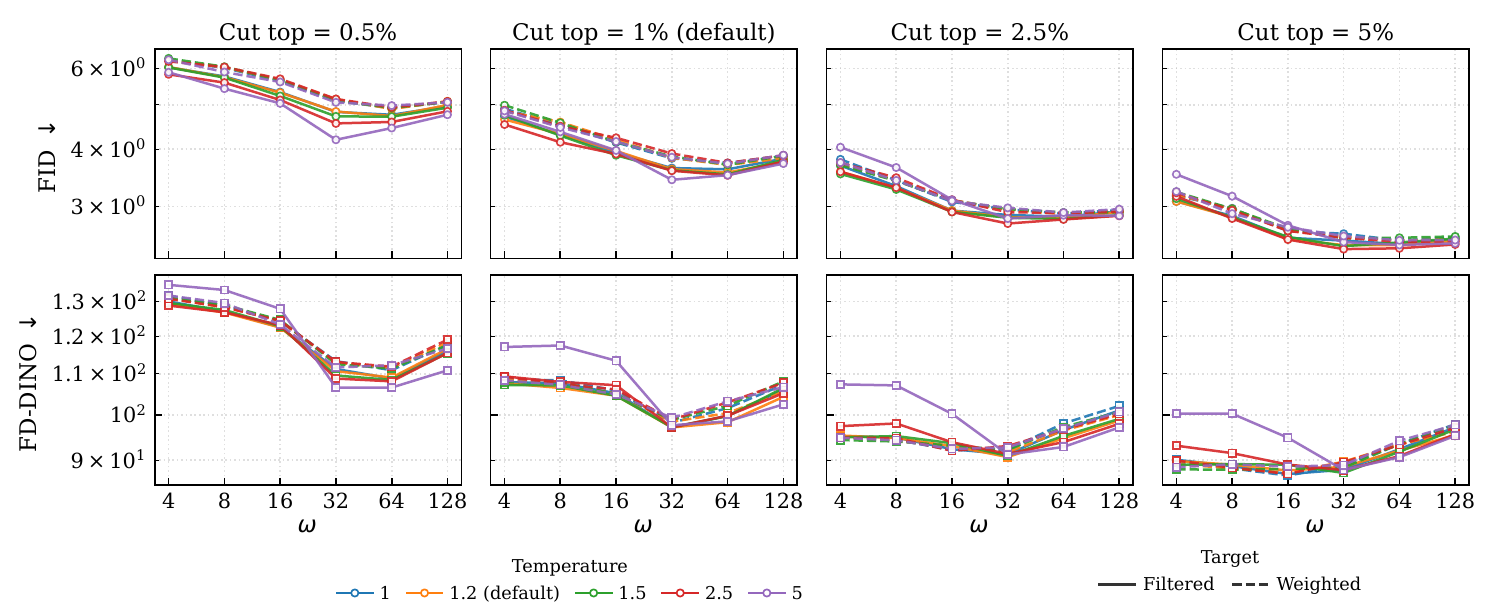}
    \caption{\textbf{Sensitivity to target-regularization parameters.}
    Each column corresponds to a clipped upper fraction $\tau$, and colors denote the temperature $T$. Every corrector is trained using 250-step log-density estimates and evaluated against weighted and filtered references constructed using its matching $(T,\tau)$, with 50 generated samples per ImageNet category. Solid and dashed lines denote the filtered and weighted references, respectively. More aggressive clipping produces an easier but more strongly modified operational target, while the effect of temperature depends on the powering parameter. Lower is better.}
    \label{fig:ablation_regularization_params}
\end{figure}

\section{Exploratory Analysis of Powered-Density Rankings on Real Data}
\label{sec:characterizing_omega_pow}

The preceding experiments provide the primary validation of target recovery and alignment with the regularized ImageNet references. Here, we present a complementary, exploratory characterization of how powering reorganizes real ImageNet examples. For this analysis, we evaluate the conditional log-density $\log p(x|y)$ and unconditional log-density $\log p(x)$ of every ImageNet instance \citep{imagenet1k} using the same pretrained SiT model \citep{ma2024sit} that underlies standard CFG.

\begin{figure}[t]
    \centering
    \includegraphics[width=\linewidth]{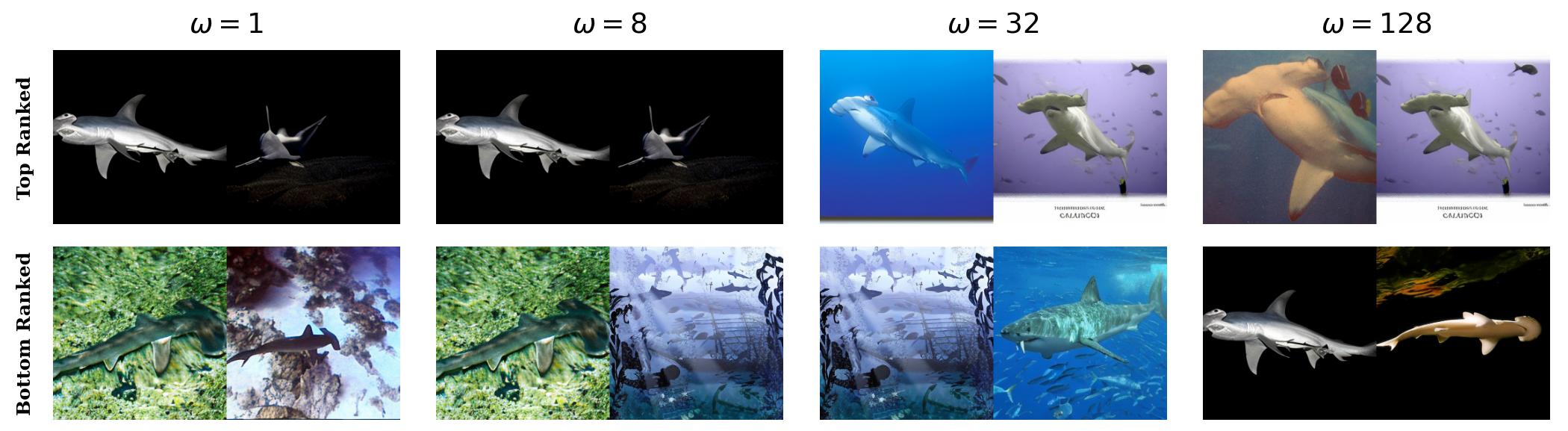}
    \caption{Illustrative powered-density rankings of real images from the ImageNet \enquote{hammerhead shark} category at different values of $\omega$.}
    \label{fig:powered_top_bottom}
\end{figure}

\subsection{Powered-Density Rankings of ImageNet Data}

Within each category, we rank ImageNet instances by the powered-density score defined in \cref{eq:powered_target_intro}; the class-dependent normalizing constant does not affect the ordering. All rankings in this section are computed from the unregularized, model-defined powered log-density score, prior to the temperature scaling and log-weight clipping introduced in \cref{sec:scaling_method}. Those transformations are used only to construct the operational ImageNet training and reference weights. Varying $\omega$ changes which images occupy the extremes of the ranking. \Cref{fig:powered_top_bottom} provides one qualitative example, and \cref{app:omega_pow_imagenet} presents additional selected categories.

\textbf{Illustrative ranking changes.}
In the displayed \enquote{hammerhead shark} examples in \cref{fig:powered_top_bottom}, the top-ranked images at lower $\omega$ have relatively simple compositions, whereas those at higher $\omega$ include more complex scenes and category-representative features. Separately, an image that appears to depict a great white shark with digitally added saber teeth is placed among the bottom-ranked examples at $\omega=32$ but does not appear among the corresponding extremes at lower $\omega$; it is therefore more plausibly interpreted as anomalous or mislabeled than as part of the compositional change. The additional categories in \cref{app:omega_pow_imagenet} exhibit changes of varying form and direction.

\subsection{CLIP Analysis of Real-Data Rankings}
\label{sec:clip_suite}

We use paired CLIP scores to characterize four selected attribute contrasts in real images. The complexity and naturalness contrasts are adopted from \citet{wang2023exploring}; the realism and texture contrasts are constructed following the same paired-prompt design (\cref{tab:prompt_pairs}). For an image $x$ from class $c$ and attribute contrast $a$, the paired score $S_{c,a}(x)$ is the positive-prompt component of a two-way softmax over the temperature-scaled CLIP similarities and therefore lies in $[0,1]$; \cref{app:clip_details} gives the precise definition. We interpret these scores only as prompt-specific contrasts within the chosen CLIP representation, not as general measures of the corresponding attributes.

\textbf{Real-Data Ranking Analysis.}
To accommodate class-dependent baselines, we compare the paired CLIP scores of the top- and bottom-ranked 10\% of real images separately within each ImageNet category. We summarize their separation by Cohen's $d$ \citep{cohen1988statistical}, defined as the difference between the group means divided by their pooled standard deviation. A larger value indicates stronger relative alignment of the top-ranked group with the positive prompt.

To characterize changes relative to the unpowered ranking, we define $\Delta d_{c,a}(\omega)=d_{c,a}(\omega)-d_{c,a}(1)$ for class $c$ and attribute contrast $a$. \Cref{fig:clip-rank-shifted} shows the resulting trajectories across all 1,000 ImageNet classes. A positive value denotes a shift toward the positive-prompt side relative to the class-specific contrast at $\omega=1$; it does not imply that the absolute contrast $d_{c,a}(\omega)$ is positive. The trajectory-density fields retain the distribution of classwise responses, while the orange curves provide the corresponding across-class means.

\begin{figure}[t]
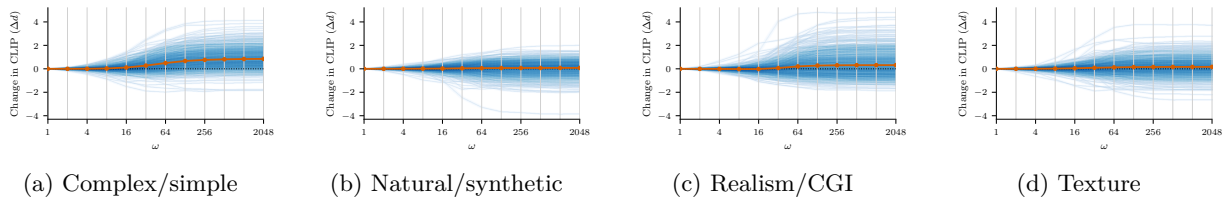

  \centering

  \begin{subfigure}[t]{0.24\linewidth}
    \centering
    \includegraphics[width=\linewidth]{%
      \cliprankdir/clip_real_rank_trajectory_density_wide_shifted_complex_simple.pdf%
    }
    \caption{Complex/simple}
    \label{fig:clip-rank-shift-complex}
  \end{subfigure}%
  \hfill
  \begin{subfigure}[t]{0.24\linewidth}
    \centering
    \includegraphics[width=\linewidth]{%
      \cliprankdir/clip_real_rank_trajectory_density_wide_shifted_natural_synthetic.pdf%
    }
    \caption{Natural/synthetic}
    \label{fig:clip-rank-shift-natural}
  \end{subfigure}%
  \hfill
  \begin{subfigure}[t]{0.24\linewidth}
    \centering
    \includegraphics[width=\linewidth]{%
      \cliprankdir/clip_real_rank_trajectory_density_wide_shifted_realism_vs_cgi.pdf%
    }
    \caption{Realism/CGI}
    \label{fig:clip-rank-shift-realism}
  \end{subfigure}%
  \hfill
  \begin{subfigure}[t]{0.24\linewidth}
    \centering
    \includegraphics[width=\linewidth]{%
      \cliprankdir/clip_real_rank_trajectory_density_wide_shifted_texture.pdf%
    }
    \caption{Texture}
    \label{fig:clip-rank-shift-texture}
  \end{subfigure}

  \caption{Class-dependent changes in the association between powered-density rankings and selected CLIP attributes. Rankings use the unregularized powered log-density score, before temperature scaling or clipping. For each ImageNet class $c$ and attribute contrast $a$, $d_{c,a}(\omega)$ compares the top- and bottom-ranked 10\%, and $\Delta d_{c,a}(\omega)$ reports its change from $\omega=1$. The density fields summarize the trajectories of all 1,000 classes; darker regions contain more trajectories, and orange curves show the across-class means.}
  \label{fig:clip-rank-shifted}
\end{figure}

At the aggregate level, the largest baseline-relative change occurs for the complex/simple contrast, while realism/CGI and texture show smaller positive shifts and natural/synthetic remains close to its $\omega=1$ level. These averages stabilize at larger $\omega$, but the classwise trajectories remain widely dispersed: individual categories exhibit increasing, decreasing, and nonmonotonic responses.

The absolute classwise contrasts from which \cref{fig:clip-rank-shifted} is derived are reported in \cref{fig:clip-rank-absolute}; the main-text trajectories subtract $d_{c,a}(1)$ from these values separately for each class and attribute contrast. The absolute values further qualify the baseline-relative view. In particular, the mean complex/simple contrast is strongly negative at $\omega=1$ and approaches zero as $\omega$ increases. Its positive baseline-relative shift thus primarily reflects attenuation of an initial association with simpler images, not the emergence of a strong absolute preference for complexity. The natural/synthetic contrast remains near zero on average, and the shifts in realism/CGI and texture are comparatively modest.

\section{Discussion}
IW-CFM has several limitations. First, it depends on the quality and support coverage of the base density models. The ImageNet implementation also depends on regularized precomputed weights, nearest-integer sampling, and interpolation. Computing these weights requires a correction dataset. We use that dataset to correct a pretrained model rather than train one from scratch, but performance may change when the correction data are smaller than or shifted from the original training set.

Second, density precomputation requires solving two augmented ODEs, one conditional and one unconditional, for each training instance. This one-time cost enables sampling with a single learned corrector. The ImageNet weight construction has a separate \emph{regularization trade-off}: log-weight clipping and temperature scaling limit concentration on a few training modes, but their settings are chosen manually. Automated regularization schedules remain future work.

Our CLIP analysis is descriptive and prompt-dependent; consequently, neither the class-averaged trajectories nor individual examples establish statistically significant or universal attribute effects.

Finally, our current nested estimator and ImageNet implementation are tailored to categorical conditioning because self-normalization groups samples by a shared class label. The population formulation itself does not require categorical conditions. Applying the practical estimator to high-cardinality or continuous conditions, such as text prompts, will require a different normalization strategy.

\section{Conclusions}
In this work, we introduced Importance-Weighted Conditional Flow Matching (IW-CFM) to learn model-defined $\omega$-powered conditional distributions. With exact densities and normalization, its population objective targets the powered vector field. The practical finite-sample estimator uses self-normalized weights, and the ImageNet implementation uses regularized precomputed integer weights with interpolated per-sample corrections.

The synthetic experiments test the analytical target directly and show that IW-CFM recovers it more faithfully than the evaluated alternatives, including CFG at the numerically matched scale $\gamma=\omega$. On ImageNet, the learned corrector aligns more closely than CFG with the weighted and filtered regularized references. Unguided evaluation provides evidence of a learned distributional shift, while the guided practical results include an additional inference-time CFG heuristic and therefore do not isolate the distributional shift learned by the corrector.

The complementary real-image analysis shows that powered-density rankings and their selected CLIP contrasts vary with $\omega$. The across-class averages exhibit several baseline-relative shifts, while the underlying classwise trajectories are heterogeneous and can be decreasing or nonmonotonic.

IW-CFM learns multiple powered targets with one corrector. Although our implementation uses a SiT backbone, the population formulation applies to other continuous domains. Applying it to video or 3D synthesis would require testing the available density estimates, normalization schemes, and operational targets in those settings.

\subsubsection*{Broader Impact Statement}
This work studies generative modeling of images, with the technical goal of producing realistic and diverse samples. That goal is not socially neutral: greater realism makes synthetic images harder to distinguish from photographic evidence, while greater diversity broadens the range of content that can be generated at scale. These capabilities may support valuable scientific and creative applications, but they may also facilitate deception, impersonation, non-consensual imagery, and other forms of abuse. Improvements in generative modeling should therefore be accompanied by careful evaluation of misuse risks and appropriate safeguards in any deployed system.


\subsubsection*{Acknowledgments}
MF, FL, and AK acknowledge financial support from the Swiss National Science Foundation through the LegoMol project (grant no. 207428). The computations were performed at University of Geneva using Baobab HPC service.

\bibliography{main}

@InProceedings{sohl-dickstein15,
  title = 	 {Deep Unsupervised Learning using Nonequilibrium Thermodynamics},
  author = 	 {Sohl-Dickstein, Jascha and Weiss, Eric and Maheswaranathan, Niru and Ganguli, Surya},
  booktitle = 	 {Proceedings of the 32nd International Conference on Machine Learning},
  pages = 	 {2256--2265},
  year = 	 {2015},
  editor = 	 {Bach, Francis and Blei, David},
  volume = 	 {37},
  series = 	 {Proceedings of Machine Learning Research},
  address = 	 {Lille, France},
  month = 	 {07--09 Jul},
  publisher =    {PMLR}
}

@inproceedings{ddpm,
 author = {Ho, Jonathan and Jain, Ajay and Abbeel, Pieter},
 booktitle = {Advances in Neural Information Processing Systems},
 editor = {H. Larochelle and M. Ranzato and R. Hadsell and M.F. Balcan and H. Lin},
 pages = {6840--6851},
 publisher = {Curran Associates, Inc.},
 title = {Denoising Diffusion Probabilistic Models},
 url = {https://proceedings.neurips.cc/paper_files/paper/2020/file/4c5bcfec8584af0d967f1ab10179ca4b-Paper.pdf},
 volume = {33},
 year = {2020}
}

@inproceedings{
song2021scorebased,
title={Score-Based Generative Modeling through Stochastic Differential Equations},
author={Yang Song and Jascha Sohl-Dickstein and Diederik P Kingma and Abhishek Kumar and Stefano Ermon and Ben Poole},
booktitle={International Conference on Learning Representations},
year={2021},
url={https://openreview.net/forum?id=PxTIG12RRHS}
}

@inproceedings{fid,
 author = {Heusel, Martin and Ramsauer, Hubert and Unterthiner, Thomas and Nessler, Bernhard and Hochreiter, Sepp},
 booktitle = {Advances in Neural Information Processing Systems},
 editor = {I. Guyon and U. Von Luxburg and S. Bengio and H. Wallach and R. Fergus and S. Vishwanathan and R. Garnett},
 pages = {},
 publisher = {Curran Associates, Inc.},
 title = {{GAN}s Trained by a Two Time-Scale Update Rule Converge to a Local {N}ash Equilibrium},
 url = {https://proceedings.neurips.cc/paper_files/paper/2017/file/8a1d694707eb0fefe65871369074926d-Paper.pdf},
 volume = {30},
 year = {2017}
}

@article{Hutchinson, 
author = {Hutchinson, M.F.},
year = {1989}, month = {01},
pages = {1059-1076},
title = {A stochastic estimator of the trace of the influence matrix for {L}aplacian smoothing splines},
volume = {18},
journal = {Communication in Statistics- Simulation and Computation},
doi = {10.1080/03610919008812866} }

@inproceedings{
liu2023flow,
title={Flow Straight and Fast: Learning to Generate and Transfer Data with Rectified Flow},
author={Xingchao Liu and Chengyue Gong and qiang liu},
booktitle={The Eleventh International Conference on Learning Representations },
year={2023},
url={https://openreview.net/forum?id=XVjTT1nw5z}
}

@inproceedings{
bradley2024classifier,
title={Classifier-Free Guidance is a Predictor-Corrector},
author={Arwen Bradley and Preetum Nakkiran},
booktitle={NeurIPS 2024 Workshop on Mathematics of Modern Machine Learning},
year={2024},
url={https://openreview.net/forum?id=2dZswRE2sD}
}

@inproceedings{
chidambaram2024does,
title={What does guidance do? A fine-grained analysis in a simple setting},
author={Muthu Chidambaram and Khashayar Gatmiry and Sitan Chen and Holden Lee and Jianfeng Lu},
booktitle={The Thirty-eighth Annual Conference on Neural Information Processing Systems},
year={2024},
url={https://openreview.net/forum?id=AdS3H8SaPi}
}

@inproceedings{
gao2025reg,
title={{REG}: Rectified Gradient Guidance for Conditional Diffusion Models},
author={Zhengqi Gao and Kaiwen Zha and Tianyuan Zhang and Zihui Xue and Duane S Boning},
booktitle={Forty-second International Conference on Machine Learning},
year={2025},
url={https://openreview.net/forum?id=uNK4ftGdnq}
}

@inproceedings{
lipman2023flow,
title={Flow Matching for Generative Modeling},
author={Yaron Lipman and Ricky T. Q. Chen and Heli Ben-Hamu and Maximilian Nickel and Matthew Le},
booktitle={The Eleventh International Conference on Learning Representations },
year={2023},
url={https://openreview.net/forum?id=PqvMRDCJT9t}
}

@article{karras2024guiding,
  title={Guiding a diffusion model with a bad version of itself},
  author={Karras, Tero and Aittala, Miika and Kynk{\"a}{\"a}nniemi, Tuomas and Lehtinen, Jaakko and Aila, Timo and Laine, Samuli},
  journal={Advances in Neural Information Processing Systems},
  volume={37},
  pages={52996--53021},
  year={2024}
}

@article{dhariwal2021diffusion,
  title={Diffusion models beat {GAN}s on image synthesis},
  author={Dhariwal, Prafulla and Nichol, Alexander},
  journal={Advances in neural information processing systems},
  volume={34},
  pages={8780--8794},
  year={2021}
}

@inproceedings{
ho2022classifier,
title={Classifier-Free Diffusion Guidance},
author={Jonathan Ho and Tim Salimans},
booktitle={NeurIPS 2021 Workshop on Deep Generative Models and Downstream Applications},
year={2021},
url={https://openreview.net/forum?id=qw8AKxfYbI}
}

@article{zheng2023guided,
  title={Guided flows for generative modeling and decision making},
  author={Zheng, Qinqing and Le, Matt and Shaul, Neta and Lipman, Yaron and Grover, Aditya and Chen, Ricky TQ},
  journal={arXiv preprint arXiv:2311.13443},
  year={2023}
}

@article{dao2023flow,
  title={Flow matching in latent space},
  author={Dao, Quan and Phung, Hao and Nguyen, Binh and Tran, Anh},
  journal={arXiv preprint arXiv:2307.08698},
  year={2023}
}

@inproceedings{meng2023distillation,
  title={On distillation of guided diffusion models},
  author={Meng, Chenlin and Rombach, Robin and Gao, Ruiqi and Kingma, Diederik and Ermon, Stefano and Ho, Jonathan and Salimans, Tim},
  booktitle={Proceedings of the IEEE/CVF Conference on Computer Vision and Pattern Recognition},
  pages={14297--14306},
  year={2023}
}

@inproceedings{
kit2025debiasing,
title={Debiasing Guidance for Discrete Diffusion with Sequential Monte Carlo},
author={Lee Cheuk Kit and Paul Jeha and Jes Frellsen and Pietro Lio and Michael Samuel Albergo and Francisco Vargas},
booktitle={Frontiers in Probabilistic Inference: Learning meets Sampling},
year={2025},
url={https://openreview.net/forum?id=vg0dOu3w9g}
}

@inproceedings{
skretafeynman,
title={Feynman-Kac Correctors in Diffusion: Annealing, Guidance, and Product of Experts},
author={Marta Skreta and Tara Akhound-Sadegh and Viktor Ohanesian and Roberto Bondesan and Alan Aspuru-Guzik and Arnaud Doucet and Rob Brekelmans and Alexander Tong and Kirill Neklyudov},
booktitle={Forty-second International Conference on Machine Learning},
year={2025},
url={https://openreview.net/forum?id=Vhc0KrcqWu}
}

@article{chen2018neural,
  title={Neural ordinary differential equations},
  author={Chen, Ricky TQ and Rubanova, Yulia and Bettencourt, Jesse and Duvenaud, David K},
  journal={Advances in neural information processing systems},
  volume={31},
  year={2018}
}

@inproceedings{rainforth2018nesting,
  title={On nesting monte carlo estimators},
  author={Rainforth, Tom and Cornish, Rob and Yang, Hongseok and Warrington, Andrew and Wood, Frank},
  booktitle={International Conference on Machine Learning},
  pages={4267--4276},
  year={2018},
  organization={PMLR}
}

@inproceedings{
zhangenergy,
title={Energy-Weighted Flow Matching for Offline Reinforcement Learning},
author={Shiyuan Zhang and Weitong Zhang and Quanquan Gu},
booktitle={The Thirteenth International Conference on Learning Representations},
year={2025},
url={https://openreview.net/forum?id=HA0oLUvuGI}
}

@inproceedings{radford2021learning,
  title={Learning transferable visual models from natural language supervision},
  author={Radford, Alec and Kim, Jong Wook and Hallacy, Chris and Ramesh, Aditya and Goh, Gabriel and Agarwal, Sandhini and Sastry, Girish and Askell, Amanda and Mishkin, Pamela and Clark, Jack and others},
  booktitle={International conference on machine learning},
  pages={8748--8763},
  year={2021},
  organization={PmLR}
}

@ARTICLE{clipiga2022,
    
AUTHOR={Hentschel, Simon  and Kobs, Konstantin  and Hotho, Andreas },
           
TITLE={CLIP knows image aesthetics},
          
JOURNAL={Frontiers in Artificial Intelligence},
          
VOLUME={Volume 5 - 2022},
  
YEAR={2022},
  
URL={https://www.frontiersin.org/journals/artificial-intelligence/articles/10.3389/frai.2022.976235},
  
DOI={10.3389/frai.2022.976235},
  
ISSN={2624-8212},
}

@inproceedings{zhang2023blind,
  title={Blind image quality assessment via vision-language correspondence: A multitask learning perspective},
  author={Zhang, Weixia and Zhai, Guangtao and Wei, Ying and Yang, Xiaokang and Ma, Kede},
  booktitle={Proceedings of the IEEE/CVF conference on computer vision and pattern recognition},
  pages={14071--14081},
  year={2023}
}

@inproceedings{wang2023exploring,
  title={Exploring clip for assessing the look and feel of images},
  author={Wang, Jianyi and Chan, Kelvin CK and Loy, Chen Change},
  booktitle={Proceedings of the AAAI conference on artificial intelligence},
  volume={37},
  number={2},
  pages={2555--2563},
  year={2023}
}

@article{albergo2025unifying,
  author  = {Michael Albergo and Nicholas M. Boffi and Eric Vanden-Eijnden},
  title   = {Stochastic Interpolants: A Unifying Framework for Flows and Diffusions},
  journal = {Journal of Machine Learning Research},
  year    = {2025},
  volume  = {26},
  number  = {209},
  pages   = {1--80},
  url     = {http://jmlr.org/papers/v26/23-1605.html}
}

@article{jensen2025efficient,
  title={Efficient Distillation of Classifier-Free Guidance using Adapters},
  author={Jensen, Cristian Perez and Sadat, Seyedmorteza},
  journal={Transactions on Machine Learning Research},
  year={2025}
}

@inproceedings{ma2024sit,
  title={Sit: Exploring flow and diffusion-based generative models with scalable interpolant transformers},
  author={Ma, Nanye and Goldstein, Mark and Albergo, Michael S and Boffi, Nicholas M and Vanden-Eijnden, Eric and Xie, Saining},
  booktitle={European Conference on Computer Vision},
  pages={23--40},
  year={2024},
  organization={Springer}
}

@article{stein2023exposing,
  title={Exposing flaws of generative model evaluation metrics and their unfair treatment of diffusion models},
  author={Stein, George and Cresswell, Jesse and Hosseinzadeh, Rasa and Sui, Yi and Ross, Brendan and Villecroze, Valentin and Liu, Zhaoyan and Caterini, Anthony L and Taylor, Eric and Loaiza-Ganem, Gabriel},
  journal={Advances in Neural Information Processing Systems},
  volume={36},
  pages={3732--3784},
  year={2023}
}

@INPROCEEDINGS{imagenet1k,
  author={Deng, Jia and Dong, Wei and Socher, Richard and Li, Li-Jia and Kai Li and Li Fei-Fei},
  booktitle={2009 IEEE Conference on Computer Vision and Pattern Recognition}, 
  title={ImageNet: A large-scale hierarchical image database},
  year={2009},
  volume={},
  number={},
  pages={248-255},
  doi={10.1109/CVPR.2009.5206848}}

@misc{Ma2024SiTRepo,
    author       = {Ma, William and Ryoo, Michael S.},
    title        = {SiT: Exploring Flow Matching for Scalable Text-to-Image Generation - Official Repository},
    howpublished = {\url{https://github.com/willisma/SiT}},
    year         = {2024},
    note         = {GitHub repository}
}

@inproceedings{rombach2022high,
  title={High-resolution image synthesis with latent diffusion models},
  author={Rombach, Robin and Blattmann, Andreas and Lorenz, Dominik and Esser, Patrick and Ommer, Bj{\"o}rn},
  booktitle={Proceedings of the IEEE/CVF conference on computer vision and pattern recognition},
  pages={10684--10695},
  year={2022}
}

@inproceedings{
chungcfg++,
title={{CFG}++: Manifold-constrained Classifier Free Guidance for Diffusion Models},
author={Hyungjin Chung and Jeongsol Kim and Geon Yeong Park and Hyelin Nam and Jong Chul Ye},
booktitle={The Thirteenth International Conference on Learning Representations},
year={2025},
url={https://openreview.net/forum?id=E77uvbOTtp}
}

@article{galashov2025learn,
  title={Learn to guide your diffusion model},
  author={Galashov, Alexandre and Pokle, Ashwini and Doucet, Arnaud and Gretton, Arthur and Delbracio, Mauricio and De Bortoli, Valentin},
  journal={arXiv preprint arXiv:2510.00815},
  year={2025}
}

@book{cohen1988statistical,
  title={Statistical Power Analysis for the Behavioral Sciences},
  author={Cohen, Jacob},
  year={1988},
  edition={2nd},
  publisher={Routledge},
  doi={10.4324/9780203771587},
  address={New York}
}

@article{
gumemorization,
title={On Memorization in Diffusion Models},
author={Xiangming Gu and Chao Du and Tianyu Pang and Chongxuan Li and Min Lin and Ye Wang},
journal={Transactions on Machine Learning Research},
issn={2835-8856},
year={2025},
url={https://openreview.net/forum?id=D3DBqvSDbj},
note={}
}

@inproceedings{karczewski2025diffusion,
  title={Diffusion models as cartoonists: The curious case of high density regions},
  author={Karczewski, Rafa{\l} and Heinonen, Markus and Garg, Vikas},
  booktitle={International Conference on Learning Representations},
  volume={2025},
  pages={12882--12915},
  year={2025}
}

@article{kynkaanniemi2019improved,
  title={Improved precision and recall metric for assessing generative models},
  author={Kynk{\"a}{\"a}nniemi, Tuomas and Karras, Tero and Laine, Samuli and Lehtinen, Jaakko and Aila, Timo},
  journal={Advances in neural information processing systems},
  volume={32},
  year={2019}
}

@inproceedings{naeem2020reliable,
  title={Reliable fidelity and diversity metrics for generative models},
  author={Naeem, Muhammad Ferjad and Oh, Seong Joon and Uh, Youngjung and Choi, Yunjey and Yoo, Jaejun},
  booktitle={International conference on machine learning},
  pages={7176--7185},
  year={2020},
  organization={PMLR}
}

@misc{openai_adm_evaluations,
  author       = {{OpenAI}},
  title        = {{OpenAI's Ablated Diffusion Model (ADM) Evaluation Code and Sample Batches}},
  year         = {2021},
  howpublished = {\url{https://github.com/openai/guided-diffusion/tree/main/evaluations}},
  note         = {Part of the openai/guided-diffusion repository; includes FID, sFID, Precision, Recall, and Inception Score evaluation utilities and precomputed sample batches. Accessed: 2026-05-19}
}

@misc{gao2025regcode,
  author       = {Gao, Zhengqi and Zha, Kaiwen and Zhang, Tianyuan
                  and Xue, Zihui and Boning, Duane S.},
  title        = {{REG}: Rectified Gradient Guidance for Conditional Diffusion Models},
  year         = {2025},
  howpublished = {\url{https://github.com/zhengqigao/REG}},
  note         = {GitHub repository}
}

@misc{skreta2025fkccode,
  author       = {Skreta, Marta and Akhound-Sadegh, Tara and Ohanesian, Viktor
                  and Bondesan, Roberto and Aspuru-Guzik, Al{\'a}n
                  and Doucet, Arnaud and Brekelmans, Rob and Tong, Alexander
                  and Neklyudov, Kirill},
  title        = {Feynman--Kac Correctors in Diffusion: Annealing, Guidance,
                  and Product of Experts},
  year         = {2025},
  howpublished = {\url{https://github.com/martaskrt/fkc-diffusion}},
  note         = {GitHub repository}
}

@inproceedings{guo2017calibration,
  title={On calibration of modern neural networks},
  author={Guo, Chuan and Pleiss, Geoff and Sun, Yu and Weinberger, Kilian Q},
  booktitle={International conference on machine learning},
  pages={1321--1330},
  year={2017},
  organization={PMLR}
}
\bibliographystyle{tmlr}

\newpage
\appendix
\onecolumn

\section{Theoretical Analysis and Proofs}
\label{app:theory}

\subsection{Gradient Equivalence Proof}
\label{app:gradient_proof}

We first establish the population identity for the exact IW-CFM objective $\mathcal{L}_{\text{IW-CFM}}^\star$. We then relate it to the finite-batch SNIS estimator used in practice.

\begin{theorem}
Assume exact conditional and unconditional densities and access to the exact normalizer $Z(y,\omega)$. Then the gradient $\nabla_\phi \mathcal{L}_{\text{IW-CFM}}^\star(\phi, \omega)$, defined using the exact density ratio $w_\omega^\star$ in \cref{eq:exact_importance_weight}, equals the gradient of the flow-matching objective targeting $p_\omega(x|y)$.
\end{theorem}

\begin{proof}
The ideal objective targeting the powered distribution is
\begin{equation*}
\mathcal{L}_{\text{Target}}(\phi, \omega) = \int p(y) \int p_\omega(x_1|y) \mathcal{L}_{\text{cFM}}(x_1, \phi) dx_1 dy
\end{equation*}
Using the exact importance-sampling identity $p_\omega(x_1|y) = w_\omega^\star(x_1, y) p(x_1|y)$ gives
\begin{equation*}
\nabla_\phi \mathcal{L}_{\text{Target}} = \int p(y) \int p(x_1|y) w_\omega^\star(x_1, y) \nabla_\phi \mathcal{L}_{\text{cFM}}(x_1, \phi) dx_1 dy
\end{equation*}
\begin{equation*}
\nabla_\phi \mathcal{L}_{\text{Target}} = \mathbb{E}_{p(y), p(x_1|y)} \left[ w_\omega^\star(x_1, y) \nabla_\phi \mathcal{L}_{\text{cFM}}(x_1, \phi) \right]
\end{equation*}
By the flow-matching gradient identity \citep{lipman2023flow}, the conditional objective that regresses to $u_t$ has the same gradient as the marginal objective that regresses to the target vector field. Thus, training on $p(x_1|y)$ with exact weights $w_\omega^\star$ recovers the vector field $v^{\phi}(x, t, y, \omega)$ for $p_\omega$.
\end{proof}

The practical estimator in \cref{alg:iw_cfm_training} substitutes the finite-batch weights $\widehat{w}_{\omega,K}$ from \cref{eq:snis_importance_weight} for $w_\omega^\star$. This SNIS estimator is biased at finite $K$; under the standard importance-sampling regularity conditions, its gradient converges to the population gradient above as $K\to\infty$ and is therefore only asymptotically unbiased.

\subsection{Bias-Variance Analysis of Amortized Estimators}
\label{app:bias_variance}

\subsubsection{Structural Bias in Global Normalization}
In \cref{sec:amortized_powering}, we use nested rather than global normalization because global normalization changes the effective distribution over $(y,\omega)$ pairs. Let $g_\phi(y,\omega)$ denote the expected gradient of the ideal flow-matching objective targeting $p_\omega(x|y)$ for a fixed $(y,\omega)$. Jointly normalizing importance weights across different conditions and powering parameters yields the population gradient
\begin{equation}
\nabla_\phi \mathcal{L}_{\text{Global}}
=
\frac{\mathbb{E}_{p(y)p(\omega)}\left[Z(y,\omega)g_\phi(y,\omega)\right]}
{\mathbb{E}_{p(y)p(\omega)}\left[Z(y,\omega)\right]},
\label{eq:global_bias}
\end{equation}
where
\begin{equation}
Z(y,\omega)
= \int p(x|y)^\omega p(x)^{1-\omega}\,dx
= \exp\!\left((\omega-1)D_\omega\bigl(p(x|y)\Vert p(x)\bigr)\right).
\end{equation}
Here, $D_\omega(P\Vert Q) \coloneqq \frac{1}{\omega-1}\log\int P(x)^\omega Q(x)^{1-\omega}\,dx$ is the R\'enyi divergence of order $\omega$. For $\omega>1$, $D_\omega(P\Vert Q)\geq 0$, so $Z(y,\omega)\geq 1$. When finite, $Z(y,\omega)$ is also non-decreasing in $\omega$ for fixed $y$. Global normalization therefore induces the effective outer distribution
\begin{equation}
q(y,\omega)
= \frac{p(y)p(\omega)Z(y,\omega)}
{\mathbb{E}_{p(y)p(\omega)}[Z(y,\omega)]}.
\end{equation}
This distribution gives more influence to $(y,\omega)$ pairs with larger $Z(y,\omega)$. Nested normalization operates within each $(y,\omega)$ group and preserves the intended outer distribution $p(y)p(\omega)$.

\subsubsection{Statistical Bias-Variance Trade-off in Nested Normalization}
Nested normalization uses $M$ condition--parameter pairs $(y,\omega)$, each with a subbatch of $K$ data samples and finite-batch weights $\widehat{w}_{\omega,K}$. Under standard finite-variance assumptions, the leading variance terms scale as $\mathcal{O}(1/M)$ across $(y,\omega)$ pairs and $\mathcal{O}(1/(MK))$ for within-pair Monte Carlo noise. With fixed $B=MK$, these terms become $\mathcal{O}(K/B)$ and $\mathcal{O}(1/B)$, respectively. Increasing $K$ reduces finite-$K$ SNIS bias but leaves fewer distinct $(y,\omega)$ pairs for estimating the outer expectation. This allocation trade-off is characteristic of nested Monte Carlo estimators \citep{rainforth2018nesting}.

\FloatBarrier
\section{Extended Characterization of $\omega$-Powering}
\label{app:omega_pow}

\subsection{Univariate Toy Example}
\label{app:omega_pow_1d}

\begin{figure}[h]
    \centering
    \begin{overpic}[width=0.75\linewidth]{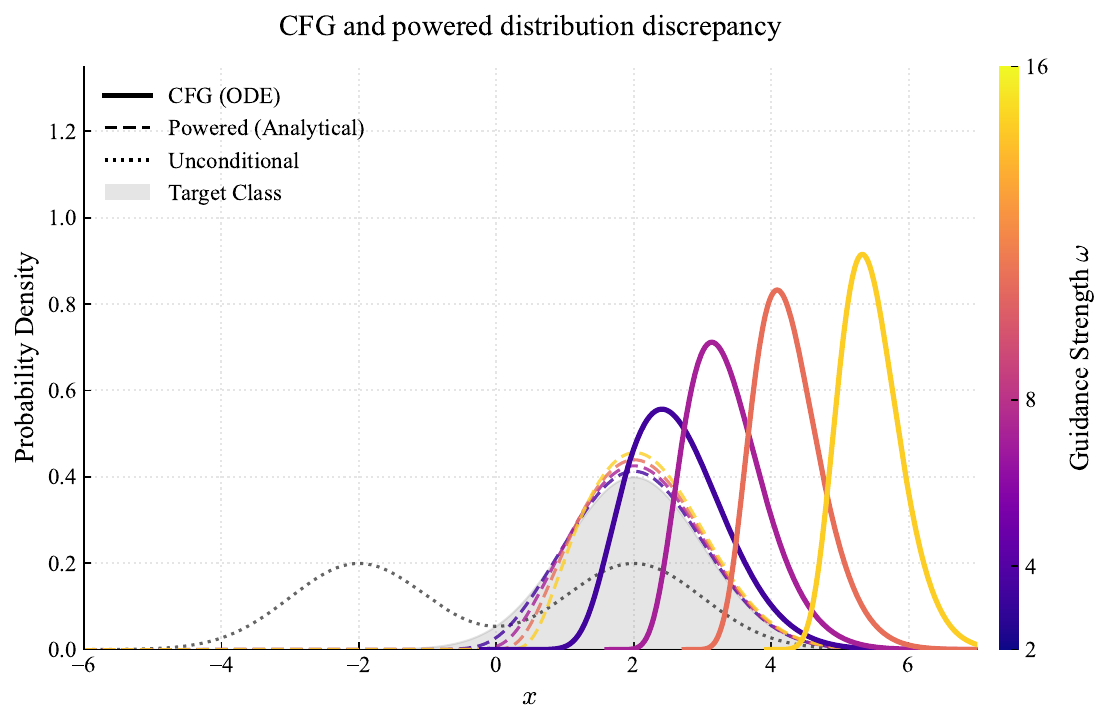}
        \put(96,23){\rotatebox{90}{\colorbox{white}{\scriptsize Matched guidance $\gamma=\omega$}}}
    \end{overpic}
    \caption{\textbf{Comparison of the $\omega$-powered distribution and matched-guidance CFG on a univariate Gaussian mixture.} The $\omega$-powered distribution (dashed) concentrates around the target-class mode. The CFG distribution (solid), evaluated at the numerically matched guidance scale $\gamma=\omega$, shifts probability mass toward low-density regions as the matched values increase.}
    \label{fig:cfg_vs_powered_1d}
\end{figure}

In \cref{fig:cfg_vs_powered_1d}, we illustrate the discrepancy between the $\omega$-powered distribution and CFG using a univariate Gaussian mixture ($\mu_1=-2$, $\mu_2=2$, $\sigma_1=\sigma_2=1$, $\pi_1=\pi_2=0.5$). We define the target class as the component centered at $\mu_2=2$, vary the powering parameter over $\omega\in[2,16]$, and compare each target $p_\omega$ with CFG at the matched guidance scale $\gamma=\omega$.

This example shows two differences:

\begin{itemize}

\item \textbf{Powered target.} Increasing $\omega$ moves mass away from the overlap region near the decision boundary and toward the mode of the target class.

\item \textbf{Matched CFG.} Increasing $\gamma=\omega$ also moves mass away from the decision boundary, but toward a low-density region beyond the target-class mode. Because the Gaussian components have full support, this is a low-density shift rather than movement outside the support.
\end{itemize}

The discrepancy grows with the matched values: the powered distribution concentrates around the target-class mode, while CFG moves farther into the tail. Thus, directly approximating $p_\omega(x|y)$ differs from setting the CFG guidance scale to the same numerical value.

\subsection{Powered-Density Rankings across ImageNet Categories}
\label{app:omega_pow_imagenet}

\Cref{fig:rankings_optimized_4_nperblock_12,fig:rankings_optimized_33_nperblock_12,fig:rankings_optimized_139_nperblock_12,fig:rankings_optimized_414_nperblock_12,fig:rankings_optimized_417_nperblock_12,fig:rankings_optimized_728_nperblock_12,fig:rankings_optimized_757_nperblock_12,fig:rankings_optimized_770_nperblock_12,fig:rankings_optimized_873_nperblock_12} show the top- and bottom-ranked 12 images for nine selected ImageNet categories at different $\omega$ values. These examples illustrate heterogeneous changes in the rankings rather than a universal progression with $\omega$.

In several displayed categories, higher $\omega$ favors more complex scenes or different category-representative features. The pattern is neither uniform nor necessarily monotonic. For example, in the \enquote{recreational vehicle, RV, R.V.} category (\cref{fig:rankings_optimized_757_nperblock_12}), low $\omega$ values favor images with a prominent, centrally located vehicle, whereas high $\omega$ values favor landscapes in which the vehicle occupies only a small part of the scene.

For the \enquote{triumphal arch} category in \cref{fig:rankings_optimized_873_nperblock_12}, the dataset contains arches from several locations, but the top-ranked images at high $\omega$ are different views of the Arch of Constantine next to the Colosseum in Rome.

\begin{figure}[h]
    \centering
    \includegraphics[width=\linewidth]{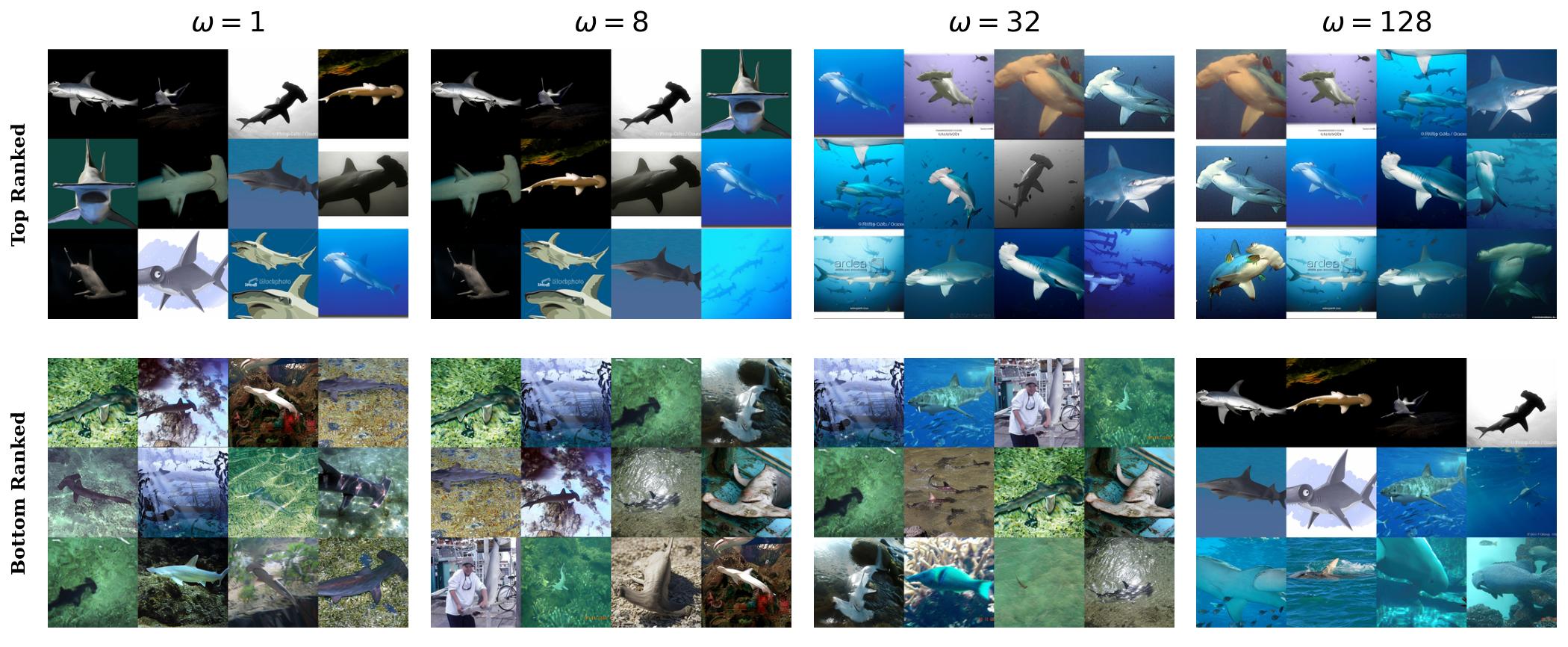}
    \caption{hammerhead shark}
    \label{fig:rankings_optimized_4_nperblock_12}
\end{figure}

\begin{figure}[h]
    \centering
    \includegraphics[width=\linewidth]{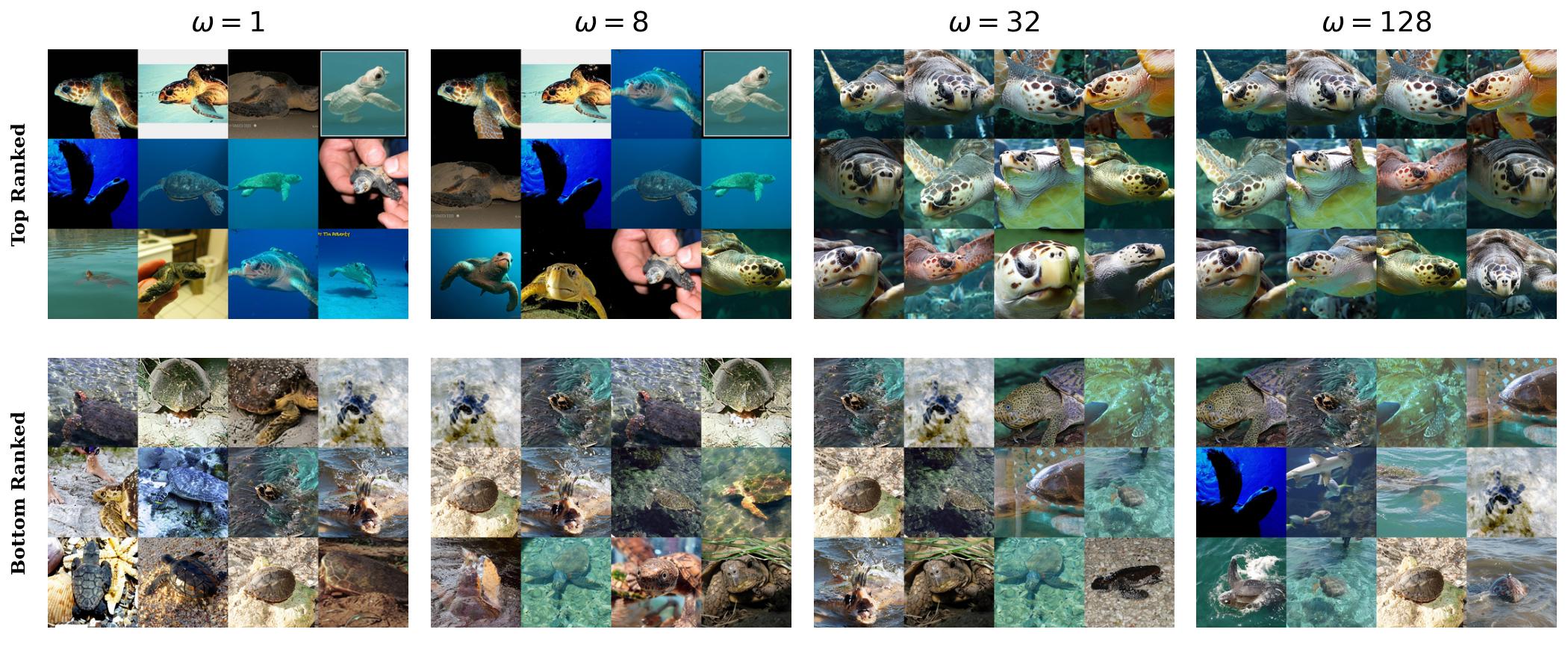}
    \caption{loggerhead turtle (lat. Caretta caretta)}
    \label{fig:rankings_optimized_33_nperblock_12}
\end{figure}

\begin{figure}[h]
    \centering
    \includegraphics[width=\linewidth]{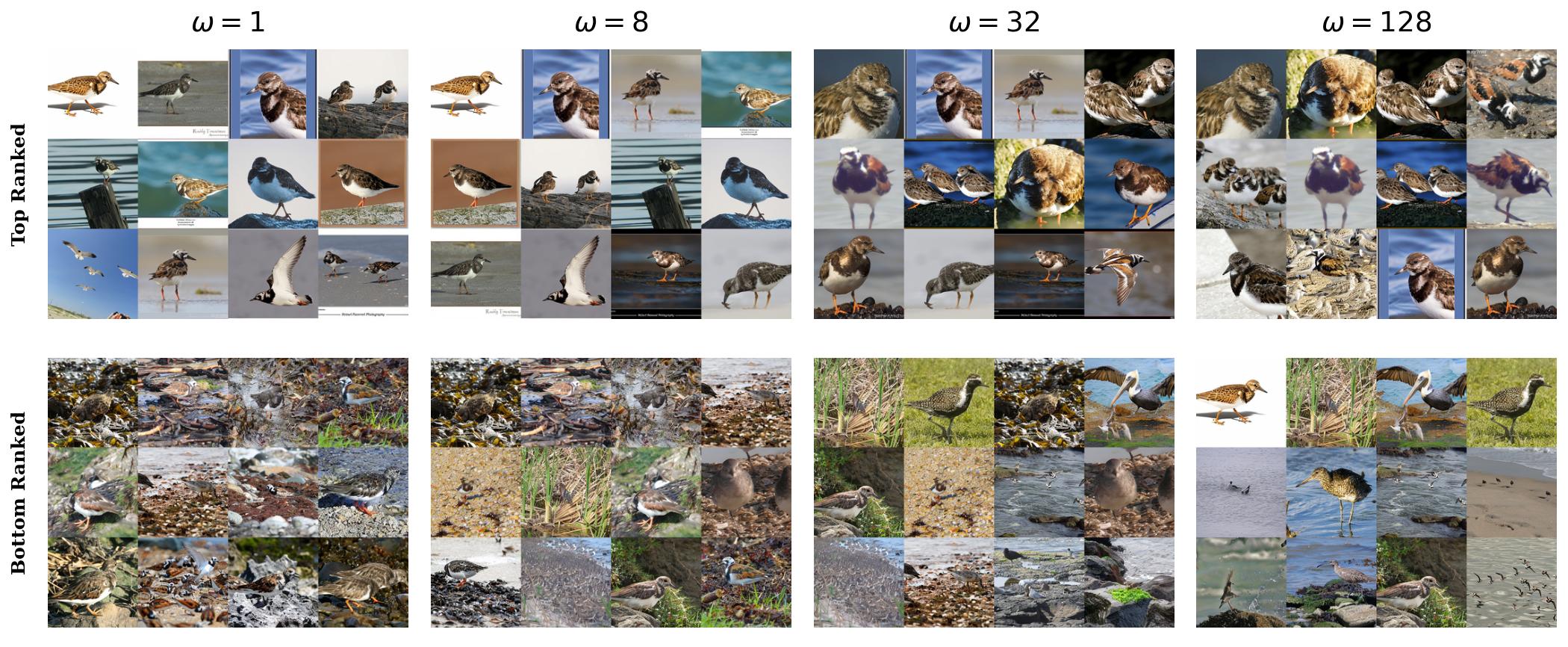}
    \caption{ruddy turnstone (lat. Arenaria interpres)}
    \label{fig:rankings_optimized_139_nperblock_12}
\end{figure}

\begin{figure}[h]
    \centering
    \includegraphics[width=\linewidth]{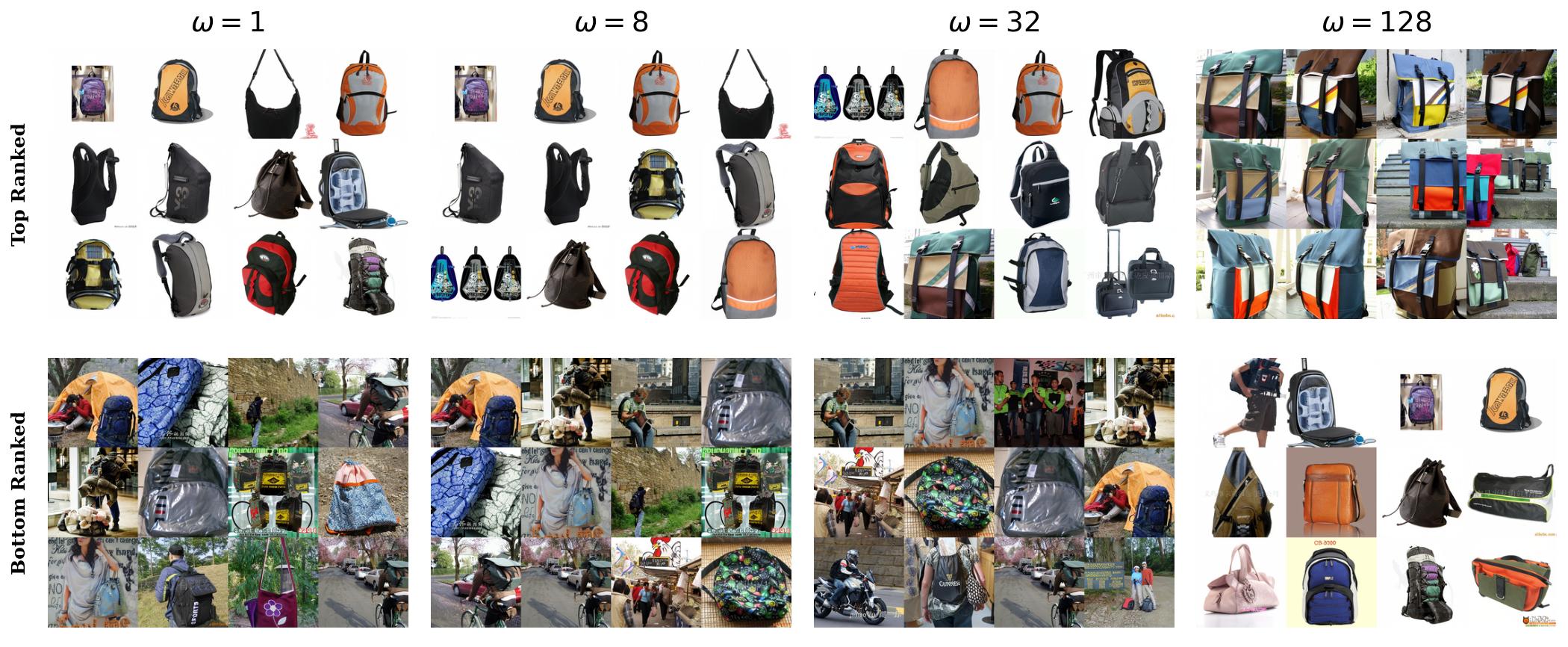}
    \caption{backpack, back pack, knapsack, packsack, rucksack, haversack}
    \label{fig:rankings_optimized_414_nperblock_12}
\end{figure}

\begin{figure}[h]
    \centering
    \includegraphics[width=\linewidth]{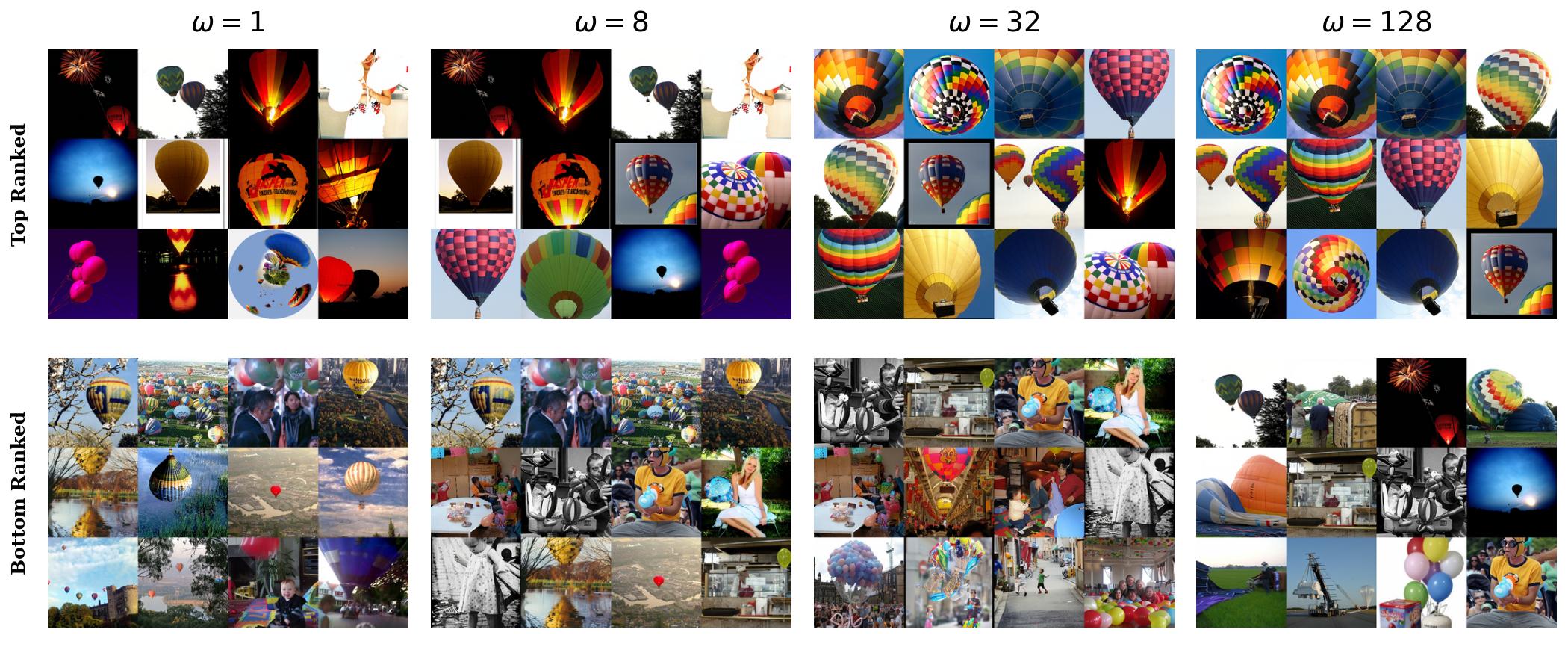}
    \caption{balloon}
    \label{fig:rankings_optimized_417_nperblock_12}
\end{figure}

\begin{figure}[h]
    \centering
    \includegraphics[width=\linewidth]{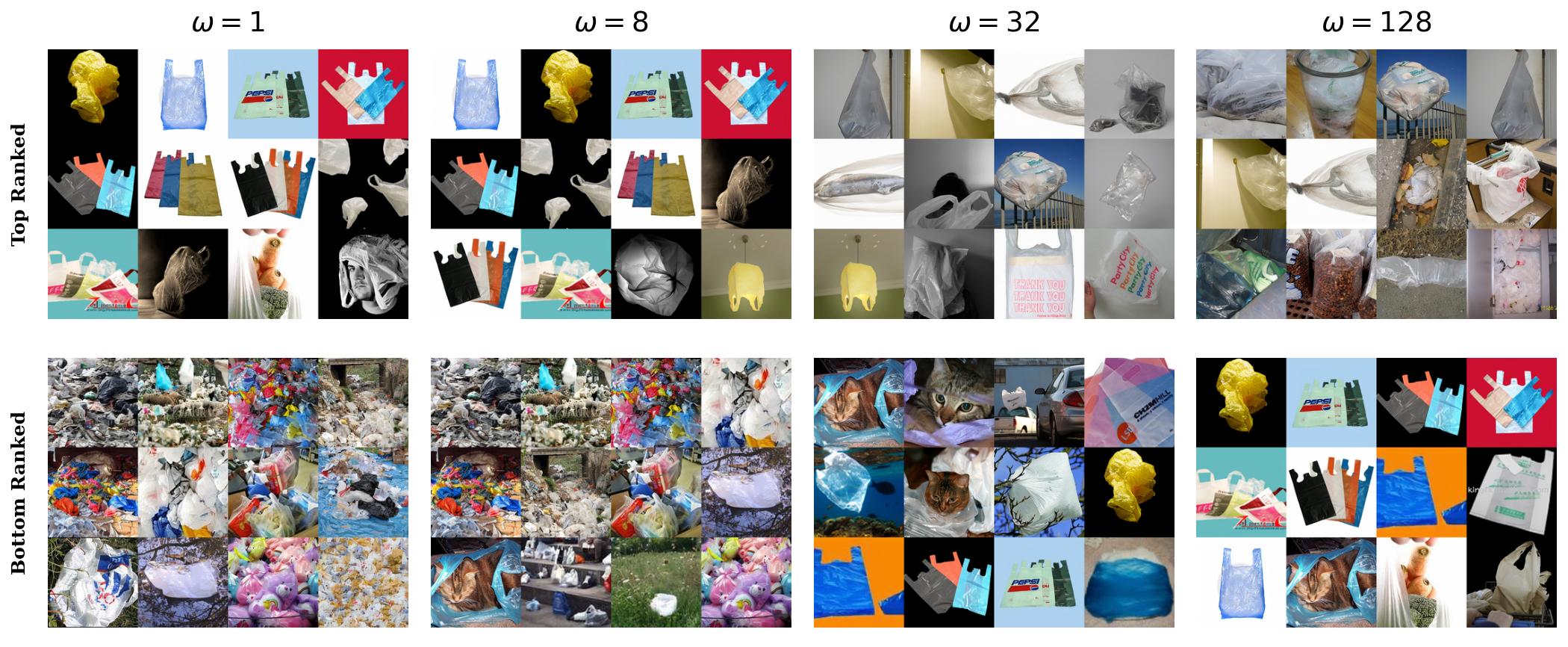}
    \caption{plastic bag}
    \label{fig:rankings_optimized_728_nperblock_12}
\end{figure}

\begin{figure}[h]
    \centering
    \includegraphics[width=\linewidth]{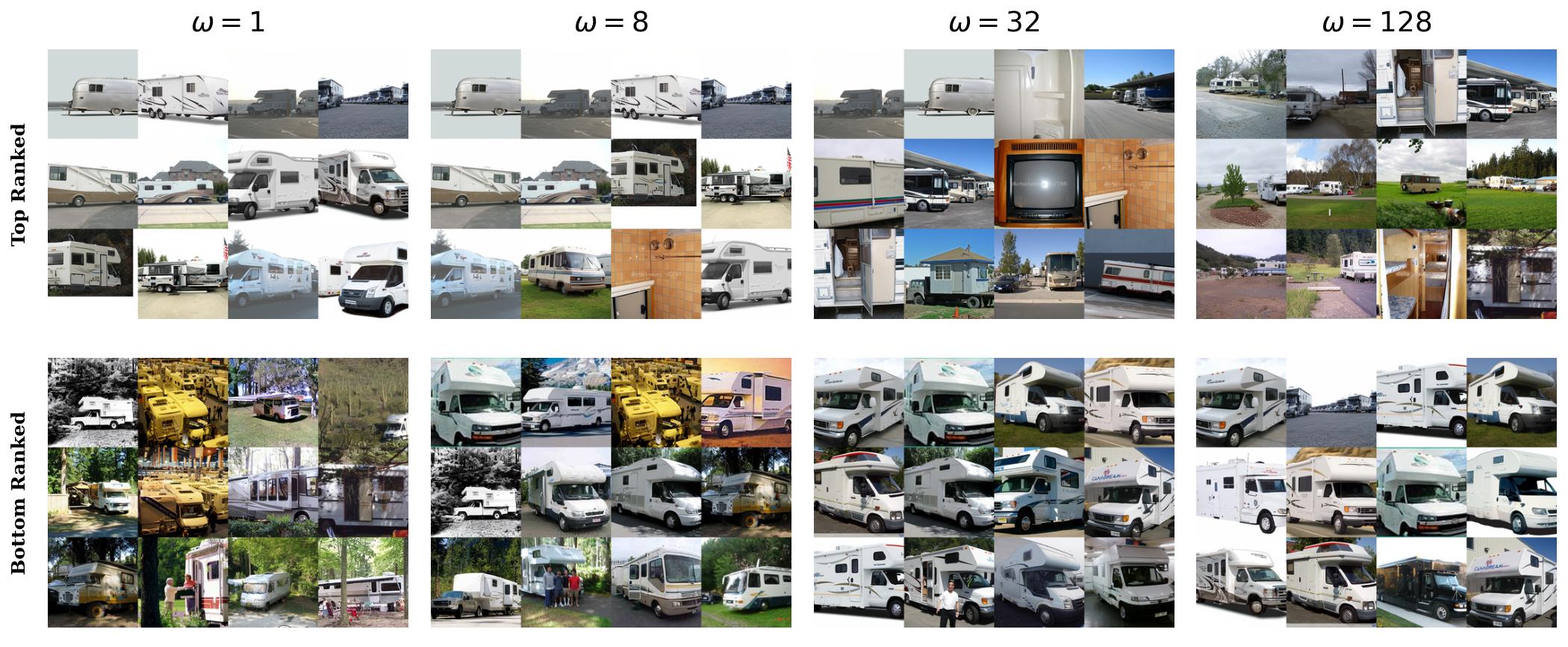}
    \caption{recreational vehicle, RV, R.V.}
    \label{fig:rankings_optimized_757_nperblock_12}
\end{figure}

\begin{figure}[h]
    \centering
    \includegraphics[width=\linewidth]{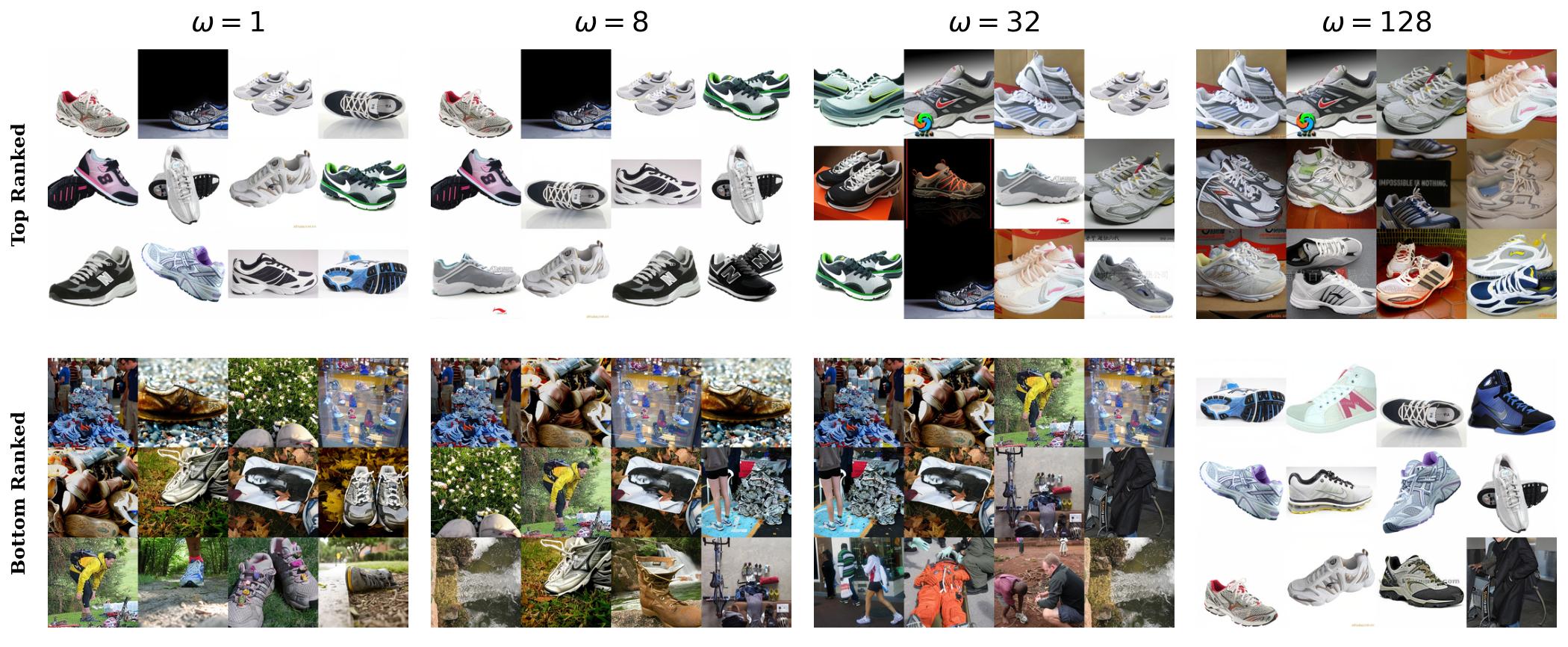}
    \caption{running shoe}
    \label{fig:rankings_optimized_770_nperblock_12}
\end{figure}

\begin{figure}[h]
    \centering
    \includegraphics[width=\linewidth]{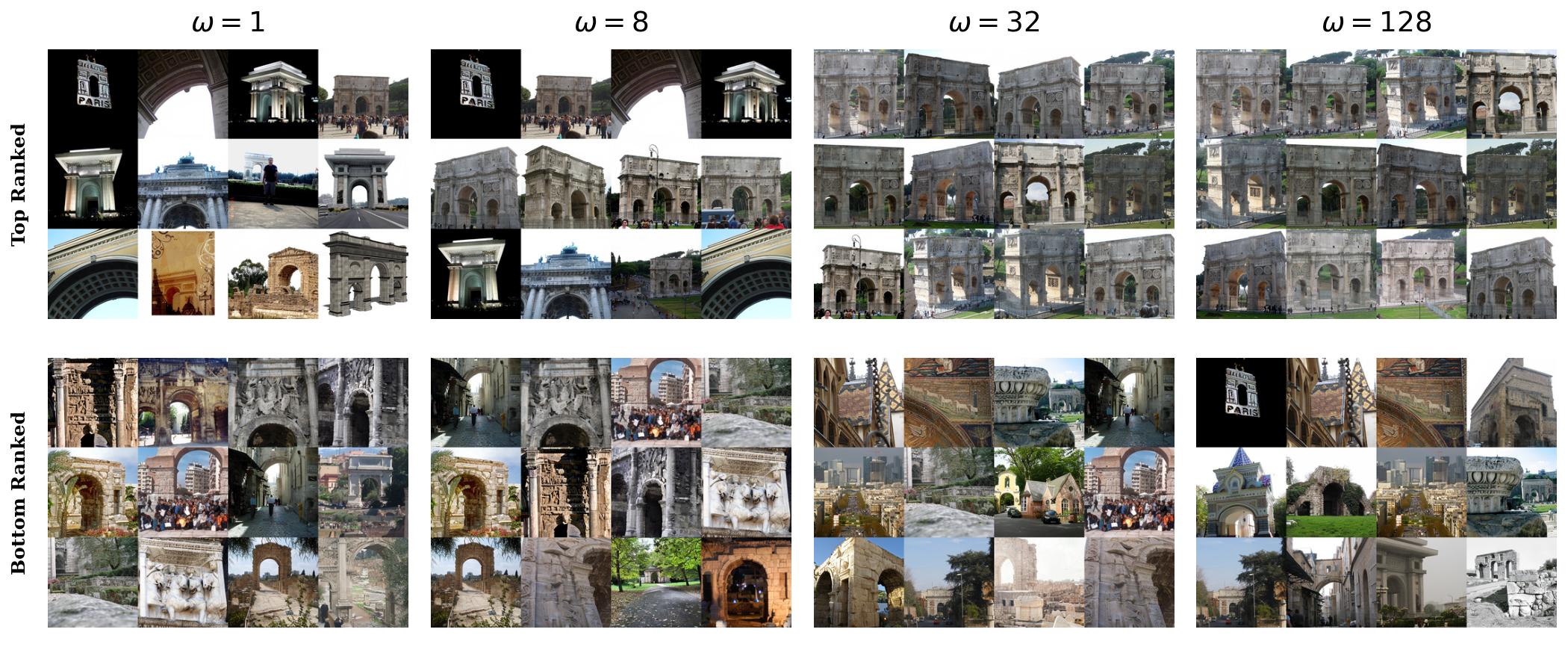}
    \caption{triumphal arch}
    \label{fig:rankings_optimized_873_nperblock_12}
\end{figure}

\FloatBarrier
\section{Implementation Details and Setup}
\label{app:exp_details}

\subsection{Synthetic}
\label{app:exp_details_synthetic}

The synthetic tasks test whether IW-CFM recovers $p_\omega(x|y)$ when the ground-truth density and vector field are available analytically. We use a three-mode Gaussian mixture and a distribution with a tree-like branching structure.

\textbf{Analytical quantities.} The synthetic implementations provide

\begin{itemize}
    \item conditional and unconditional log-densities, $\log p(x|y)$ and $\log p(x)$;
    \item conditional and unconditional ground-truth vector fields, $v(x_t, t, y)$ and $v(x_t, t)$; and
    \item exact conditional and unconditional sampling procedures.
\end{itemize}

\textbf{Distribution Parameters.}

\begin{itemize}
    \item Three-Gaussian mixture: three components with means $\mu \in \{[-1, -0.866], [0, 0.866], [1, -0.866]\}$ and identity covariance.
    \item Fractal distribution: a reimplementation of the tree-like structure in \citet{karras2024guiding}. Each category is a mixture of 1,016 Gaussian components arranged along the branching structure. We use slightly modified parameters.
\end{itemize}

\textbf{Corrector Architecture.} We train a lightweight affine correction to the analytical conditional vector field:
\begin{equation}
v^\phi(x, t, y, \omega) = v(x_t, t, y) \odot \text{scale}(x_t, t, y, \omega) + \text{shift}(x_t, t, y, \omega)
\end{equation}
where a multilayer perceptron maps $(x_t, t, y, \omega)$ to $\text{scale}$ and $\text{shift}$.

\textbf{Training and Hyperparameters.} We train the corrector for 8,192 iterations with a total batch size of 8,192, partitioned into $M=256$ $(y, \omega)$ pairs and $K=32$ samples per pair for nested normalization. For both tasks, $p(\omega)$ is uniform on $[1, 8]$.

\subsection{CLIP Evaluation Suite Details}
\label{app:clip_details}

For each ImageNet class $c$ and attribute contrast $a$, let $p_{c,a}^{+}$ and $p_{c,a}^{-}$ denote the positive and negative prompts listed in \cref{tab:prompt_pairs}. For an image $x$ from class $c$, we define the paired CLIP score by applying a two-way softmax to the temperature-adjusted cosine similarities:
\begin{equation}
    S_{c,a}(x)
    =
    \frac{\exp\!\left(\cos(x,p_{c,a}^{+})/\tau_{\mathrm{CLIP}}\right)}
    {\exp\!\left(\cos(x,p_{c,a}^{+})/\tau_{\mathrm{CLIP}}\right)
    +\exp\!\left(\cos(x,p_{c,a}^{-})/\tau_{\mathrm{CLIP}}\right)}.
\end{equation}
Here, $\cos(x,p)$ is the CLIP cosine similarity between image $x$ and prompt $p$, and $\tau_{\mathrm{CLIP}}=0.01$. During evaluation, \{\texttt{class\_name\}} is replaced with the name of class $c$.

\begin{table}[h]
\centering
\caption{Prompt pairs for the CLIP evaluation suite. During evaluation, \{\texttt{class\_name\}} is replaced with the ImageNet class name. The \texttt{complex/simple} and \texttt{natural/synthetic} pairs are from \citet{wang2023exploring}; the \texttt{realism/CGI} and \texttt{texture} pairs are inspired by its paired-prompt design.}
\label{tab:prompt_pairs}
\vskip 0.15in
\renewcommand{\arraystretch}{1.5}
\begin{tabular}{l p{0.35\linewidth} p{0.35\linewidth}}
\toprule
\textbf{Contrast} & \textbf{Positive Prompt} & \textbf{Negative Prompt} \\
\midrule
\texttt{complex/simple} & Complex photo. & Simple photo. \\
\texttt{natural/synthetic} & Natural photo. & Synthetic photo. \\
\texttt{realism/CGI} & An authentic, real-world photo of a \{\texttt{class\_name\}}, captured by a camera. & A synthetic, computer-generated image of a \{\texttt{class\_name\}}, a digital rendering. \\
\texttt{texture} & A photo of a \{\texttt{class\_name\}} with realistic, detailed surface textures. & A photo of a \{\texttt{class\_name\}} with an unnaturally smooth, waxy, or plastic-like texture. \\
\bottomrule
\end{tabular}
\end{table}

\subsection{Base Model Density Computational Cost}
\label{app:exp_details_logp_cost}

For the main results presented in the paper the augmented ODE for log-probability evaluation (\cref{eq:log_prob_flow}) uses a first-order Euler solver with 250 steps. The computational cost of processing the whole ImageNet dataset is approximately 2100 GPUhours on an RTX 4090. The resulting log-densities in float32 precision take approximately 20MB of storage. Since we use the horizontal flip data augmentation, both the computational and storage cost is doubled.

\subsection{Base Model Density Ablation}
\label{app:exp_details_logp}

We assess the numerical sensitivity of the base-model log-density estimates used for importance weighting.

\textbf{Numerical Consistency and Ranking Stability.} We vary the integration step count on a random subset of $2^{17}$ ImageNet instances. \Cref{fig:logp_steps_heatmap} shows strong linear correlation between lower-step and 250-step log-density estimates. Because the importance weights and ranking analysis depend on instance order, we also measure Spearman rank correlation ($\rho$) against the 250-step estimates.

\begin{figure}[h]
    \centering
    \includegraphics[width=0.75\linewidth]{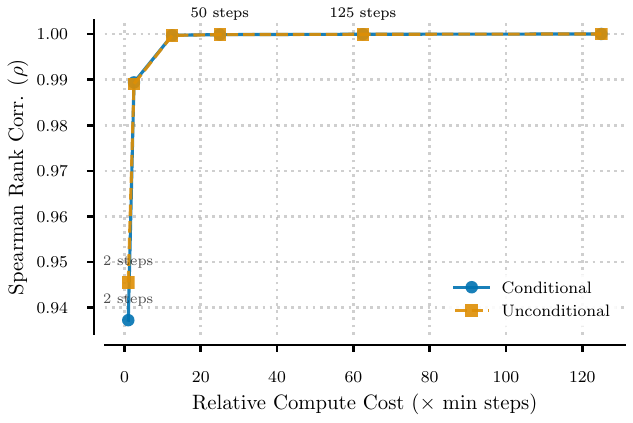}
    \caption{\textbf{Ranking stability of log-density evaluation.} Spearman rank correlation ($\rho$) between rankings at each step count and the 250-step estimates. Both conditional and unconditional rankings reach $\rho>0.99$ by 125 steps.}
    \label{fig:logp_steps_rank_correlation}
\end{figure}

\textbf{Log-density correlation.} \Cref{fig:logp_steps_heatmap} compares lower-step estimates with the 250-step reference. The diagonal concentration shows that absolute values change with numerical integration error while instance ordering remains similar.

\begin{figure}[h]
    \centering
    \begin{subfigure}[b]{\textwidth}
        \centering
        \includegraphics[width=\linewidth]{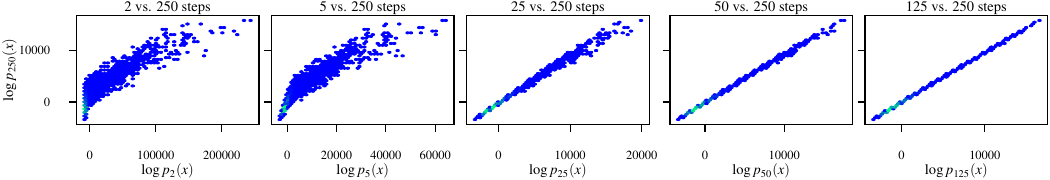}
        \caption{Stepwise joint correlation for \emph{conditional} log-density.} \label{fig:logp_steps_heatmap_cond}
    \end{subfigure}
    \vspace{0.5cm}
    \begin{subfigure}[b]{\textwidth}
        \centering
        \includegraphics[width=\linewidth]{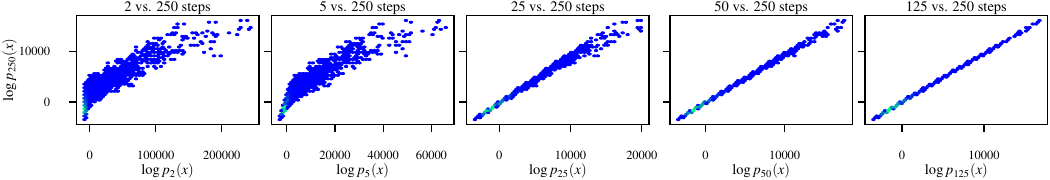}
        \caption{Stepwise joint correlation for \emph{unconditional} log-density.} \label{fig:logp_steps_heatmap_uncond}
    \end{subfigure}
    \caption{\textbf{Numerical convergence of log-densities.} Hexbin plots comparing log-density estimates at $N\in\{2,5,25,50,125\}$ steps with the 250-step estimates.}
    \label{fig:logp_steps_heatmap}
\end{figure}

\subsection{SiT Corrector: Implementation and Training Details}
\label{app:exp_details_imagenet}

To scale our method to ImageNet, we implement the IW-CFM corrector as a parameter-efficient adapter integrated into a frozen SiT-XL/2 backbone. The architecture and parameter-efficient conditioning follow the design of \citet{jensen2025efficient}.

\textbf{Training Strategy and Efficiency.} We precompute the normalized importance weights on an integer grid. For each continuous $\omega\sim\mathcal{U}[4,128]$, the ImageNet data loader rounds to the nearest grid point $\bar m(\omega)$ and samples examples with replacement according to the stored sampling weights at that point. For each sampled instance $i$, we interpolate the unmodified weights at the adjacent grid points to obtain $r_\omega^{\mathrm{reg}}(i|y)$. The flow-matching loss is multiplied by $r_\omega^{\mathrm{reg}}(i|y)/r_{\bar m(\omega)}^{\mathrm{reg}}(i|y)$ when the denominator exceeds $10^{-9}$, and by one otherwise. The unit fallback gives re-injected training examples an unweighted flow-matching loss; these examples would otherwise receive zero weight. It also avoids division by a numerically zero stored weight. We therefore interpolate weights only for sampled instances. This avoids the nested normalization in \cref{alg:iw_cfm_training} during corrector training and is more efficient than uniform sampling followed by self-normalized weighting, for which most sampled examples would have zero loss weight.

\begin{itemize}
    \item We train with a global batch size of 256 distributed across 8 GPUs.
    \item Normalizing independently on each GPU's local subbatch would reduce the normalization-group size, lowering the effective sample size (ESS) and increasing variance.
    \item The precomputed proposal and correction weights are shared across workers.
\end{itemize}

\textbf{Optimizer and Schedule.} We use Adam with $\beta_1=0.5$, $\beta_2=0.9$, and a base learning rate of $1.0 \times 10^{-3}$. We train for 45,000 iterations with a 2,000-iteration linear warmup followed by cosine decay.

\textbf{Evaluation Setup.} We sample with a Heun solver using 250 function evaluations (125 steps). At each $\omega$, we generate 50,000 instances (50 per ImageNet class).

The computational cost of training the corrector is approximately 96 GPUhours on a RTX A6000. In \cref{tab:corrector_overhead} we report the overhead of the corrector during inference on a RTX 4090. For each of the setups we run 10 independent measurements, each preceded by proper warming up. We report the comparison on a single forward run, on the full sampling with a 250-NFE Heun sampler, as well as, sampling with decoding to image space. All experiments are run in both batch size 1 and 24 setups, in eager and compiled mode. While the overhead in batch size 1 eager mode is noticeable (more than 20\%) the difference is decreasing significantly (to less than 1\% in certain setups) with batch size 24 and/or compiled mode.

\begin{table}[t]
\centering
\caption{Corrector runtime overhead relative to base SiT for 250-NFE Heun sampling. Central values are computed from the ratio of mean latencies; uncertainties are delta-method standard errors assuming independent base-model and corrector timing samples. The VAE decoder is always compiled.}
\label{tab:corrector_overhead}
\begin{tabular}{rlrrr}
\toprule
Batch & Mode & Forward (\%) & Sampling (\%) & Sampling + VAE (\%) \\
\midrule
1 & Eager & $+21.41 \pm 0.33$ & $+21.49 \pm 0.12$ & $+18.38 \pm 3.03$ \\
1 & Compiled & $+1.98 \pm 0.03$ & $+2.91 \pm 0.02$ & $+2.88 \pm 0.00$ \\
\midrule
24 & Eager & $+0.58 \pm 0.07$ & $+0.38 \pm 0.47$ & $+0.84 \pm 0.40$ \\
24 & Compiled & $+0.12 \pm 0.11$ & $+0.36 \pm 0.50$ & $+0.18 \pm 0.12$ \\
\bottomrule
\end{tabular}
\end{table}

\subsubsection{ImageNet Weight Hyperparameters}
\label{app:distribution_tempering}

We use $T=1.2$, a per-class clipping fraction $\tau=0.01$ (1\%), a zero-weight re-injection fraction $\rho=0.05$, and a loss-correction threshold of $10^{-9}$. Re-injection is applied separately within each class that has exact-zero stored weights: nonzero sampling weights are scaled by $1-\rho$, and the reserved mass is distributed uniformly only among that class's zero-weight entries. Classes without exact-zero entries are left unchanged. This modification is used only for training-data proposal sampling; interpolation and loss correction use the unmodified weights from \cref{eq:regularized_integer_target,eq:interpolated_target}.

We define the clipping operator precisely. For fixed $m$ and $y$, write the ordered log-weights as $a_{m,(1)}^y\leq\cdots\leq a_{m,(N_y)}^y$, let $k_y=\lceil \tau N_y\rceil$, and let $c_{m,y}=a_{m,(N_y-k_y)}^y$ be the largest log-weight below the clipped group. Then
\begin{equation}
\operatorname{Clip}_{\tau,y}(a_m(i,y))=\min(a_m(i,y),c_{m,y}).
\label{eq:clip_operator}
\end{equation}
Thus, the largest $k_y$ log-weights are capped at $c_{m,y}$ rather than removed. For the degenerate case $k_y=N_y$, we set $c_{m,y}=a_{m,(1)}^y$.

In exact arithmetic, the softmax probabilities are strictly positive. Stored float32 weights can nevertheless become exactly zero when exponentiation and normalization underflow. As illustrated in \cref{fig:non_zero_weights}, the number of weights that remain nonzero is sensitive to $T$ and $\tau$; the selected configuration yields approximately 30,000 numerically active instances per $\omega$ across ImageNet. Here and in the filtered-reference evaluation, \emph{active} means that the stored float32 weight is strictly greater than zero. Clipping itself does not remove instances.

\begin{figure}[t]
    \centering
    \includegraphics[width=0.75\linewidth]{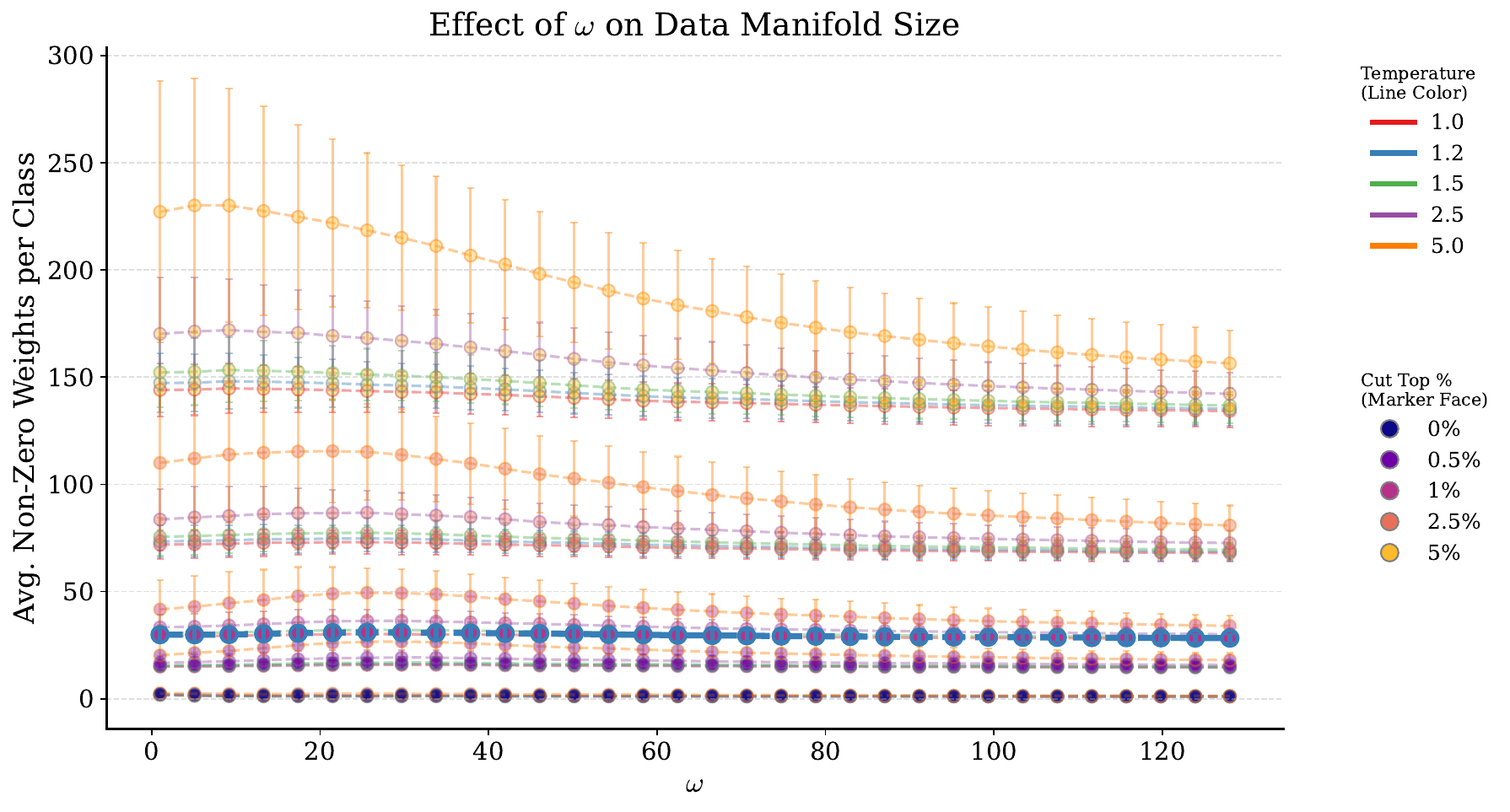}
    \caption{\textbf{Effect of weight regularization.} Average number of numerically nonzero per-class weights after finite-precision normalization as a function of $\omega$, for different temperatures $T$ and clipping fractions $\tau$. The highlighted curve is the configuration used in our experiments ($T=1.2$, $\tau=0.01$). Error bars show the standard deviation across the 1,000 ImageNet categories. Horizontal-flip augmentation doubles the number of available instances in practice.}
    \label{fig:non_zero_weights}
\end{figure}

\subsection{Details of the ablation experiments}

\subsubsection{Number of steps in log-densities}
\label{app:ablation_num_steps}

In \cref{tab:imagenet100_categories} we report the 100 ImageNet categories used in the number of steps in log-densities ablation experiment, grouped into thematic sections.

\begin{table*}[t]
\centering
\caption{The 100 ImageNet categories used in the number of steps in log-densities ablation experiment.
Numbers in parentheses denote the ImageNet class IDs.}
\label{tab:imagenet100_categories}
\small
\setlength{\tabcolsep}{5pt}
\renewcommand{\arraystretch}{1.05}

\begin{tabularx}{\textwidth}{YYYYY}
\toprule

\multicolumn{5}{l}{\textbf{Animals (29 categories)}} \\
great white shark~(2) &
bald eagle~(22) &
axolotl~(29) &
loggerhead turtle~(33) &
Komodo dragon~(48) \\
triceratops~(51) &
scorpion~(71) &
peacock~(84) &
platypus~(103) &
jellyfish~(107) \\
chambered nautilus~(117) &
hermit crab~(125) &
flamingo~(130) &
king penguin~(145) &
killer whale~(148) \\
Labrador retriever~(208) &
tabby cat~(281) &
tiger~(292) &
bee~(309) &
dragonfly~(319) \\
monarch butterfly~(323) &
starfish~(327) &
zebra~(340) &
hippopotamus~(344) &
Arabian camel~(354) \\
three-toed sloth~(364) &
chimpanzee~(367) &
African elephant~(386) &
pufferfish~(397) &
\\[1pt]

\multicolumn{5}{l}{\textbf{Food, plants, and natural scenes (12 categories)}} \\
ice cream~(928) &
cheeseburger~(933) &
broccoli~(937) &
strawberry~(949) &
banana~(954) \\
pizza~(963) &
rapeseed~(984) &
daisy~(985) &
alp~(970) &
coral reef~(973) \\
seashore~(978) &
volcano~(980) &
&
&
\\[1pt]

\multicolumn{5}{l}{\textbf{Vehicles and transport (12 categories)}} \\
airliner~(404) &
ambulance~(407) &
bullet train~(466) &
canoe~(472) &
container ship~(510) \\
fire engine~(555) &
mountain bike~(671) &
pickup truck~(717) &
school bus~(779) &
sports car~(817) \\
submarine~(833) &
tractor~(866) &
&
&
\\[1pt]

\multicolumn{5}{l}{\textbf{Music, sports, and leisure (10 categories)}} \\
acoustic guitar~(402) &
cello~(486) &
French horn~(566) &
grand piano~(579) &
saxophone~(776) \\
barbell~(422) &
basketball~(430) &
golf ball~(574) &
parachute~(701) &
soccer ball~(805) \\
[1pt]

\multicolumn{5}{l}{\textbf{Technology, tools, and scientific objects (12 categories)}} \\
binoculars~(447) &
chain saw~(491) &
computer keyboard~(508) &
desktop computer~(527) &
hammer~(587) \\
joystick~(613) &
laptop~(620) &
oscilloscope~(688) &
power drill~(740) &
projector~(745) \\
reflex camera~(759) &
stethoscope~(823) &
&
&
\\[1pt]

\multicolumn{5}{l}{\textbf{Household, clothing, and personal objects (12 categories)}} \\
backpack~(414) &
bookcase~(453) &
coffee mug~(504) &
jeans~(608) &
kimono~(614) \\
microwave~(651) &
refrigerator~(760) &
running shoe~(770) &
umbrella~(879) &
vacuum cleaner~(882) \\
wallet~(893) &
washing machine~(897) &
&
&
\\[1pt]

\multicolumn{5}{l}{\textbf{Architecture and infrastructure (10 categories)}} \\
lighthouse~(437) &
castle~(483) &
church~(497) &
dam~(525) &
greenhouse~(580) \\
mosque~(668) &
palace~(698) &
suspension bridge~(839) &
street sign~(919) &
traffic light~(920) \\
[1pt]

\multicolumn{5}{l}{\textbf{People and human activities (3 categories)}} \\
baseball player~(981) &
groom / bridegroom~(982) &
scuba diver~(983) &
&
\\

\bottomrule
\end{tabularx}
\end{table*}

\FloatBarrier
\section{Extended Empirical Evidence}
\label{app:results_extended}

\subsection{Absolute CLIP Contrasts for Real-Data Rankings}
\label{app:clip_rank_absolute}

The baseline-relative trajectories in \cref{fig:clip-rank-shifted} are obtained by centering each classwise contrast at its value for $\omega=1$. This centering isolates changes with $\omega$, but suppresses the sign and magnitude of the baseline association. \Cref{fig:clip-rank-absolute} reports the corresponding uncentered contrasts, with $\Delta d_{c,a}(\omega)=d_{c,a}(\omega)-d_{c,a}(1)$ linking the two representations.

\begin{figure}[ht]
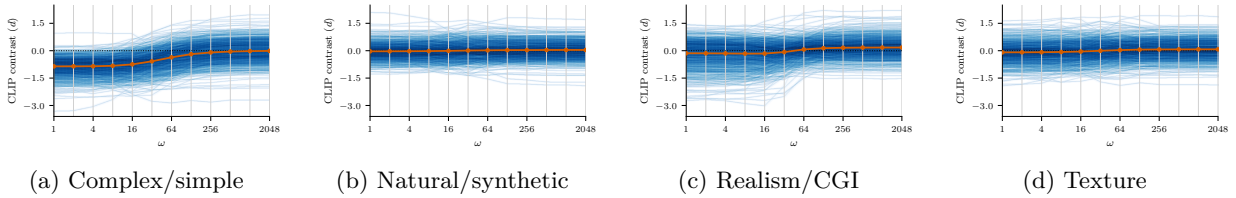

  \centering

  \begin{subfigure}[t]{0.24\linewidth}
    \centering
    \includegraphics[width=\linewidth]{%
      \cliprankdir/clip_real_rank_trajectory_density_wide_complex_simple.pdf%
    }
    \caption{Complex/simple}
    \label{fig:clip-rank-absolute-complex}
  \end{subfigure}%
  \hfill
  \begin{subfigure}[t]{0.24\linewidth}
    \centering
    \includegraphics[width=\linewidth]{%
      \cliprankdir/clip_real_rank_trajectory_density_wide_natural_synthetic.pdf%
    }
    \caption{Natural/synthetic}
    \label{fig:clip-rank-absolute-natural}
  \end{subfigure}%
  \hfill
  \begin{subfigure}[t]{0.24\linewidth}
    \centering
    \includegraphics[width=\linewidth]{%
      \cliprankdir/clip_real_rank_trajectory_density_wide_realism_vs_cgi.pdf%
    }
    \caption{Realism/CGI}
    \label{fig:clip-rank-absolute-realism}
  \end{subfigure}%
  \hfill
  \begin{subfigure}[t]{0.24\linewidth}
    \centering
    \includegraphics[width=\linewidth]{%
      \cliprankdir/clip_real_rank_trajectory_density_wide_texture.pdf%
    }
    \caption{Texture}
    \label{fig:clip-rank-absolute-texture}
  \end{subfigure}

  \caption{Absolute classwise CLIP contrasts underlying the baseline-relative trajectories in \cref{fig:clip-rank-shifted}. Positive $d_{c,a}(\omega)$ indicates that the top-ranked images align more strongly with the positive prompt than the bottom-ranked images. The density fields show the classwise trajectories, and the orange curves show their across-class means.}
  \label{fig:clip-rank-absolute}
\end{figure}

\FloatBarrier

\subsection{CLIP Comparison of IW-CFM and CFG}
\label{app:clip_generated_comparison}

As an ancillary method comparison, we evaluate IW-CFM and CFG generations using the paired CLIP scores defined in \cref{sec:clip_suite,app:clip_details}, with inference-time guidance held at $\gamma=4$ for both methods.

For each ImageNet category and prompt pair, we select the target $\omega$ that maximizes the mean paired CLIP score of the IW-CFM samples over the evaluated values; we term this the \emph{score-optimal $\omega$}. \Cref{fig:max_gain_clip} reports the resulting oracle comparison. In this in-sample, post-selected comparison, the selected IW-CFM score exceeds the CFG score in 70.6\%, 69.7\%, 68.2\%, and 69.5\% of categories for complex/simple, natural/synthetic, realism/CGI, and texture, respectively---approximately 70\% for each prompt pair.

\begin{figure}[t]
    \centering
    \includegraphics[width=\linewidth]{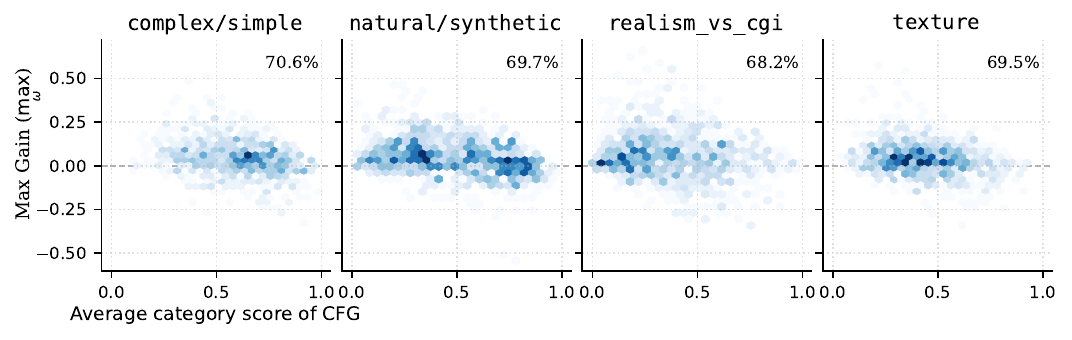}
    \caption{In-sample oracle upper envelope for the additional IW-CFM--CFG comparison. Each panel shows ImageNet categories according to the paired CLIP score of CFG at $\gamma=4$ (x-axis) and the maximum score difference attained by IW-CFM at the same inference-time guidance after selecting its target $\omega$ on the evaluation samples (y-axis). Color intensity denotes the number of categories in each hexagonal bin, and the percentage is the fraction with a positive post-selection difference. Because selection and evaluation use the same samples, the figure is neither a held-out performance estimate nor evidence of statistical significance.}
    \label{fig:max_gain_clip}
\end{figure}

\FloatBarrier

\subsection{Additional Synthetic Results}
\label{app:results_extended_synthetic}

\subsubsection{Additional Synthetic Results: IW-CFM}
\label{app:results_extended_synthetic_iwcfm}

\Cref{fig:cfg_vs_powered_3gaussians,fig:cfg_vs_powered_fractal} provide an expanded qualitative comparison between the sampling behavior of standard CFG and IW-CFM on synthetic datasets where the analytical target density is known.

\begin{figure}[h]
    \centering
    \begin{subfigure}[t]{0.48\linewidth}
        \centering
        \includegraphics[width=\linewidth]{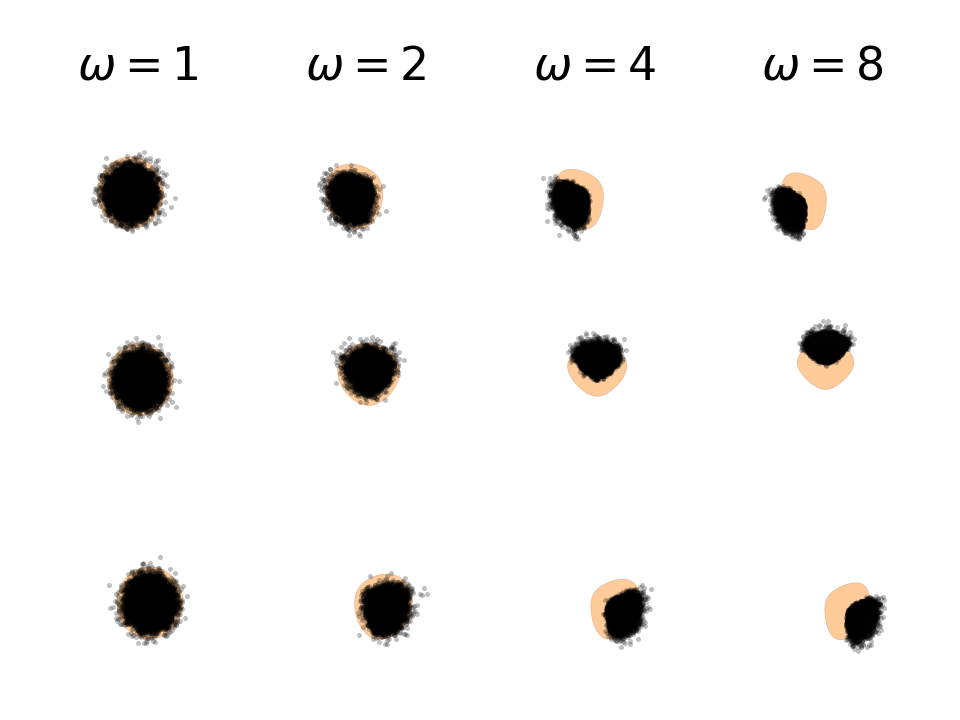}
        \caption{99.7\% probability-mass contours of the $\omega$-powered distribution for the three classes of the three-Gaussian distribution, with CFG samples overlaid at the matched guidance scale $\gamma=\omega$.}
        \label{fig:cfg_3gaussians}
    \end{subfigure}
    \hfill 
    \begin{subfigure}[t]{0.48\linewidth}
        \centering
        \includegraphics[width=\linewidth]{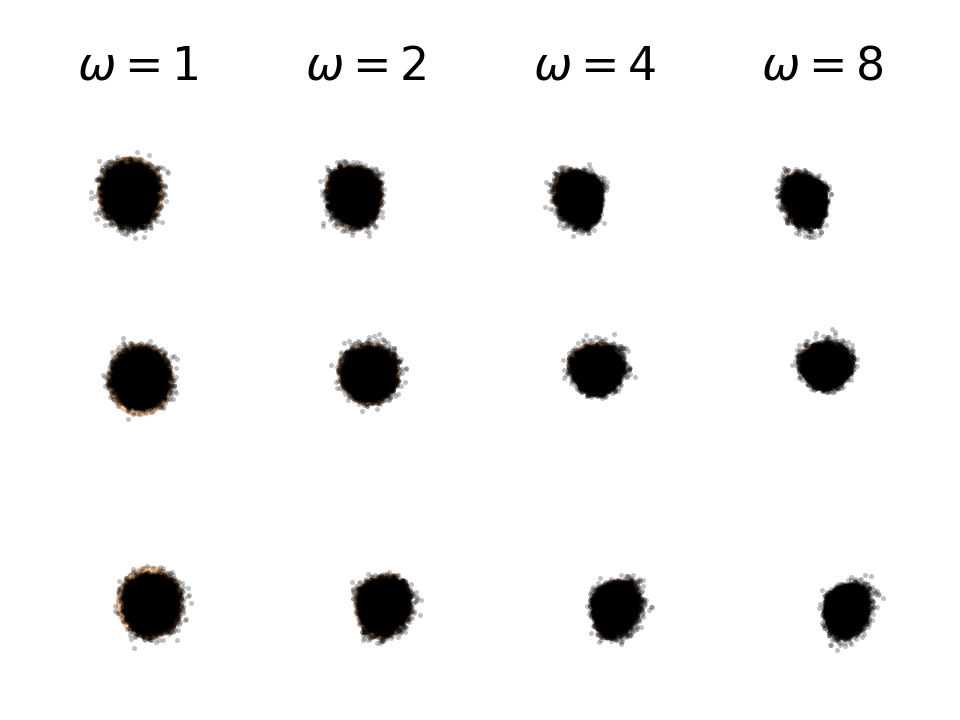}
        \caption{99.7\% probability-mass contours of the $\omega$-powered distribution for the three classes of the three-Gaussian distribution, with directly sampled IW-CFM realizations overlaid at the corresponding $\omega$; no additional CFG is applied.}
        \label{fig:powered_3gaussians}
    \end{subfigure}
    \caption{\textbf{Three-Gaussian comparison.} CFG with matched guidance $\gamma=\omega$ (left) departs from the analytical target contours, whereas IW-CFM targeting $\omega$ (right) remains close to them.}
    \label{fig:cfg_vs_powered_3gaussians}
\end{figure}

\begin{figure}[h]
    \centering
    \begin{subfigure}[t]{0.48\linewidth}
        \centering
        \includegraphics[width=\linewidth]{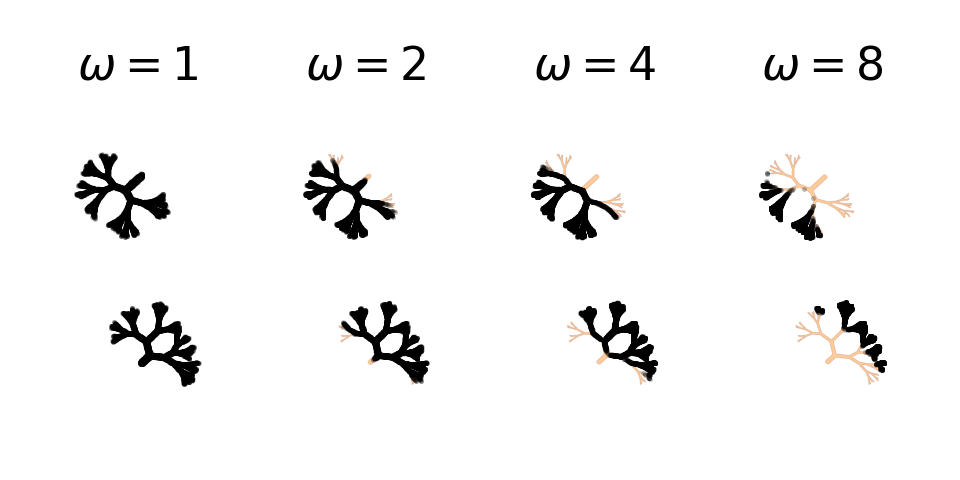}
        \caption{99.7\% probability-mass contours of the $\omega$-powered distribution for the two classes of the fractal distribution, with CFG samples overlaid at the matched guidance scale $\gamma=\omega$.}
        \label{fig:cfg_fractal}
    \end{subfigure}
    \hfill 
    \begin{subfigure}[t]{0.48\linewidth}
        \centering
        \includegraphics[width=\linewidth]{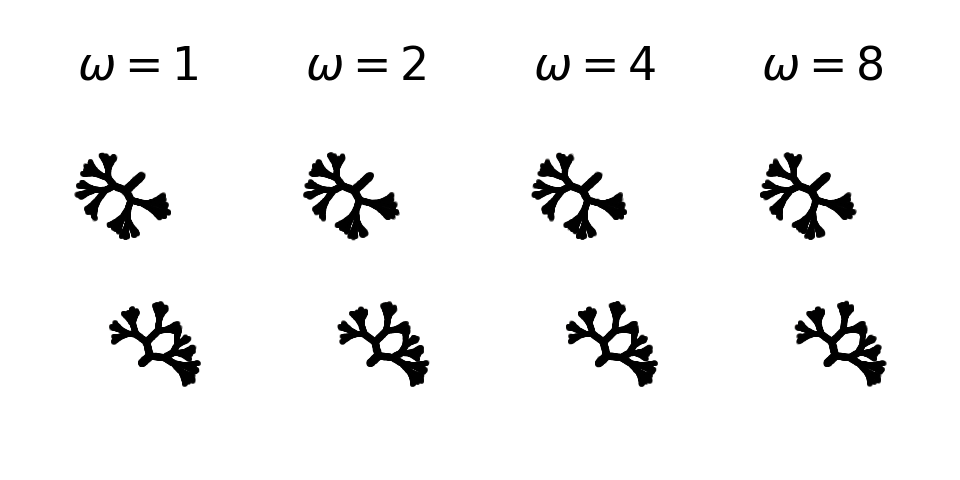}
        \caption{99.7\% probability-mass contours of the $\omega$-powered distribution for the two classes of the fractal distribution, with directly sampled IW-CFM realizations overlaid at the corresponding $\omega$; no additional CFG is applied.}
        \label{fig:powered_fractal}
    \end{subfigure}

    \caption{\textbf{Fractal comparison.} CFG with matched guidance $\gamma=\omega$ (left) departs from the analytical target contours, whereas IW-CFM targeting $\omega$ (right) remains close to them.}
    \label{fig:cfg_vs_powered_fractal}
\end{figure}

\subsubsection{Additional Synthetic Results: REG \& FKC}
\label{app:results_extended_synthetic_reg_fkc}

We implement Rectified Gradient Guidance (REG) \citep{gao2025reg,gao2025regcode} and Feynman--Kac Correctors (FKC) \citep{skretafeynman,skreta2025fkccode} within our evaluation codebase, following the published algorithms and reusing components from the released implementations where applicable. We apply both methods to the exact conditional and unconditional vector fields of the three-Gaussian problem. Under the tested configurations, neither resulting sample distribution closely follows the analytical powered contours in \cref{fig:reg_and_fkc_3gaussians}. This observation is specific to these implementations and configurations and should not be interpreted as a general failure of either framework.

\begin{figure}[h]
    \centering
    \begin{subfigure}[t]{0.48\linewidth}
        \centering
        \includegraphics[width=\linewidth]{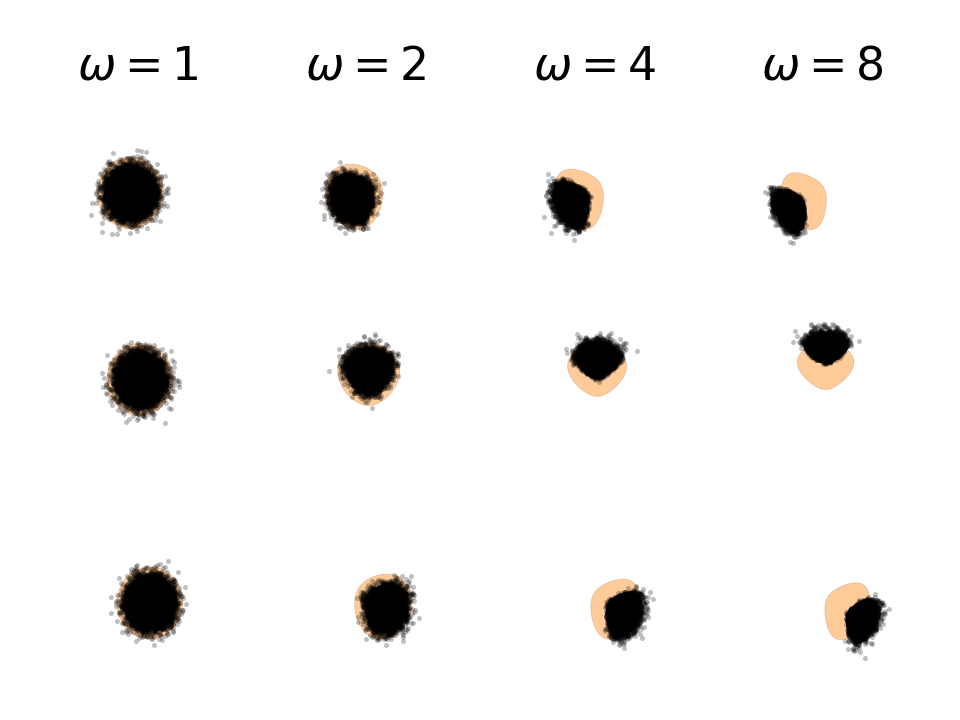}
        \caption{Samples from Rectified Gradient Guidance over the 99.7\% probability-mass contours of the $\omega$-powered target for the three-Gaussian distribution.}
        \label{fig:reg_3gaussians}
    \end{subfigure}
    \hfill 
    \begin{subfigure}[t]{0.48\linewidth}
        \centering
        \includegraphics[width=\linewidth]{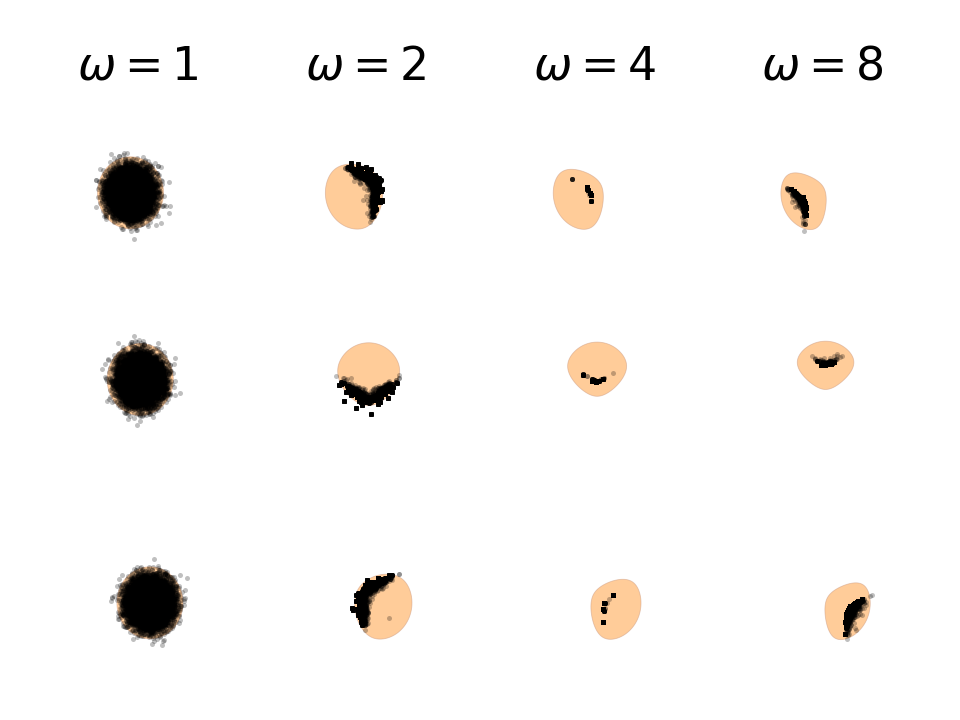}
        \caption{Samples from Feynman--Kac Correctors over the 99.7\% probability-mass contours of the $\omega$-powered target for the three-Gaussian distribution.}
        \label{fig:fkc_3gaussians}
    \end{subfigure}
    \caption{\textbf{REG and FKC on a Gaussian mixture.} Results from our implementations of REG \citep{gao2025reg} (left) and FKC \citep{skretafeynman} (right), which follow the published algorithms and reuse components from released code where applicable. Under these configurations, neither sample distribution closely follows the analytical $\omega$-powered contours.}
    \label{fig:reg_and_fkc_3gaussians}
\end{figure}

\subsection{ImageNet Metrics Extension}

\subsubsection{Main Text Metrics Extension}
\label{app:results_extended_imagenet_fd}

The following analysis extends \cref{fig:fid_wrt_oracle} by comparing IW-CFM targeting each $\omega$ against CFG whose guidance scale is numerically matched, $\gamma=\omega$.

As shown in \cref{fig:fid_wrt_oracle_full}, increasing CFG guidance to match the reference powering parameter, $\gamma=\omega$, leads to rapidly increasing distance from the regularized empirical references. IW-CFM remains closer to these references, so the result is not limited to comparison with fixed CFG guidance at $\gamma=1.5$. IW-CFM is evaluated on the following $\omega \in \{4.0, 4.17, 8.0, 8.5, 15.87, 16.0, 32.0, 32.24, 63.89, 64.0, 127.57, 128.0\}$ to validate the effectiveness of the interpolation of precomputed self-normalized weights. The baselines are evaluated only on the integer $\omega$.

\begin{figure}[h]
    \centering
    \includegraphics[width=0.5\linewidth,trim=0 23bp 0 0,clip]{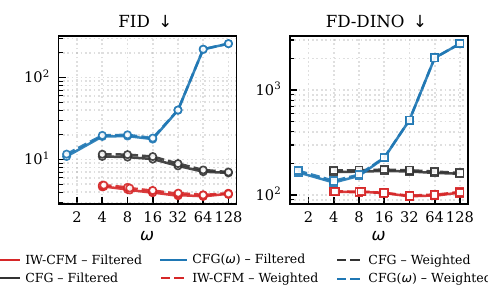}

    {\scriptsize
    \begin{tabular}{lll}
    \textcolor{black}{\rule{1.2em}{0.5pt}} CFG--Filtered &
    \textcolor{blue}{\rule{1.2em}{0.5pt}} CFG ($\gamma=\omega$)--Filtered &
    \textcolor{red}{\rule{1.2em}{0.5pt}} IW-CFM--Filtered \\
    \textcolor{black}{--} CFG--Weighted &
    \textcolor{blue}{--} CFG ($\gamma=\omega$)--Weighted &
    \textcolor{red}{--} IW-CFM--Weighted
    \end{tabular}}
    \caption{\textbf{Distributional alignment across $\omega$.} Extension of \cref{fig:fid_wrt_oracle} across the full range of target powering parameters. IW-CFM targets each $\omega$ and uses inference-time guidance $\gamma=1.5$. CFG uses either fixed guidance $\gamma=1.5$ or matched guidance $\gamma=\omega$. The log-scale y-axis shows the increase in FID and FD-DINOv2 for matched-guidance CFG at larger $\gamma=\omega$, while IW-CFM remains closer to the regularized empirical reference.}
    \label{fig:fid_wrt_oracle_full}
\end{figure}

\begin{figure}[ht]
    \centering
    \includegraphics[width=\linewidth]{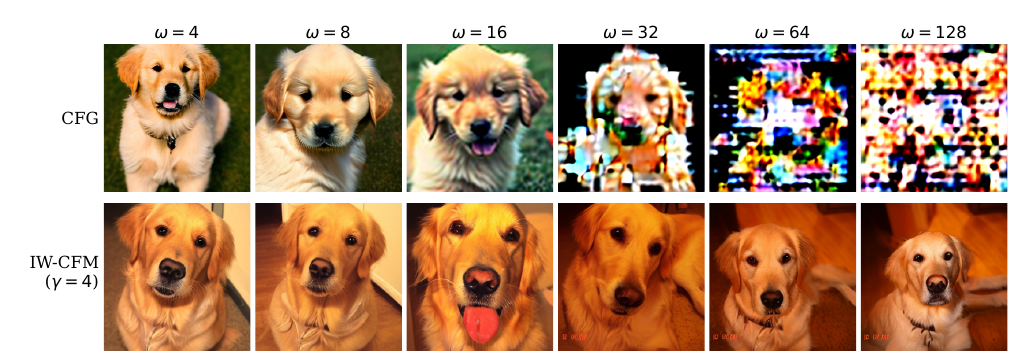}
    \caption{\textbf{Qualitative matched-guidance comparison on ImageNet.} Columns correspond to target powers $\omega\in\{4,8,16,32,64,128\}$. The top row shows standard CFG at the numerically matched guidance scale $\gamma=\omega$; the bottom row shows IW-CFM targeting the same $\omega$ while holding inference-time guidance fixed at $\gamma=4$. In this golden retriever example, CFG remains recognizable through $\gamma=16$ but exhibits severe saturation and loss of coherent image structure for $\gamma\geq32$, whereas IW-CFM remains visually coherent across the evaluated target powers. This illustrative comparison complements the distributional alignment metrics in \cref{fig:fid_wrt_oracle_full}.}
    \label{fig:cfg_vs_iwcfm_imagenet}
\end{figure}

\FloatBarrier

\subsubsection{Unguided Distribution-Shift Evaluation}
\label{app:metrics_wo_guidance}

This evaluation tests whether the learned IW-CFM vector field shifts the generated distribution toward the regularized reference without relying on inference-time guidance. We report Fr\'echet Distance in \cref{tab:fid_filtered_noguidance} and PRDC in \cref{tab:prdc_inception_filtered_noguidance}, comparing unguided IW-CFM with the conditional Base SiT ($\gamma=1$). IW-CFM is consistently better aligned with the reference across $\omega$, providing evidence that the distributional shift is learned by the corrector itself. The remaining gap to the guided results quantifies the practical benefit of the additional inference heuristic.

\begin{table}[h]
\begin{center}
\caption{Unguided distribution-shift evaluation using Fr\'echet Distance to the filtered regularized ImageNet reference. All IW-CFM rows use $\gamma=1$ and differ only in the target powering parameter $\omega$. IW-CFM outperforms the conditional Base SiT at every evaluated $\omega$; comparison with \cref{fig:fid_wrt_oracle} shows the additional benefit of guided inference.}
\begin{tabular}{llll}
\toprule
Model & Reference $\omega$ & FID & FD-DINOv2 \\
\midrule
Base SiT ($\gamma=1$) & 4.0 & 19.97 & 324.02 \\
IW-CFM ($\omega=4.0$) & 4.0 & \textbf{8.48} & \textbf{215.80} \\
\midrule
Base SiT ($\gamma=1$) & 8.0 & 20.55 & 328.66 \\
IW-CFM ($\omega=8.0$) & 8.0 & \textbf{8.90} & \textbf{221.41} \\
\midrule
Base SiT ($\gamma=1$) & 16.0 & 21.41 & 339.18 \\
IW-CFM ($\omega=16.0$) & 16.0 & \textbf{9.29} & \textbf{223.39} \\
\midrule
Base SiT ($\gamma=1$) & 32.0 & 21.73 & 348.58 \\
IW-CFM ($\omega=32.0$) & 32.0 & \textbf{10.56} & \textbf{227.64} \\
\midrule
Base SiT ($\gamma=1$) & 64.0 & 20.23 & 341.78 \\
IW-CFM ($\omega=64.0$) & 64.0 & \textbf{12.27} & \textbf{245.24} \\
\midrule
Base SiT ($\gamma=1$) & 128.0 & 18.73 & 334.46 \\
IW-CFM ($\omega=128.0$) & 128.0 & \textbf{13.52} & \textbf{263.39} \\
\bottomrule
\label{tab:fid_filtered_noguidance}
\end{tabular}
\end{center}
\end{table}

\begin{table}[h]
\begin{center}
\caption{Unguided distribution-shift evaluation using Precision, Recall, Density, and Coverage with respect to the filtered regularized ImageNet reference. All IW-CFM rows use $\gamma=1$ and differ only in the target powering parameter $\omega$. IW-CFM is better aligned than the conditional Base SiT at every $\omega$; \cref{tab:prdc_inception_filtered} reports the guided practical configuration.}
\begin{tabular}{llllll}
\toprule
Model & Reference $\omega$ & Precision & Recall & Density & Coverage \\
\midrule
Base SiT ($\gamma=1$) & 4.0 & 0.37 & 0.59 & 0.64 & 0.81 \\
IW-CFM ($\omega=4.0$) & 4.0 & \textbf{0.51} & \textbf{0.68} & \textbf{0.87} & \textbf{0.93} \\
\midrule
Base SiT ($\gamma=1$) & 8.0 & 0.37 & 0.60 & 0.62 & 0.80 \\
IW-CFM ($\omega=8.0$) & 8.0 & \textbf{0.51} & \textbf{0.69} & \textbf{0.84} & \textbf{0.93} \\
\midrule
Base SiT ($\gamma=1$) & 16.0 & 0.39 & 0.62 & 0.61 & 0.79 \\
IW-CFM ($\omega=16.0$) & 16.0 & \textbf{0.51} & \textbf{0.70} & \textbf{0.83} & \textbf{0.92} \\
\midrule
Base SiT ($\gamma=1$) & 32.0 & 0.41 & 0.64 & 0.59 & 0.74 \\
IW-CFM ($\omega=32.0$) & 32.0 & \textbf{0.50} & \textbf{0.70} & \textbf{0.75} & \textbf{0.89} \\
\midrule
Base SiT ($\gamma=1$) & 64.0 & 0.44 & 0.65 & 0.63 & 0.72 \\
IW-CFM ($\omega=64.0$) & 64.0 & \textbf{0.48} & \textbf{0.71} & \textbf{0.71} & \textbf{0.86} \\
\midrule
Base SiT ($\gamma=1$) & 128.0 & 0.47 & 0.65 & 0.69 & 0.72 \\
IW-CFM ($\omega=128.0$) & 128.0 & \textbf{0.48} & \textbf{0.71} & \textbf{0.71} & \textbf{0.85} \\
\bottomrule
\label{tab:prdc_inception_filtered_noguidance}
\end{tabular}
\end{center}
\end{table}

\subsubsection{Standard ImageNet Target Metrics}
\label{app:results_extended_imagenet_target_metrics}

We also report the main-text metrics with respect to standard ImageNet references:
\begin{itemize}
    \item for Fr\'echet Distance: a category-saturated 50k subset of the training set,
    \item for Precision, Recall, Density, and Coverage: a 10k-feature batch from the standard evaluation suite of \citet{openai_adm_evaluations}.
\end{itemize}

In \cref{tab:fid_imagenet_target,tab:prdc_imagenet_target}, CFG at a suitable guidance scale $\gamma$ is better aligned with the standard ImageNet references, as expected. IW-CFM, evaluated with fixed inference-time guidance $\gamma=1.5$, moves closer to these references as its powering parameter $\omega$ increases.

\Cref{tab:fid_imagenet_target_noguidance} reports Fr\'echet Distance without inference-time guidance. In this setup, IW-CFM at $\omega=64$ is closest to the standard ImageNet reference.

\begin{table}[h]
\begin{center}
\caption{Fr\'echet Distance with respect to the category-saturated 50k subset of the ImageNet training set using two feature backbones. CFG rows are labeled by guidance scale $\gamma$; IW-CFM rows are labeled by powering parameter $\omega$ and use inference-time guidance $\gamma=1.5$. The best CFG configuration outperforms IW-CFM, although the gap narrows as $\omega$ increases.}
\begin{tabular}{lll}
\toprule
Model & FID-50k & FD-DINOv2-50k \\
\midrule
CFG ($\gamma=1.5$) & \textbf{2.04} & 88.36 \\
CFG ($\gamma=4.0$) & 14.07 & \textbf{85.64} \\
IW-CFM ($\omega=4.0$) & 8.33 & 157.81 \\
CFG ($\gamma=8.0$) & 16.32 & 137.56 \\
IW-CFM ($\omega=8.0$) & 8.49 & 159.98 \\
CFG ($\gamma=16.0$) & 16.25 & 238.82 \\
IW-CFM ($\omega=16.0$) & 8.22 & 152.25 \\
CFG ($\gamma=32.0$) & 36.00 & 517.38 \\
IW-CFM ($\omega=32.0$) & 6.48 & 120.31 \\
CFG ($\gamma=64.0$) & 213.17 & 1991.20 \\
IW-CFM ($\omega=64.0$) & 4.54 & 99.67 \\
CFG ($\gamma=128.0$) & 253.71 & 2749.11 \\
IW-CFM ($\omega=128.0$) & 3.80 & 99.69 \\
\bottomrule
\label{tab:fid_imagenet_target}
\end{tabular}
\end{center}
\end{table}

\begin{table}[h]
\begin{center}
\caption{Precision, Recall, Density, and Coverage with respect to the 10k-feature ImageNet evaluation batch from \citet{openai_adm_evaluations}. CFG is labeled by guidance scale $\gamma$; IW-CFM is labeled by powering parameter $\omega$ and uses inference-time guidance $\gamma=1.5$. As in \cref{tab:fid_imagenet_target}, the baseline has the best overlap with the reference, although the gap narrows as $\omega$ increases.}
\begin{tabular}{lllll}
\toprule
Model & Precision & Recall & Density & Coverage \\
\midrule
CFG ($\gamma=1.5$) & \textbf{0.83} & \textbf{0.58} & \textbf{1.17} & \textbf{0.98} \\
IW-CFM ($\omega=4.0$) & 0.78 & 0.49 & 1.05 & 0.94 \\
IW-CFM ($\omega=8.0$) & 0.78 & 0.49 & 1.05 & 0.93 \\
IW-CFM ($\omega=16.0$) & 0.79 & 0.48 & 1.07 & 0.94 \\
IW-CFM ($\omega=32.0$) & 0.81 & 0.49 & 1.08 & 0.94 \\
IW-CFM ($\omega=64.0$) & 0.81 & 0.52 & 1.08 & 0.95 \\
IW-CFM ($\omega=128.0$) & 0.81 & 0.54 & 1.08 & 0.95 \\
\bottomrule
\label{tab:prdc_imagenet_target}
\end{tabular}
\end{center}
\end{table}

\begin{table}[h]
\begin{center}
\caption{Fr\'echet Distance with respect to the category-saturated 50k subset of the ImageNet training set without inference-time guidance ($\gamma=1$). IW-CFM at $\omega=64$ achieves the best result for both feature backbones, outperforming the conditional Base SiT.}
\begin{tabular}{lll}
\toprule
Model & FID-50k & FD-DINOv2-50k \\
\midrule
Base SiT ($\gamma=1$) & 9.55 & 246.24 \\
IW-CFM ($\omega=4.0$) & 11.21 & 247.82 \\
IW-CFM ($\omega=8.0$) & 11.53 & 252.24 \\
IW-CFM ($\omega=16.0$) & 10.52 & 238.72 \\
IW-CFM ($\omega=32.0$) & 7.64 & 204.28 \\
IW-CFM ($\omega=64.0$) & \textbf{7.14} & \textbf{202.63} \\
IW-CFM ($\omega=128.0$) & 8.39 & 219.95 \\
\bottomrule
\label{tab:fid_imagenet_target_noguidance}
\end{tabular}
\end{center}
\end{table}

\subsection{Memorization Analysis}
\label{app:nearest_neighbor_analysis}

To assess whether IW-CFM retrieves high-likelihood training instances rather than generating novel variations, we conduct a nearest-neighbor analysis. This test asks whether training-data retrieval could explain the observed distribution-alignment and paired-CLIP score differences.

\textbf{Methodology.} We follow the framework of \citet{gumemorization}. After deduplicating the training set as described below, we flatten the VAE latent tensors and compute Euclidean distances from each generated sample to every remaining training instance. We apply no additional distance filter before sorting these distances. Let $d_1\leq d_2$ be the two smallest distances. We flag a generated sample as a potential retrieval when
\begin{equation}
    1.82d_1<d_2.
\end{equation}
For $d_1>0$, this is equivalent to $d_2/d_1>1.82$; the multiplicative form also covers an exact match with $d_1=0$. The reported memorization ratio is the fraction of generated samples satisfying this condition. We use the VAE latent representation on which the generative model operates; this choice does not guarantee that the distances measure semantic similarity.

\textbf{Deduplication.} ImageNet contains duplicate and near-duplicate training instances. A generated sample may therefore have similar distances to two copies of the same image ($d_1\approx d_2$), causing the ratio test to miss a possible retrieval. Before the nearest-neighbor search, we deduplicate the training set by removing instances within an $\ell_2$-distance threshold of $d_\textrm{duplicate}=35$.

\textbf{Retrieval-Test Results.} As shown in \cref{fig:memorization_check}, the fraction flagged by the detector remains below $0.05\%$ across all evaluated powering parameters $\omega$. Few samples are therefore flagged as retrievals under the specified test, but this is not a general guarantee of non-memorization.

\begin{figure}[h]
    \centering
    \includegraphics[width=0.5\linewidth]{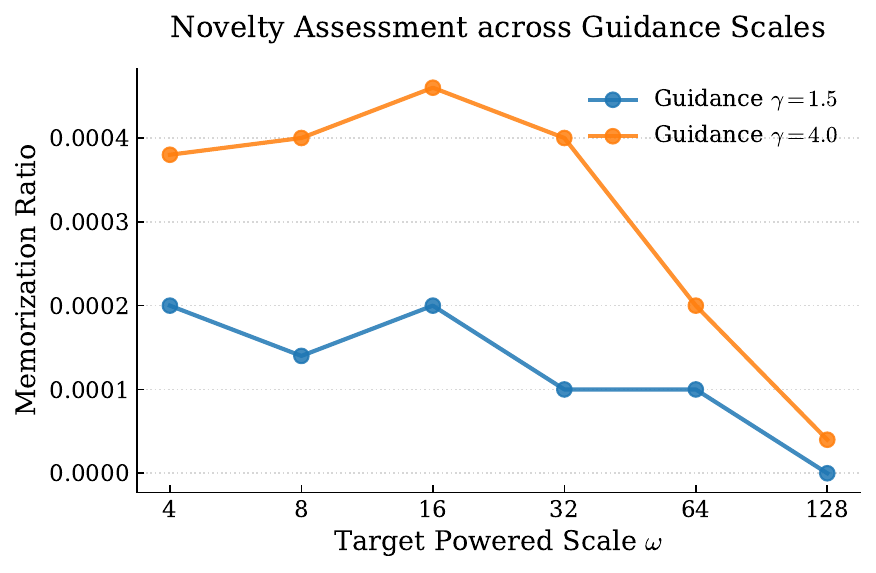}
    \caption{\textbf{Nearest-neighbor retrieval test.} Fraction of IW-CFM realizations satisfying $1.82d_1<d_2$ across powering parameters $\omega$ and inference-time guidance scales $\gamma\in\{1.5,4\}$, where $d_1$ and $d_2$ are Euclidean distances to the first and second nearest neighbors in the deduplicated training set. The ratios remain low under this test, but do not by themselves establish generalization.}
    \label{fig:memorization_check}
\end{figure}

\FloatBarrier
\subsection{Extended ImageNet Samples}
\label{app:results_extended_imagenet_samples}

We present additional qualitative results for the ImageNet categories used in the powered-density ranking analysis of \cref{app:omega_pow_imagenet}. These visualizations show how IW-CFM generations change with the target powering parameter $\omega$ while the inference-time guidance scale remains fixed at $\gamma=4$; they are not a systematic evaluation of attribute transfer or generalization.

As shown in \cref{fig:traj_4_33_139_414,fig:traj_417_728_757_770,fig:traj873}, each grid compares a fixed set of four initial latents (rows) generated by baseline CFG at $\gamma=4$ and by IW-CFM at varying target $\omega$, also with $\gamma=4$.

\begin{figure}[h]
    \centering
    \begin{subfigure}[t]{0.48\linewidth}
        \centering
        \includegraphics[width=\linewidth]{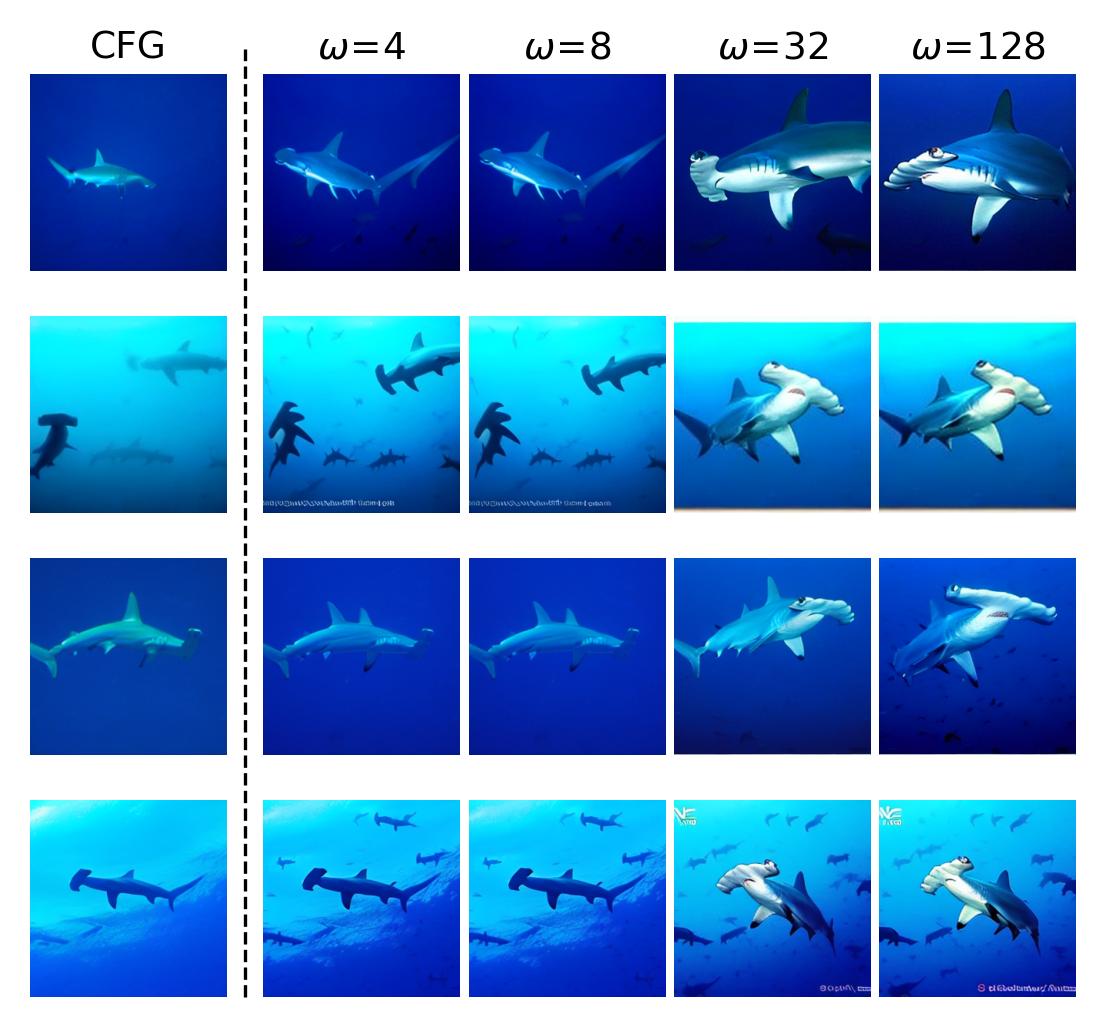}
        \caption{hammerhead shark}
        \label{fig:traj4}
    \end{subfigure}
    \hfill 
    \begin{subfigure}[t]{0.48\linewidth}
        \centering
        \includegraphics[width=\linewidth]{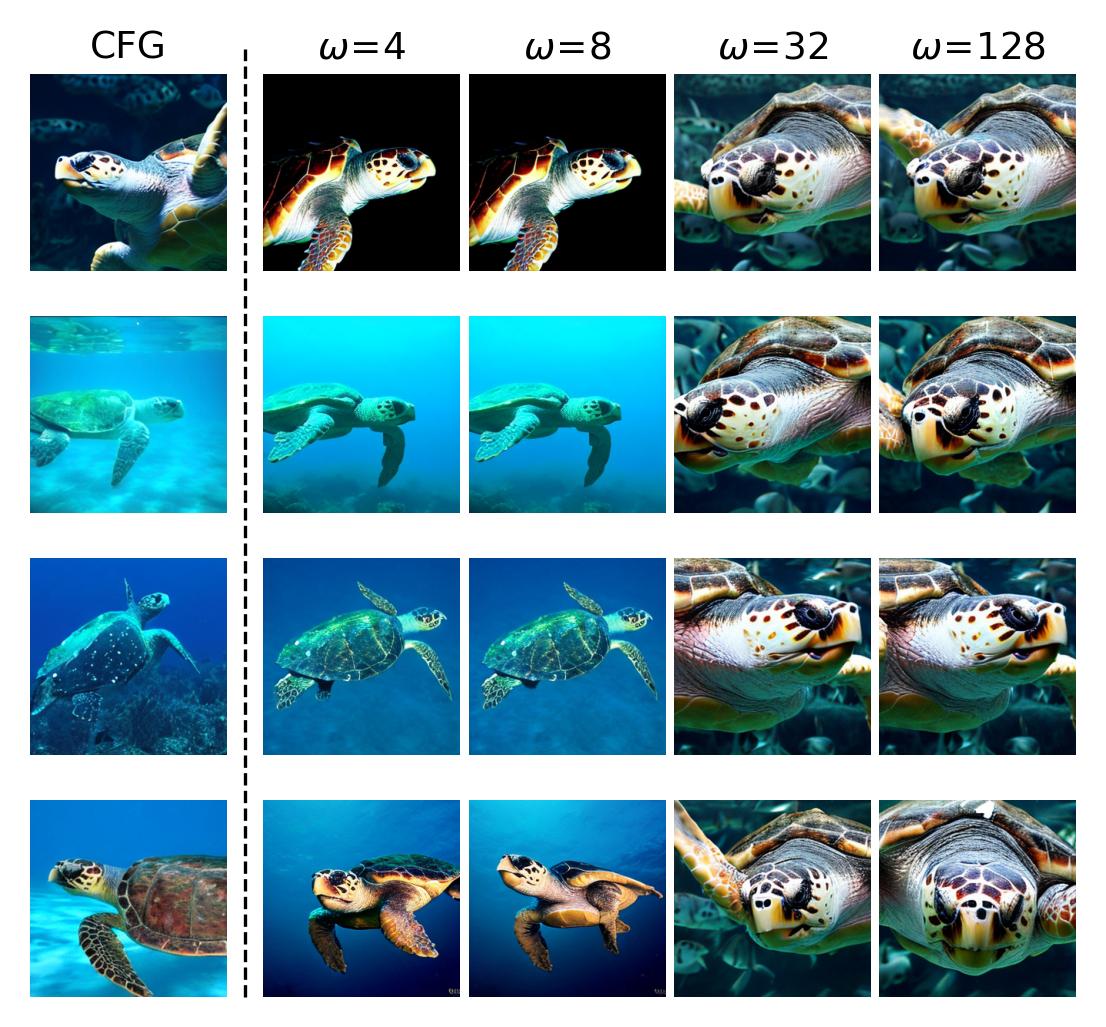}
        \caption{loggerhead turtle (lat. Caretta caretta)}
        \label{fig:traj33}
    \end{subfigure}

    \vspace{0.5cm}

    \begin{subfigure}[t]{0.48\linewidth}
        \centering
        \includegraphics[width=\linewidth]{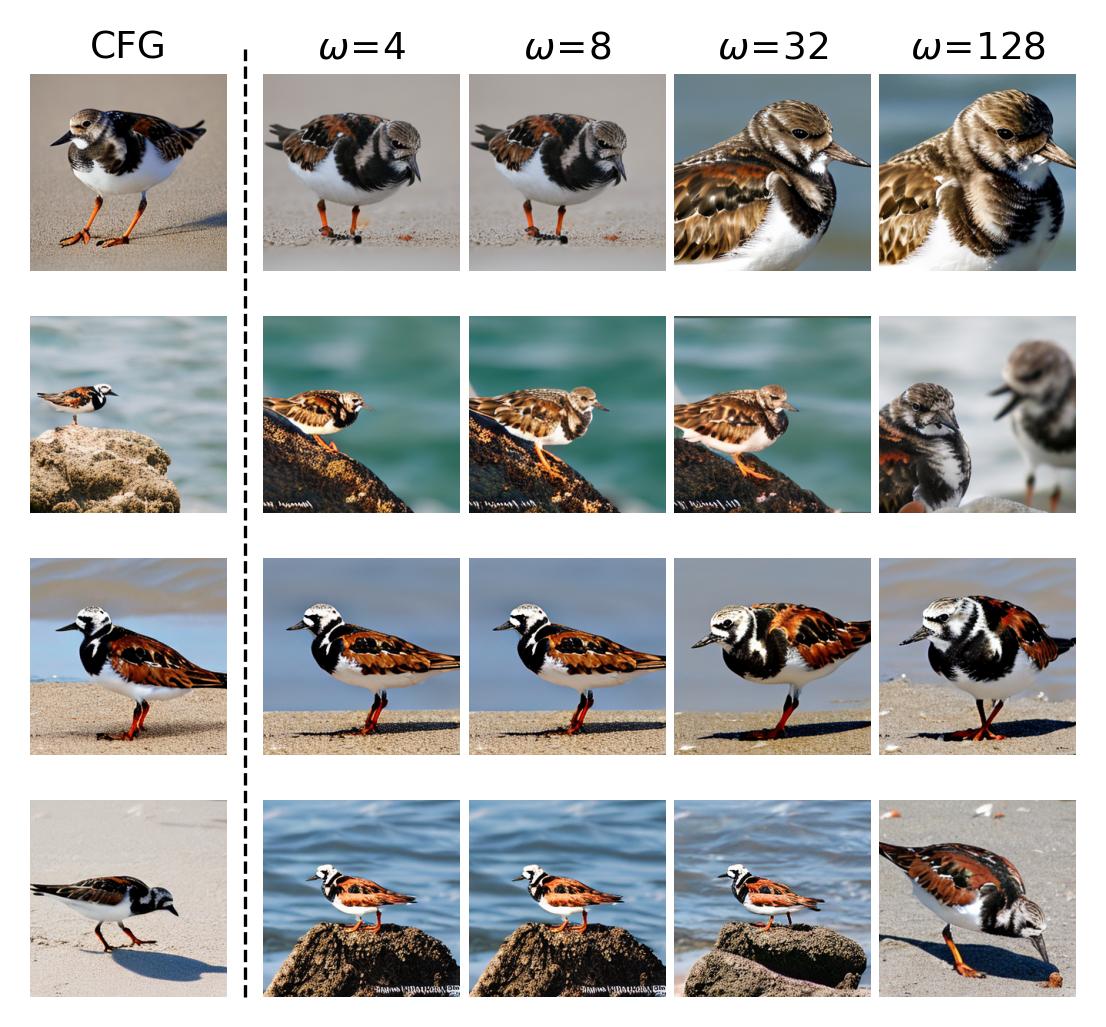}
        \caption{ruddy turnstone (lat. Arenaria interpres)}
        \label{fig:traj139}
    \end{subfigure}
    \hfill 
    \begin{subfigure}[t]{0.48\linewidth}
        \centering
        \includegraphics[width=\linewidth]{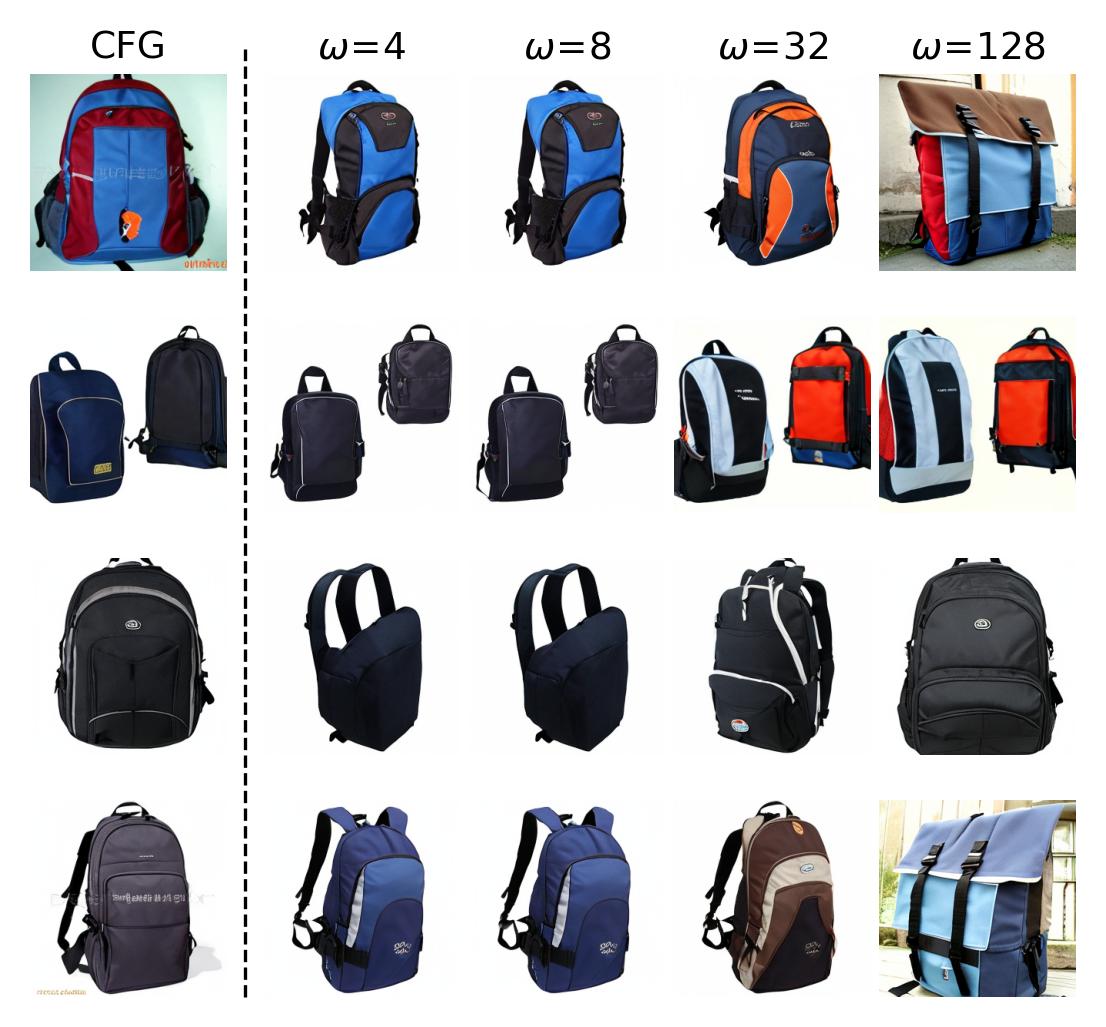}
        \caption{backpack, back pack, knapsack, packsack, rucksack, haversack}
        \label{fig:traj414}
    \end{subfigure}

    \caption{Additional qualitative ImageNet trajectories for four classes. Baseline CFG and IW-CFM both use $\gamma=4$; the IW-CFM columns vary the target powering parameter $\omega$.}
    \label{fig:traj_4_33_139_414}
\end{figure}

\begin{figure}[h]
    \centering
    \begin{subfigure}[t]{0.48\linewidth}
        \centering
        \includegraphics[width=\linewidth]{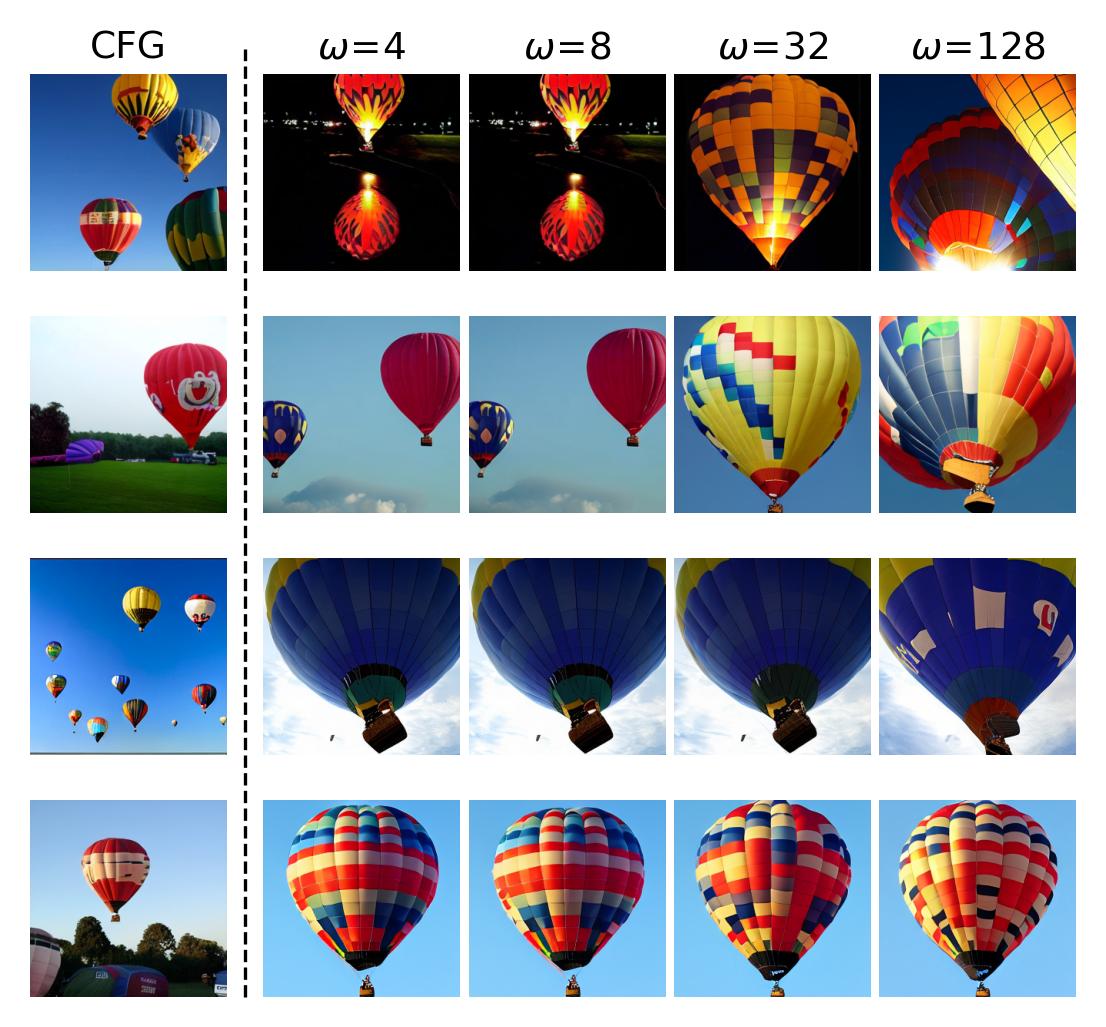}
        \caption{balloon}
        \label{fig:traj417}
    \end{subfigure}
    \hfill 
    \begin{subfigure}[t]{0.48\linewidth}
        \centering
        \includegraphics[width=\linewidth]{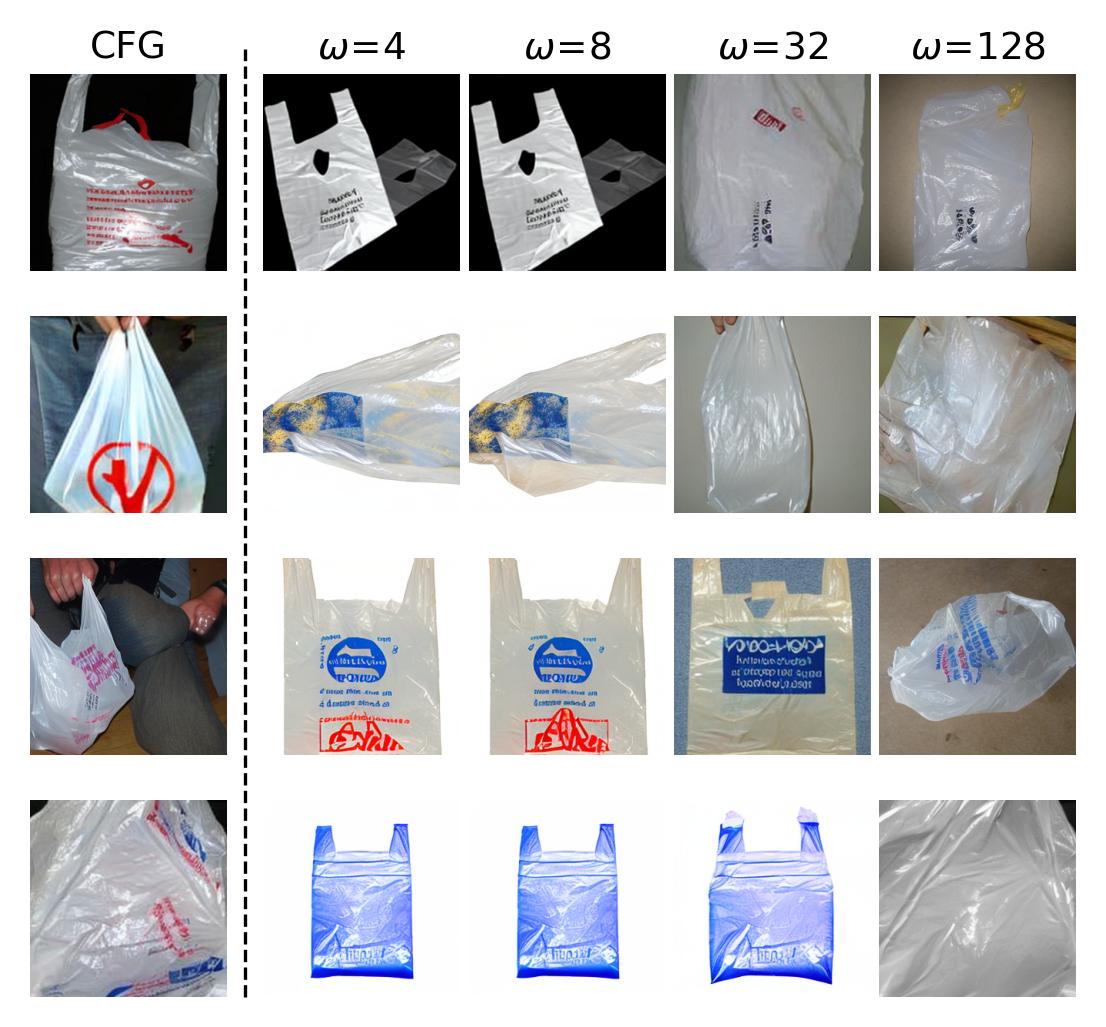}
        \caption{plastic bag}
        \label{fig:traj728}
    \end{subfigure}

    \vspace{0.5cm}

    \begin{subfigure}[t]{0.48\linewidth}
        \centering
        \includegraphics[width=\linewidth]{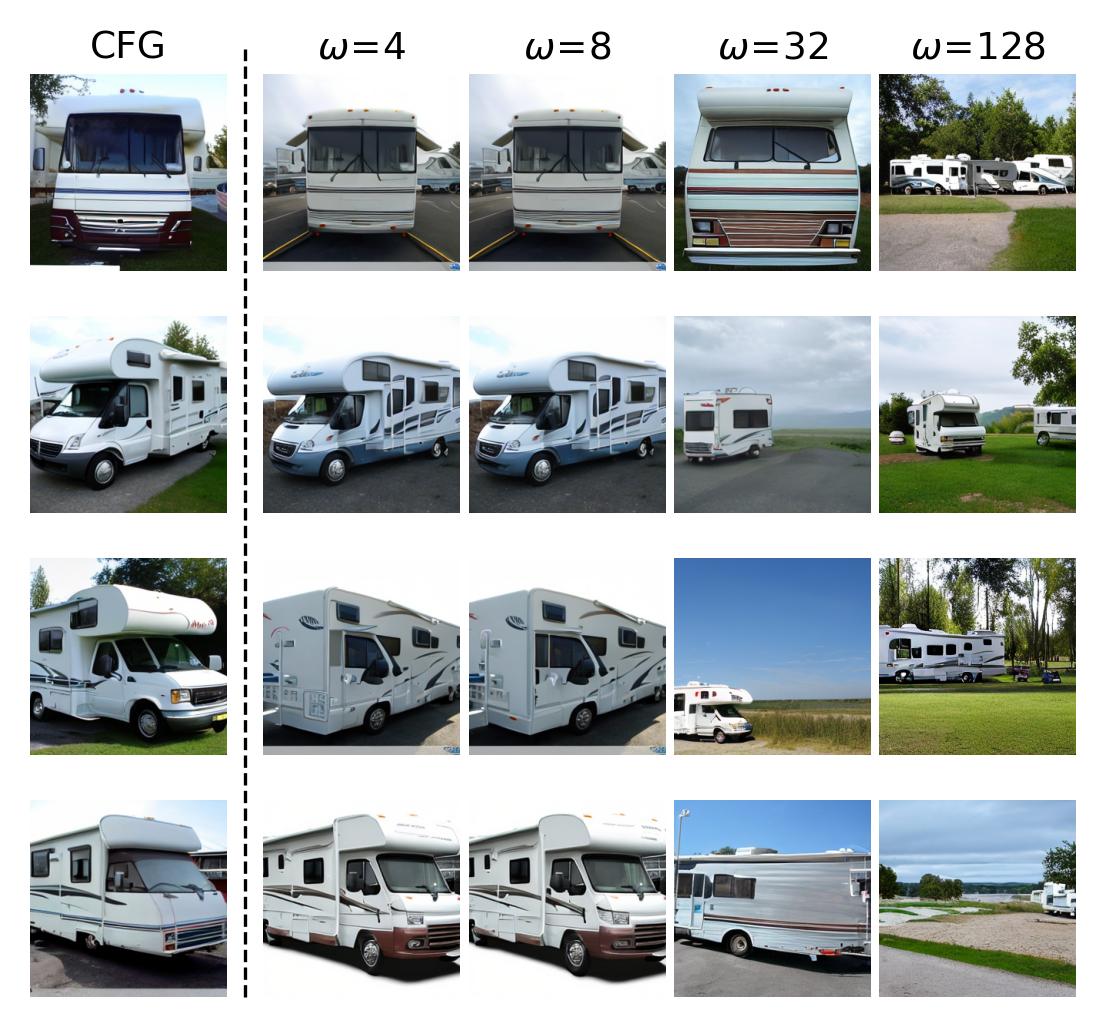}
        \caption{recreational vehicle, RV, R.V.}
        \label{fig:traj757}
    \end{subfigure}
    \hfill 
    \begin{subfigure}[t]{0.48\linewidth}
        \centering
        \includegraphics[width=\linewidth]{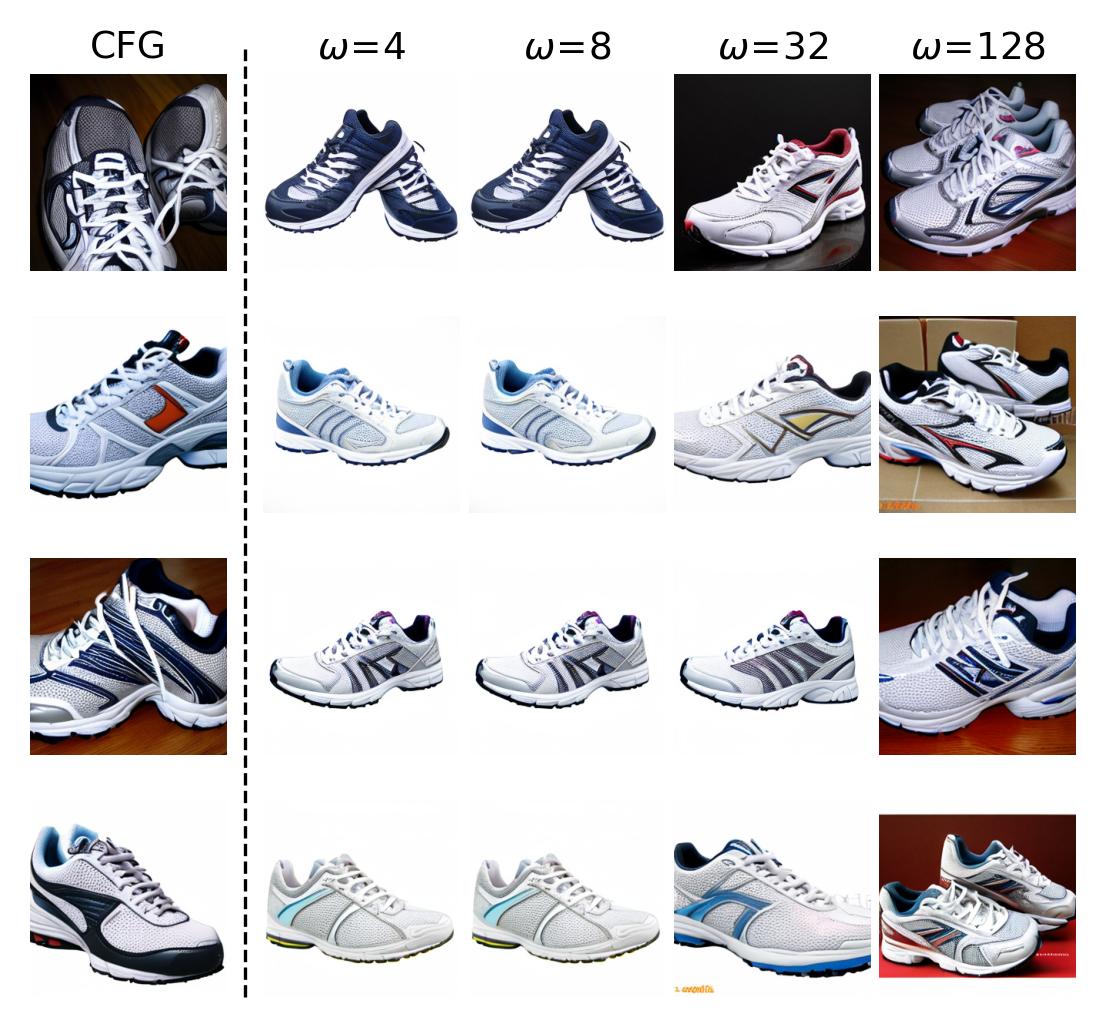}
        \caption{running shoe}
        \label{fig:traj770}
    \end{subfigure}

    \caption{Additional qualitative ImageNet trajectories for four classes. Baseline CFG and IW-CFM both use $\gamma=4$; the IW-CFM columns vary the target powering parameter $\omega$.}
    \label{fig:traj_417_728_757_770}
\end{figure}

\begin{figure}[h]
    \centering
    \includegraphics[width=0.48\linewidth]{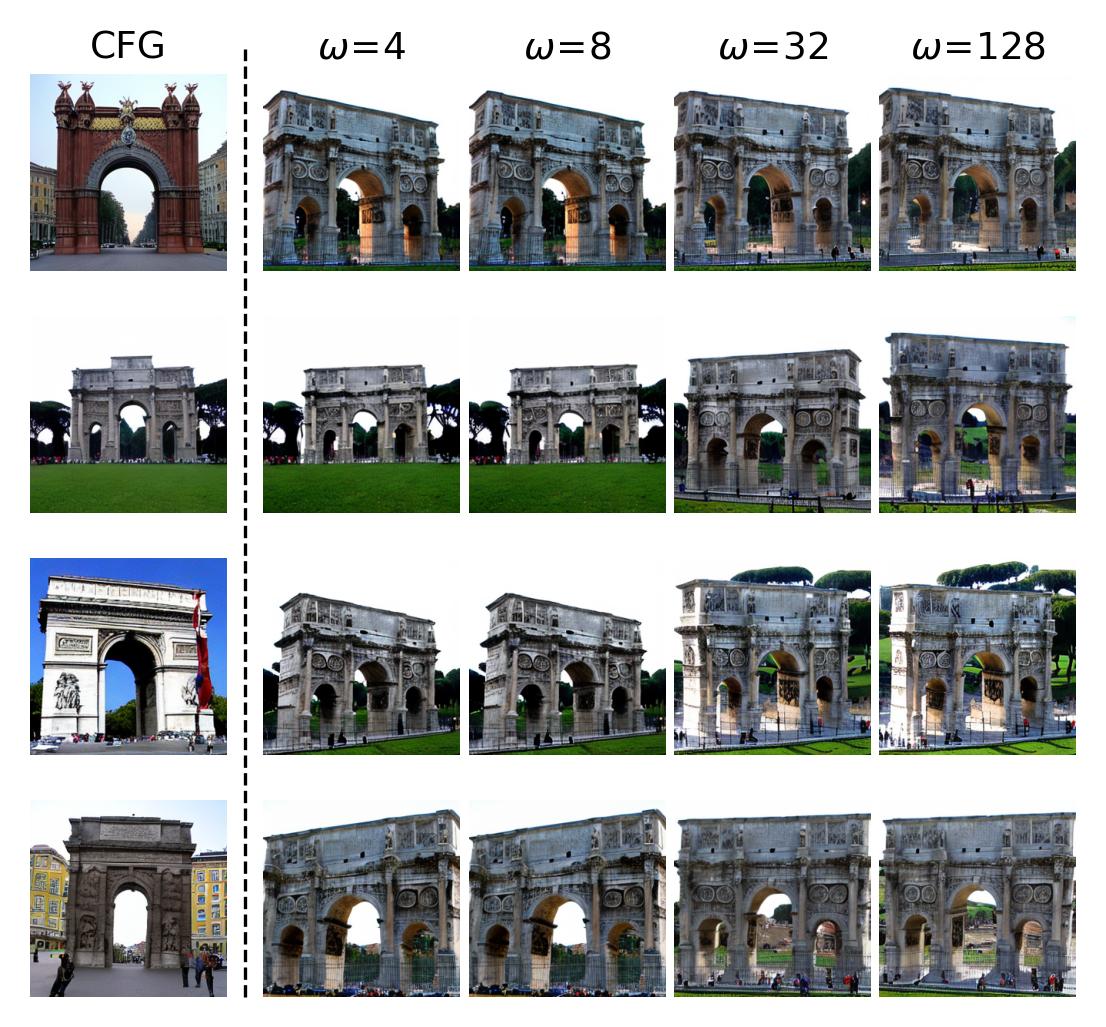}
    \caption{Additional qualitative trajectory for the triumphal arch class. Baseline CFG and IW-CFM both use $\gamma=4$; the IW-CFM columns vary the target powering parameter $\omega$.}
    \label{fig:traj873}
\end{figure}

\end{document}